\documentclass[10pt,letterpaper]{article}
\usepackage{language2drive_preprint}
\usepackage{amsmath,amssymb,amsthm,booktabs,array,tabularx,longtable,multirow,makecell}
\usepackage{graphicx,colortbl,url,pifont}
\usepackage{tikz}
\usetikzlibrary{positioning,arrows.meta,calc}
\usepackage{listings,placeins}
\hypersetup{
pdftitle={Language2Drive: Do Vision-Language-Action Models Really Follow the Instructions in Autonomous Driving?},
pdfauthor={Zihan You, Zehua Fan, Wenqi Lyu, Aoqi Wang, Ni Wang, Yifei Yang, Linge Zhao, Huan Li, Sining Ang, Yuguang Yang, Jueyu Chen, Xiangyu Li, Qianshan Wei, Jieyuan Pei, Yuxin Wang, Zhiwen Chen, Jiaqi Li, Wenxuan Song, Yanyuan Qiao, Benjin Zhu, Haoang Li, Yan Wang},
pdfsubject={Language2Drive benchmark},
pdfkeywords={autonomous driving, vision-language-action models, instruction following, benchmark}}
\newcolumntype{Y}{>{\raggedright\arraybackslash}X}

\newcommand{\ltd}{\textnormal{Language2Drive}}

\newcommand{\SISR}{\mathrm{SISR}}
\newcommand{\SemISR}{\mathrm{SemISR}}
\newcommand{\CFPS}{\mathrm{CFPS}}
\newcommand{\SemCFPS}{\mathrm{Sem\mbox{-}CFPS}}

\newtheorem{proposition}{Proposition}
\newtcolorbox{mygraybox}{
  colback=gray!10,
  colframe=black,
  boxrule=1.5pt,
  arc=4pt
}
\begin{document}
\pdfbookmark[0]{Language2Drive}{title}
\thispagestyle{plain}
\begin{tcolorbox}[enhanced,frame hidden,colback=ltdpale,arc=9pt,
left=13pt,right=13pt,top=13pt,bottom=13pt,before skip=0pt,after skip=0pt]
\ifshowinstitutionbadges
  % Official institutional marks; see assets/branding/README.md for sources.
% Display order follows the affiliation list; disabled by default.
\begingroup\setlength{\parskip}{0pt}\centering
\institutionbadge{tsinghua.pdf}\hfill
\institutionbadge{southeast.pdf}\hfill
\institutionbadge{liauto.pdf}\hfill
\institutionbadge{hkustgz.pdf}\hfill
\institutionbadge{sjtu.pdf}\hfill
\institutionbadge{adelaide.pdf}\hfill
\institutionbadge{bistu.pdf}\hfill
\institutionbadge{amazon.pdf}\hfill
\institutionbadge{bjtu.pdf}\hfill
\institutionbadge{wuhan.pdf}\par
\vspace{7pt}
\institutionbadge{hust.pdf}\hfill
\institutionbadge{ustc.pdf}\hfill
\institutionbadge{nio.pdf}\hfill
\institutionbadge{utexas.pdf}\hfill
\institutionbadge{cas.pdf}\hfill
\institutionbadge{hithink.pdf}\hfill
\institutionbadge{amd.pdf}\hfill
\institutionbadge{hku.pdf}\hfill
\institutionbadge{epfl.pdf}\par
\endgroup

  \vspace{10pt}
\fi
{\sffamily\bfseries\fontsize{18}{22}\selectfont\color{ltdtitle}\centering
Language2Drive: Do Vision-Language-Action Models\\
Really Follow the Instructions in Autonomous Driving?\par}
\vspace{9pt}
{\rmfamily\fontsize{10}{14}\selectfont\color{ltdink}\centering\setlength{\parskip}{0pt}
\authorname{Zihan You}{1,2,*,\dagger}\hspace{10pt}%
\authorname{Zehua Fan}{5,*}\hspace{10pt}%
\authorname{Wenqi Lyu}{6,*}\hspace{10pt}%
\authorname{Aoqi Wang}{7,*}\hspace{10pt}%
\authorname{Ni Wang}{8}\hspace{10pt}%
\authorname{Yifei Yang}{9}\\[1.5pt]
\authorname{Linge Zhao}{10}\hspace{10pt}%
\authorname{Huan Li}{11}\hspace{10pt}%
\authorname{Sining Ang}{12}\hspace{10pt}%
\authorname{Yuguang Yang}{1}\hspace{10pt}%
\authorname{Jueyu Chen}{12,13}\hspace{10pt}%
\authorname{Xiangyu Li}{14}\\[1.5pt]
\authorname{Qianshan Wei}{15}\hspace{10pt}%
\authorname{Jieyuan Pei}{1,4,16}\hspace{10pt}%
\authorname{Yuxin Wang}{17}\hspace{10pt}%
\authorname{Zhiwen Chen}{4}\hspace{10pt}%
\authorname{Jiaqi Li}{18}\\[1.5pt]
\authorname{Wenxuan Song}{4}\hspace{10pt}%
\authorname{Yanyuan Qiao}{19}\hspace{10pt}%
\authorname{Benjin Zhu}{3}\hspace{10pt}%
\authorname{Haoang Li}{4}\hspace{10pt}%
\authorname{Yan Wang}{1,\ddagger}\par}
\vspace{7pt}
{\centering\color{ltdink!18}\rule{.96\linewidth}{.35pt}\par}
\vspace{6pt}
{\rmfamily\fontsize{8.2}{11.3}\selectfont\color{ltdink!85}\centering\setlength{\parskip}{0pt}
\affil{1}{Institute for AI Industry Research (AIR), Tsinghua University}\\[1pt]
\affil{2}{Southeast University}\hspace{5pt}%
\affil{3}{Li Auto Inc.}\hspace{5pt}%
\affil{4}{The Hong Kong University of Science and Technology (Guangzhou)}\\[1pt]
\affil{5}{Shanghai Jiao Tong University}\hspace{5pt}%
\affil{6}{University of Adelaide}\hspace{5pt}%
\affil{7}{Beijing Information Science and Technology University}\\[1pt]
\affil{8}{Amazon}\hspace{5pt}%
\affil{9}{Beijing Jiaotong University}\hspace{5pt}%
\affil{10}{Wuhan University}\hspace{5pt}%
\affil{11}{Huazhong University of Science and Technology}\\[1pt]
\affil{12}{University of Science and Technology of China}\hspace{5pt}%
\affil{13}{NIO}\hspace{5pt}%
\affil{14}{University of Texas at Austin}\\[1pt]
\affil{15}{Chinese Academy of Sciences}\hspace{5pt}%
\affil{16}{HiThink Research}\hspace{5pt}%
\affil{17}{AMD}\hspace{5pt}%
\affil{18}{The University of Hong Kong}\hspace{5pt}%
\affil{19}{EPFL}\par}
\vspace{5pt}
{\rmfamily\fontsize{8}{11}\selectfont\color{ltdink!85}\centering
\textsuperscript{*}Core Contributors\qquad
\textsuperscript{\ensuremath{\dagger}}Project Lead\qquad
\textsuperscript{\ensuremath{\ddagger}}Corresponding Author\par}
\vspace{9pt}
{\sffamily\bfseries\small\color{ltdink}Abstract\par}
\vspace{2pt}
{\color{ltdink}\fontsize{9.5}{11.8}\selectfont
Recent Vision-Language-Action (VLA) models increasingly use natural-language instructions to condition autonomous driving behavior. Yet existing closed-loop benchmarks primarily measure route completion, safety, and rule compliance, leaving a fundamental question unresolved: does language actually determine which feasible behavior the vehicle executes? A trajectory may be safe and successful while still violating or even ignoring the passenger's instruction. To address this, we introduce \ltd, a benchmark designed to directly test this capability. Its key evaluation unit, the paired counterfactual unit (PCU), holds the initial driving scene fixed while changing only the instruction. This design enables us to distinguish isolated instruction success from counterfactual paired success across four instruction-following scopes. To test whether this distinction persists across evaluation regimes, we instantiate Language2Drive in complementary closed-loop CARLA simulation, reconstructed real-world scenes, and open-loop driving benchmarks, including our released CARLA-Language2Drive simulator with purpose-built maps. Evaluating representative driving VLAs across these settings reveals a striking mismatch between strong driving competence and weak language control: 93.2\% of successful rollouts do not have a successful counterpart under the alternative feasible instruction. This failure is dominated by one-sided completion, where a model repeatedly executes one preferred behavior but fails when the same scene requires the counterfactual alternative. These results expose two fundamental limitations of current evaluation: \textbf{isolated success and driving scores, which are widely used by conventional driving benchmarks, can substantially overstate language-following capability, while current driving VLAs themselves lack reliable instruction following under matched counterfactual tests}. Effective evaluation must therefore measure not only whether a VLA can drive successfully, but whether changing the instruction reliably alters which feasible behavior it selects and successfully executes.

\par}
\vspace{8pt}
{\fontsize{9.5}{12.5}\selectfont\hypersetup{urlcolor=ltdtitle}
\begin{minipage}[b]{.61\linewidth}
\textsf{\textbf{Project Page:}}\enspace\href{https://zihanyou.github.io/Language2Drive/}{\projectpageicon\enspace Project Page}\par
\textsf{\textbf{Code:}}\enspace\href{https://github.com/ZihanYou/Language2Drive}{\raisebox{-.2ex}{\includegraphics[height=1.4ex]{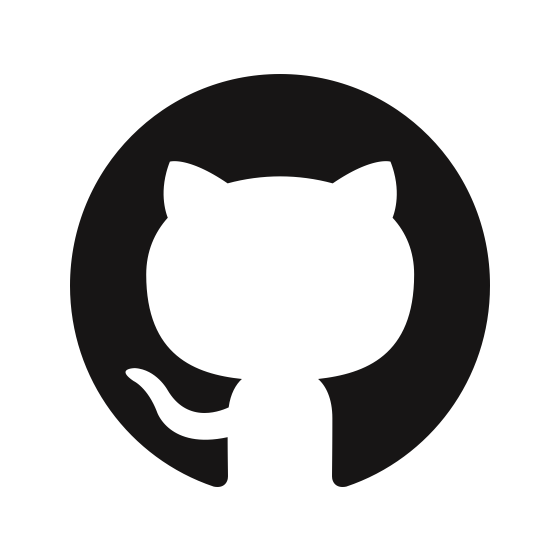}}\enspace GitHub}\par
\textsf{\textbf{Benchmark:}}\enspace\href{https://huggingface.co/datasets/Mike087/Language2Drive}{\raisebox{-.2ex}{\includegraphics[height=1.4ex]{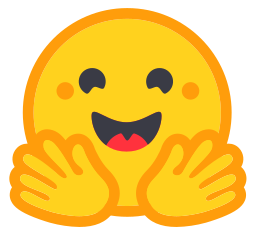}}\enspace Hugging Face}\par
\end{minipage}\hfill
\begin{minipage}[b]{.35\linewidth}\raggedleft
\raisebox{3pt}{\includegraphics[width=135pt]{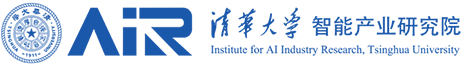}}\par
\end{minipage}\par}
\end{tcolorbox}

\clearpage
\begingroup
\pagestyle{contents}
\hypersetup{linkcolor=tocblue,linktoc=section}
\tableofcontents
\clearpage
\endgroup
\pagestyle{plain}

\section{Introduction}
\begin{figure}[t]
\centering
\includegraphics[width=\linewidth]{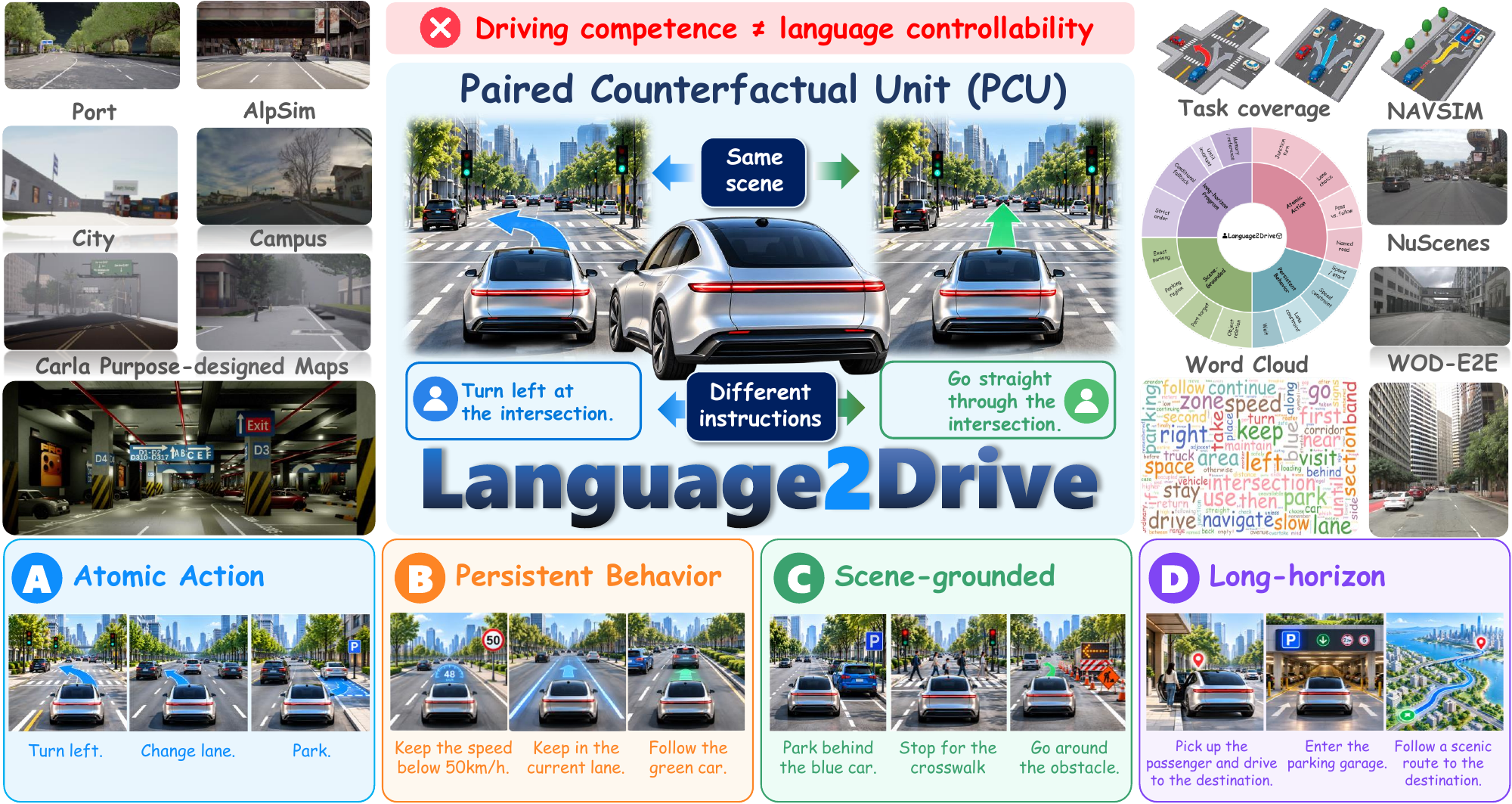}
\caption{\textbf{Overview of Language2Drive.}
Language2Drive evaluates autonomous-driving VLA instruction following across four scopes: (A) Atomic actions, (B) Persistent behaviors, (C) Scene-grounded instructions, and (D) Long-horizon programs. To support these task, we build purpose-designed CARLA maps includes cities, underground parking lot, campus and port. We also curate relevant scenarios from AlpSim, NAVSIM, nuScenes, and WOD-E2E. The Task-Coverage chart summarizes the diversity of driving behaviors, while the Word Cloud illustrates the breadth of natural-language expressions. We further introduce the Paired Counterfactual Unit (PCU), which varies the instruction while keeping the driving scene fixed, directly testing whether language changes the vehicle’s behavior. Our benchmark highlights a key distinction: driving competence does not necessarily imply language controllability or reliable instruction following.}
\label{fig:overview}
\end{figure}
\label{sec:intro}
Recent advances in Vision-Language-Action (VLA) models have opened a new paradigm for autonomous driving: natural-language instructions is able to condition driving behaviors~\citep{zhou2026opendrivevla,zhou2026autovla,shao2024lmdrive,tian2024drivevlm,hwang2024emma,renz2025simlingo}.  By building upon internet-scale pretraining in large vision-language models, VLAs inherit rich world knowledge, open-world perception, scene understanding, and multimodal reasoning capabilities from foundation models. It promise an interface to express passenger intent without direct manual control, in which passengers can specify how to drive and what task to accomplish. Realizing this promise requires more than generating plausible motion: when several safe, legal, and physically feasible behaviors are available, the system must select and sustain the one requested by the passenger. 

Existing evaluation does not fully establish this capability. Driving benchmarks emphasize route completion, safety, and traffic rule compliance~\citep{dauner2024navsim,cao2025pseudo,yu2025drivee2e,jia2024bench2drive}, while language-guided evaluations assess instructed route execution or instruction-specific predicted actions~\citep{shao2024lmdrive,renz2025simlingo}. These evaluations provide important evidence of driving competence and task completion. Yet a safe trajectory can still reach the wrong named parking space or violate a requested passing restriction. More fundamentally, even successful execution of an individual instruction does not establish that language reliably determines the selected behavior. The central question is therefore: Do VLAs really follow instructions, or can their apparent successes be explained by behaviors they would execute regardless of the request?

% To address this, we introduce \textbf{Language2Drive}, a closed-loop benchmark
% Consider an intersection where turning left and going straight are both safe and feasible. A policy that always turns left may succeed on the left-turn request without actually responding to the instruction, while changing the scene together with the instruction makes it unclear whether behavior changes are caused by language or the environment. We therefore argue that instruction following should be evaluated through matched alternatives from the same decision context rather than isolated successes. Based on this principle, \textbf{Language2Drive} built around \emph{paired counterfactual units} (PCUs), where two mutually exclusive but feasible instructions are evaluated in separate runs with matched scenes, non-language inputs, and simulation seeds. Since an instruction can be satisfied by many valid trajectories, each task is defined by  machine-checkable success criteria that specifies  spatial, temporal, and safety conditions on the realized behavior. We report Strict Instruction Success Rate (SISR) for individual requests and Counterfactual Paired Success (CFPS) for PCUs in which both alternatives are successfully completed, making it possible to identify systems that succeed only on one side of an instruction contrast. We further use three meaning-preserving expressions for each CARLA instruction to test whether this paired success remains stable across different wording.
To address this, we introduce \textbf{Language2Drive}, a closed-loop benchmark for evaluating reliable instruction following in autonomous driving. The key challenge is that success on a single request does not establish language-conditioned control. At an intersection where turning left and going straight are both safe and feasible, for example, a policy that consistently turns left may succeed on the left-turn request without demonstrating that it can execute the alternative when asked. Conversely, comparing instructions across different scenes confounds language variation with environmental variation. Language2Drive therefore evaluates matched alternatives through paired counterfactual units (PCUs). Each PCU holds the decision context, non-language inputs, and simulation seed fixed while evaluating two mutually exclusive but feasible instructions in separate runs. This design distinguishes isolated success from the ability to execute either requested alternative under matched conditions. Because an instruction may admit many valid trajectories, each request is evaluated by machine-checkable spatial, temporal, and safety criteria rather than reference-trajectory matching. Strict Instruction Success Rate (SISR) measures individual completion, whereas Counterfactual Paired Success (CFPS) requires both alternatives in a PCU to succeed. Three meaning-preserving expressions further test whether paired success remains stable across linguistic variation.

Language2Drive organizes tasks by what language controls: atomic actions, persistent behavior, scene-grounded navigation, and long-horizon programs. To support these tasks, we release \textbf{CARLA-Language2Drive}, a CARLA extension that adds six purpose-designed maps covering underground parking, a port, Chinese urban roads, a U.S.-style city with detailed signage, a campus, and a highway. A complementary AlpaSim track is built from reconstructed real-world road scenes selected through targeted mining and manual review~\citep{alpasim_2025,wu20253dgut}. These environments provide alternative maneuvers, identifiable referents, and execution intervals needed to test language-conditioned behavior. We also extend paired evaluation based on the test sets developed from nuScenes, NAVSIM, and WOD-E2E, complementing closed-loop execution with open-loop tests on recorded driving contexts.

Our evaluation reveals that \textbf{driving quality and individual instruction success do not establish reliable language control}. Across principal CARLA systems, 93.2\% of successful rollouts lack a successful opposite-instruction counterpart. Most observed successes are therefore one-sided: one request is completed, but its feasible alternative is not. In reconstructed scenes, Alpamayo~1.5 achieves 100\% Safety and a reference-path DrivingScore of 97.17, yet \textbf{zero} CFPS. Moreover, no evaluated CARLA system completes both alternatives under all three tested language expressions. The individual--paired discrepancy also appears in open-loop evaluations across all three datasets. Together, these findings show why driving scores and isolated successes must be supplemented by tests of alternative selection, complete execution, and consistency across equivalent instructions.

Our contributions are threefold. \textbf{First}, we introduce a paired evaluation framework that distinguishes individual instruction completion from successful execution of alternative requests under matched conditions. \textbf{Second}, we instantiate four language-control scopes in Language2Drive, contributing an extended CARLA simulator with six purpose-designed maps, a curated AlpaSim track, and complementary open-loop evaluations. \textbf{Third}, we reveal a substantial separation between driving quality and reliable instruction following, characterizing one-sided success, incomplete execution, and fragile paired success across paraphrases. These findings establish matched, instruction-specific execution as an essential complement to conventional driving evaluation.

\section{Language2Drive}
\label{sec:method}
\subsection{Overview}
% \vspace{-10pt}
We introduce Language2Drive, a closed-loop benchmark designed to evaluate whether vision-language-action models for autonomous driving reliably follow passenger instructions.
% Unlike evaluations centered on route completion and safety, Language2Drive targets scenarios where language must select among multiple feasible driving behaviors and govern their execution over time. 
\begin{tcolorbox}[
    colback=gray!10,  % Background color: 10% gray (light gray)
    colframe=black!95, % Frame color: 75% gray (darker gray for contrast)
    arc=4pt,          % Rounded corners with 4pt radius
    boxrule=1pt,    % Frame thickness
    title=Task Definition, % Optional title for the box
    fonttitle=\bfseries,
    breakable,        % Allows the box to break across pages if needed
    enhanced,         % For more advanced features like overlay
    attach boxed title to top center={yshift=-0.25mm-\tcboxedtitleheight/2,yshifttext=2mm-\tcboxedtitleheight/2},
    boxed title style={colback=gray!75,colframe=gray!75,sharp corners}
]
\textbf{Autonomous-driving instruction following} is the ability to select and execute the passenger-requested behavior among safe, legal, and physically feasible alternatives when navigation is absent or does not determine the choice. Without a passenger specifying each reaction, the system must still respond to traffic; when a request conflicts with safety or feasibility, it should perform an appropriate safe response. 

\end{tcolorbox}

\subsection{Paired counterfactual units}
\label{sec:pcu}
A single successful request can coincide with a system's default behavior. We therefore test the same decision context twice: for example, one run asks the vehicle to turn left, while the other to continue straight. Both requests are feasible, but a successful execution of one excludes the other. We call this matched test a Paired Counterfactual Unit (PCU). It tests whether the system can execute either alternative when asked. 
% (Figure~\ref{fig:pcu-anatomy}).

Both experimental conditions are initialized with the same map, ego/actor states, history, non-language inputs, and simulation seed. The human language instruction changes; later observations evolve with each run's actions. Condition-specific destinations, future reference paths, and turn commands remain private to the evaluator. Each PCU has two instructions, each with an associated success criterion; meaning-preserving expressions test the same pair without changing their targets. Appendix~\ref{app:theory} states the conditions under which an instruction-invariant system cannot succeed on both test conditions.

\subsection{Machine-checkable success criteria}
\label{sec:contracts}
Each instruction is paired with a machine-checkable success criterion over the executed trajectory.  It specifies the requested outcome, where the condition applies, and when execution is complete. Three examples illustrate the distinction between plausible motion and successful execution:
% \begin{itemize}\setlength{\itemsep}{1pt}
% \item \emph{``Turn left at this intersection.''} Enter the requested left exit, without satisfying the paired straight-exit criterion.
% \item \emph{``Park in space B204.''} Place the vehicle footprint inside B204 with a valid heading, sufficiently low speed, and the required dwell time.
% \item \emph{``First visit B208, then go to F206.''} Complete the two stages in that order; visiting both in reverse order does not count.
% \end{itemize}
 \emph{``Turn left at this intersection.''} Enter the requested left exit, without satisfying the paired straight-exit criterion.
\emph{``Park in space B204.''} Place the vehicle footprint inside B204 with a valid heading, sufficiently low speed, and the required dwell time.
\emph{``First visit B208, then go to F206.''} Complete the two stages in that order; visiting both in reverse order does not count.

Persistent conditions apply to a reached, nonempty interval that the vehicle must complete. Stopping before the interval cannot satisfy a request to maintain a speed throughout it. The verifier checks actual ego states and composes geometric, behavioral, and temporal predicates; exact operators and thresholds appear in Appendix~\ref{app:task-metrics}.

\subsection{Metrics}
\label{sec:metrics}

Complete instruction execution is measured by two primary metrics:
\textbf{SISR} for individual requests and \textbf{CFPS} for matched
alternatives.
Supplementary metrics characterize partial completion and consistency
across equivalent instructions; driving quality is reported separately.

\paragraph{Individual and paired success.}
Let $s_{ir,j}\in\{0,1\}$ be the strict verdict for PCU $i$,
expression $r$, and arm $j\in\{0,1\}$.
Success requires the complete requested contract, rejection of its
paired alternative, and all applicable safety and execution conditions.
We average expressions within each PCU and then weight PCUs equally:
$\langle z\rangle=
\frac{1}{N}\sum_{i=1}^{N}
\frac{1}{R_i}\sum_{r=0}^{R_i-1}z_{ir}$.
Unsuccessful executions remain in the scheduled denominator.
All instruction metrics below use a 0--100 scale.
We report
\begin{equation}
\mathrm{SISR}
=
100\left\langle
\frac{s_{ir,0}+s_{ir,1}}{2}
\right\rangle,
\qquad
\mathrm{CFPS}
=
100\langle s_{ir,0}s_{ir,1}\rangle.
\label{eq:pair}
\end{equation}
SISR credits each completed request; CFPS credits a matched pair
only when both requests are completed.
For example, always completing the left turn yields 50 SISR but
zero CFPS on a balanced left/straight test.
Their difference has an exact interpretation:
\begin{equation}
\mathrm{SISR}-\mathrm{CFPS}
=
\tfrac12 q_{\mathrm{one}},
\qquad
q_{\mathrm{one}}
=
100\langle|s_{ir,0}-s_{ir,1}|\rangle.
\label{eq:one-sided}
\end{equation}
Here $q_{\mathrm{one}}$ is the percentage of paired instances with
exactly one successful arm under the same averaging weights.
The gap therefore quantifies one-sided completion.
Conditions under which CFPS rules out instruction-invariant
execution are given in Appendix~\ref{app:metric-interpretation}.

\paragraph{Partial completion.}
To characterize execution below strict success, we retain two
diagnostic views.
\emph{Relaxed semantic proxies}, SemISR and Sem-CFPS, apply
Equation~\ref{eq:pair} to predefined task-specific relaxed verdicts
instead of strict verdicts.
They measure satisfaction of weaker criteria, not additional
strict successes.
Graded completion uses
$c_{ir,j}=C_{i,j}(\tau_{ir,j})\in[0,1]$
to score how fully a trajectory completes its own request:
\begin{equation}
\mathrm{SCS}
=
100\left\langle
\frac{c_{ir,0}+c_{ir,1}}{2}
\right\rangle,
\qquad
\mathrm{PCS}
=
100\langle\min(c_{ir,0},c_{ir,1})\rangle.
\label{eq:graded}
\end{equation}
SCS measures average completion; PCS measures completion of the
weaker alternative, so progress on one arm cannot compensate for
an uncompleted counterpart.
Their gap is
$\mathrm{SCS}-\mathrm{PCS}
=50\langle|c_{ir,0}-c_{ir,1}|\rangle$.
The relaxed and graded rules are distinct diagnostics rather than
interchangeable success definitions, as detailed in Appendix~\ref{app:metric-partial}.

\paragraph{Summary reporting metric L2D.}
For compact within-track reporting, we first summarize language-following ability across three complementary families:
\begin{equation}
\begin{aligned}
a_1&=\tfrac12(\mathrm{SISR}+\mathrm{CFPS}),\quad
a_2=\tfrac12(\mathrm{SemISR}+\mathrm{Sem\text{-}CFPS}),\quad
a_3=\tfrac12(\mathrm{SCS}+\mathrm{PCS}),\\
L_{\mathrm{bal}}
&=
\left(
\frac{\sqrt{a_1}+\sqrt{a_2}+\sqrt{a_3}}{3}
\right)^2 .
\end{aligned}
\label{eq:metric-lbal}
\end{equation}
Here, $a_1$, $a_2$, and $a_3$ summarize strict instruction following, semantic instruction following, and consistency, respectively. $L_{\mathrm{bal}}$ favors models that perform well across all three families rather than excelling in only one.

We then combine language following with track-specific driving quality:
\begin{equation}
\mathrm{L2D}_t
=
\frac{1+0.99L_{\mathrm{bal}}}{100}
\begin{cases}
\tfrac12(\mathrm{PSM}+\mathrm{DrivingScore}), & t=\mathrm{CARLA},\\
\tfrac12(\mathrm{Safety}+\mathrm{DrivingScore}), & t=\mathrm{RealSim}.
\end{cases}
\label{eq:l2d}
\end{equation}
Intuitively, the driving term measures whether the vehicle can drive safely and effectively, while $L_{\mathrm{bal}}$ measures whether its behavior is actually controlled by language. Therefore, high driving quality alone cannot produce a high L2D score: if the model largely ignores the instruction, its final score remains low.
\paragraph{Driving quality and summary reporting.}
DrivingScore measures reference-path
advancement.
Safety is characterized by Predictive Safety Margin (PSM) on
independent CARLA routes and collision/off-road-free Safety in RealSim.
Neither establishes instruction satisfaction.
L2D summarizes the aggregated language and driving measurements
within each track, with driving-only credit capped at one point.
Its definition and properties appear in
Appendix~\ref{app:metric-l2d}; SISR and CFPS remain the primary
measures of complete execution.

\subsection{Four language-control scopes}
\label{sec:scopes}
We organize instructions by what language controls and how long its requirements remain active. Figure~\ref{fig:overview} summarizes their task families and vocabulary.

\textbf{A: Atomic action.}
Language selects a local maneuver or setpoint, such as a junction turn, lane choice, speed/start action, or wait-then-park maneuver. Success requires completing the selected maneuver under its applicable dwell and progress conditions while rejecting the paired alternative. This scope concerns one local decision, even when execution takes several seconds.

\textbf{B: Persistent behavior.}
Language imposes a speed, lane, or actor-related requirement over an annotated interval, including passing versus following. Paired requests share the interval or actor but prescribe different behaviors. Verification requires complete traversal and compliance from scope entry to exit; briefly satisfying the condition or stopping before the interval is insufficient.

\textbf{C: Scene-grounded navigation.}
Language selects a named road, parking region, port target, or exact parking space. The requested entity must be grounded to its physical target: directed exit gates for roads and polygons for regions or spaces. Exact parking additionally checks footprint, heading, speed, and dwell. Reaching a nearby target does not satisfy the requested identity.

\textbf{D: Long-horizon program.}
Language specifies obligations that depend on task state: ordered visits, conditional fallbacks, invariants until an event, or reuse of a remembered reference. The verifier tracks stages, triggers, and bound entities, rejecting incorrect ordering or incomplete execution. Program structure, rather than duration alone, distinguishes this scope from persistent behavior.

\label{sec:fdq-safety}
Two CARLA diagnostics complement the feasible-request core. Foundation Driving Qualification (FDQ) tests traffic reactions without passenger-specified instructions. Safety Qualification uses requests requiring a safe fallback, refusal, or constrained alternative when literal compliance conflicts with safety or feasibility. An open-loop extension developed from nuScenes, NAVSIM, and WOD-E2E evaluates predicted trajectories; Section~\ref{sec:offline-results} describes its construction.

\subsection{Benchmark construction process}
\label{sec:construction}\label{sec:environments}
Our construction process turns a physical decision opportunity into a matched, executable test (Figure~\ref{fig:construction-process}). Each stage addresses a separate validity requirement: meaningful alternatives, unambiguous grounding, controlled inputs, or verifiable execution.
\begin{figure}[t]
\centering\includegraphics[width=\linewidth]{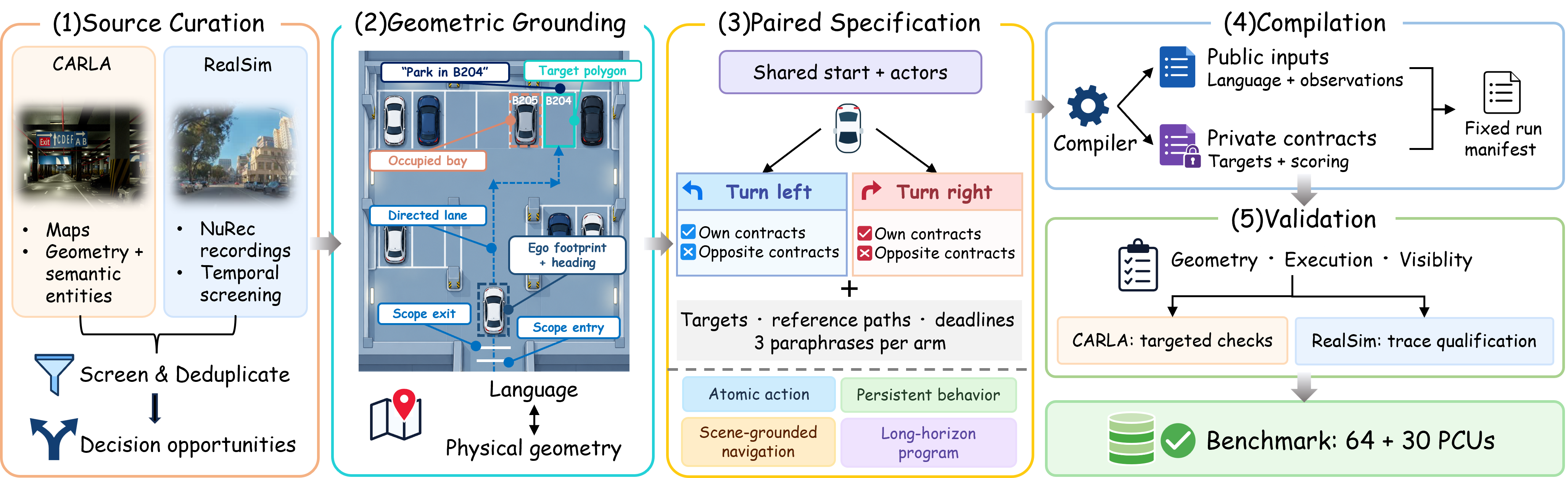}
\caption{\textbf{Benchmark construction process.}The process to build the benchmark is as follows: (1) Source curation identifies decision opportunities; (2) Geometric grounding binds language to physical entities; (3) Paired specification fixes shared conditions and opposed outcomes; (4) Compilation separates public inputs from private criteria; (5) Validation checks the resulting tasks. The benchmark contains 64 CARLA and 30 RealSim PCUs.}
\label{fig:construction-process}
\end{figure}

\textbf{Source curation.}
\textbf{CARLA-Language2Drive} extends CARLA~\citep{dosovitskiy2017carla} with six purpose-designed environment types: underground parking, a port, Chinese urban roads, a US city with detailed signage, a campus, and a highway. We select decision contexts that permit opposed feasible behaviors. \textbf{RealSim} mines reconstructed PhysicalAI-NuRec scenes executed through AlpaSim~\citep{alpasim_2025, wu20253dgut}; temporal screening requires both pre-instruction context and enough recording time for the requested maneuver.

\textbf{Geometric grounding.}
Directed lane graphs, exit gates, named regions, parking polygons, and physical actors bind linguistic references to measurable entities. Parking uses the ego footprint and heading, while persistent tasks specify explicit entry/exit geometry and relevant actors. These bindings distinguish executing a request from merely approaching its target and prevent an empty or unreached interval from counting as compliance.

\textbf{Paired specification.}
Each PCU fixes the initial scene and actor setup, then assigns mutually exclusive requests with corresponding success predicates, private reference paths, and deadlines. CARLA has three meaning-preserving expressions per arm; RealSim uses one.

\textbf{Public/private compilation.}
Compilation separates public language, observations, ego state, and neutral context from private targets and scoring information. Consequently, the evaluator can check the requested alternative without revealing its future path or destination through auxiliary inputs.

\textbf{Task validation.}
Static checks verify entity bindings, directed reachability, footprint constraints, and matched non-language initialization. All PCUs are each validated via a full-scale evaluator that accounts for dynamic execution—incorporating deadlines, occupancy, dynamic participants, and temporal requirements. Human reviewers further verify designated parking spaces, road names, and the operational ranges of "next road segments," making sure the model under test is able to determine the target and the scope of constraints based on actual camera feeds, historical data, and shared context. Every PCU has successful privileged executions of both alternatives under three seeds, giving witnesses that satisfy the requested criterion and reject its counterpart.

% \paragraph{Benchmark composition.}
% The CARLA evaluation contains 64 PCUs on ten maps, with 16 PCUs per scope and 16 task families. Its two arms and three expressions yield 384 rollouts and 192 pairs per system. Entire PCUs stay within one development/test split; split checks cover entity, topology, and actor-program overlap. RealSim contains 30 PCUs from 23 scenes: six atomic and 24 persistent tasks, yielding 60 rollouts and 30 pairs. Figure~\ref{fig:task-language} summarizes the CARLA task distribution and its 384 instruction entries (216 distinct strings). Appendix~\ref{app:maps} gives geometric operators, thresholds, and inventories.

\FloatBarrier

\section{Experiments}
\label{sec:experiments}
\subsection{Evaluation protocol}
We distinguish three outcomes: driving successfully, completing an individual instruction, and completing both requested alternatives under matched conditions.
CARLA evaluates 64 PCUs with two arms and three expressions, yielding 384 core rollouts and 192 paired instances per system as reported in Table~\ref{tab:carla-main}.
RealSim evaluates 30 PCUs from 23 reconstructed scenes with two arms and one expression, yielding 60 rollouts and 30 pairs per system as reported in Table~\ref{tab:real-main}.
Released weights remain frozen, and each evaluated system includes its input adapter and controller.
Appendix~\ref{app:adapters} specifies checkpoint identities, sensors, action interfaces, and execution settings.FDQ and Safety-Q provide complementary closed-loop diagnostics; open-loop evaluations test predicted trajectories on three recorded-driving datasets.

% SISR and CFPS are the primary measures of complete individual and paired execution.
% Tables~\ref{tab:carla-main} and~\ref{tab:realsim-main} additionally report relaxed proxies, graded completion, and driving quality, with L2D providing the summary defined in Equation~\ref{eq:l2d}.
% We retain the component scores to distinguish these outcomes rather than infer instruction following from a composite rank.

\subsection{Can success on one instruction establish language control?}
\label{sec:mainresults}\label{sec:real-results}
\begin{table*}[t]
\caption{\textbf{CARLA-Language2Drive main results.} 384 core rollouts and 192 paired instances per system (64 PCUs, three expressions). PSM uses 24 additional routes. All values are on a 0--100 scale; higher is better. Strict individual and paired success are shown first; L2D follows the component scores.}
\label{tab:carla-main}
\centering\footnotesize
\setlength{\tabcolsep}{2.3pt}
\renewcommand{\arraystretch}{1.15}
\begin{tabularx}{\linewidth}{@{}p{30mm}>{\raggedright\arraybackslash}X*{9}{r}@{}}
\toprule
\multirow{2}{*}{System} & \multirow{2}{*}{\makecell[l]{Training data}} & \multicolumn{2}{c}{Strict success} & \multicolumn{2}{c}{Relaxed proxy} & \multicolumn{2}{c}{Graded} & \multicolumn{2}{c}{Driving} & Summary\\
\cmidrule(lr){3-4}\cmidrule(lr){5-6}\cmidrule(lr){7-8}\cmidrule(lr){9-10}\cmidrule(l){11-11}
&&SISR&CFPS&SemISR&\makecell{Sem-\\CFPS}&SCS&PCS&PSM&\makecell{Drive\\Score}&L2D\\
\midrule
Alpamayo 1.5 & NVIDIA AV & 8.07 & \textbf{1.04} & 42.19 & 30.73 & \textbf{26.18} & \textbf{12.50} & 70.10 & 62.69 & \textbf{12.20}\\
SimLingo & Action Dream & 7.81 & \textbf{1.04} & \textbf{48.70} & \textbf{35.94} & 23.51 & 11.94 & 62.30 & \textbf{65.04} & 12.15\\
LMDrive & CARLA-64K & \textbf{12.24} & 0.00 & 29.17 & 18.23 & 21.45 & 9.26 & \textbf{88.84} & 50.12 & 10.39\\
OneVL & NAVSIM & 0.52 & 0.00 & 25.00 & 3.65 & 16.69 & 7.05 & 77.72 & 50.69 & 4.87\\
Qwen-Drive-1.0-4B & WOD-E2E & 1.82 & 0.00 & 15.62 & 3.12 & 17.05 & 7.09 & 76.68 & 43.67 & 4.31\\
Impromptu-VLA & Impromptu & 0.26 & 0.00 & 19.79 & 6.25 & 14.48 & 4.61 & 77.43 & 42.02 & 3.87\\
\midrule
\multicolumn{11}{@{}l}{\textit{Diagnostic configuration: high-level bridge}}\\
Senna VLM bridge & DriveX & 12.24 & 0.00 & 27.60 & 3.65 & 25.93 & 8.12 & 80.85 & 68.38 & 9.89\\
\bottomrule
\end{tabularx}
\end{table*}

\begin{table*}[t]
% \vspace{-10pt}
\caption{\textbf{RealSim main results.} Scoring v0.2 on 30 PCUs from 23 reconstructed scenes: 60 rollouts and 30 pairs per system. All values are on a 0--100 scale; higher is better. Strict individual and paired success are shown first; L2D follows the component scores.}
\label{tab:realsim-main}\label{tab:real-main}
\centering\footnotesize
\setlength{\tabcolsep}{1.6pt}
\renewcommand{\arraystretch}{1.12}
\begin{tabular*}{\linewidth}{@{\extracolsep{\fill}}lrrrrrrrrr@{}}
\toprule
\multirow{2}{*}{System} & \multicolumn{2}{c}{Strict success} & \multicolumn{2}{c}{Relaxed proxy} & \multicolumn{2}{c}{Graded} & \multicolumn{2}{c}{Driving} & Summary\\
\cmidrule(lr){2-3}\cmidrule(lr){4-5}\cmidrule(lr){6-7}\cmidrule(lr){8-9}\cmidrule(l){10-10}
&SISR&CFPS&SemISR&\makecell{Sem-\\CFPS}&SCS&PCS&Safety&\makecell{Driving\\Score}&L2D\\
\midrule
Alpamayo2-Super & \textbf{13.33} & 0.00 & \textbf{18.33} & 0.00 & \textbf{21.43} & 0.00 & 93.33 & 97.10 & \textbf{9.22}\\
Alpamayo 1.5 & 11.67 & 0.00 & 15.00 & 0.00 & 18.12 & 1.11 & \textbf{100.00} & \textbf{97.17} & 8.38\\
SimLingo & 6.67 & \textbf{3.33} & 13.33 & \textbf{6.67} & 16.67 & \textbf{10.00} & 68.33 & 83.09 & 7.58\\
AutoVLA & 1.67 & 0.00 & 8.33 & \textbf{6.67} & 9.73 & 6.67 & 41.67 & 76.94 & 3.36\\
Qwen-Drive-1.0-4B  & 0.00 & 0.00 & 1.67 & 0.00 & 4.44 & 0.00 & 43.33 & 69.16 & 0.92\\
Impromptu-VLA  & 0.00 & 0.00 & 1.67 & 0.00 & 3.71 & 0.00 & 30.00 & 68.85 & 0.78\\
OneVL & 0.00 & 0.00 & 0.00 & 0.00 & 0.00 & 0.00 & 43.33 & 43.82 & 0.44\\
ReCogDrive & 0.00 & 0.00 & 0.00 & 0.00 & 0.00 & 0.00 & 33.33 & 43.57 & 0.38\\
\bottomrule
\end{tabular*}
% \vspace{-8pt}
\end{table*}

\textbf{Most successful requests have no successful counterpart.}
% 118 rollouts satisfy their full instructions, but only eight belong to four successful pairs.
Across the principal CARLA systems,  93.2\% of successes are one-sided: one request succeeds while its feasible alternative fails.
LMDrive achieves the highest SISR, 12.24\%, yet completes no pair.
Alpamayo~1.5 and SimLingo each complete only 1.04\% CFPS, leaving 27 of 31 and 26 of 30 individual successes unpaired.
These are genuine individual completions, but they are rarely accompanied by successful execution of the alternative request. The discrepancy also appears in reconstructed scenes: only one of 240 evaluated system--PCU pairs succeeds strictly, achieved by SimLingo, while Alpamayo2-Super reaches 13.33\% SISR with zero CFPS.

\textbf{Successful pairs demonstrate the capability that isolated successes leave untested.}
Alpamayo~1.5 completes both passing and following on CARLA B-027/p1, and SimLingo satisfies both requested speed bands on RealSim P22.
These positive examples show that the evaluated systems can complete matched alternatives in some contexts; the aggregate results show how rarely they do so across the benchmark.
Paired evaluation therefore distinguishes individual completion from demonstrated access to both requested behaviors.
Appendix~\ref{app:analysis} gives the count-level decomposition.

\subsection{Do driving scores establish instruction following?}
\label{sec:diagnostic-results}
\textbf{Strong driving performance can coexist with zero paired success.} In RealSim, Alpamayo~1.5 achieves 100\% Safety and 97.17 DrivingScore but zero strict or relaxed-proxy paired success; On P11, both Alpamayo~1.5 arms attain 100 DrivingScore, although the requested lane change fails (Figure~\ref{fig:measured-case-comparison}).
Conversely, SimLingo completes both P22 instructions with DrivingScores of 97.82 and 85.78: full instruction execution need not maximize reference-path advancement.
% On the separate FDQ routes, SimLingo and LMDrive attain 88.84\% and 80.77\% route completion despite core CFPS of 1.04\% and 0\%.
These results indicate that High scores on driving performance therefore do not certify instruction-specific completion.

\textbf{Avoiding danger and fulfilling a safe alternative are also distinct outcomes.}
In the six exploratory Safety-Q cases, LMDrive passes safety in five cases but completes only two safe-refusal responses.
Across the six principal systems, all three strict Safety-Q successes are refusals; none of the 24 fallback or constrained-alternative trials succeeds.
These diagnostics evaluate responses to conflicting requests separately from the feasible-request core, distinguishing avoidance of a prohibited action from execution of a useful replacement. Further details are provided in Appendix~\ref{app:fdq}, particularly Section~\ref{app:safety}.

\begin{figure}[t]
\centering\includegraphics[width=\linewidth]{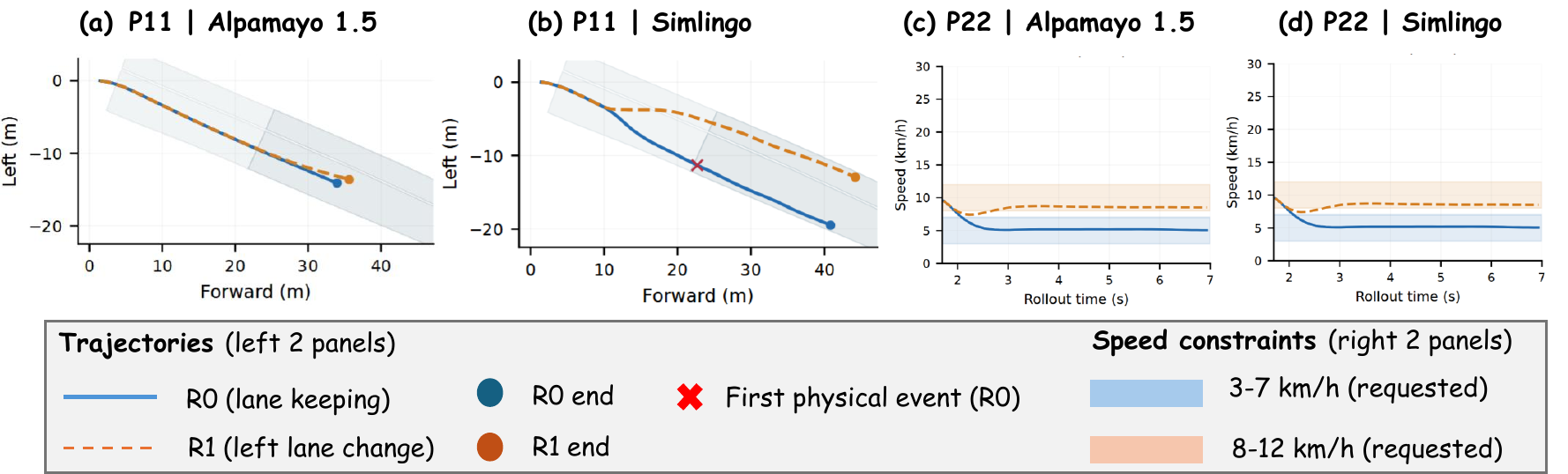}
\caption{\textbf{Advancement and behavioral response do not establish complete instruction execution.} (a,b) P11 requests lane keeping (R0, blue) or a left lane change (R1, orange dashed). Alpamayo~1.5 reaches 100 DrivingScore in both arms but succeeds only on R0.  P22 requests 3--7 or 8--12\,km/h. Shading marks the requested speed bands, \emph{not} their active time windows. Curves cover the full post-instruction interval; verification applies from entry into the annotated segment to exit and requires complete traversal and compliance at all active samples. SimLingo passes both criteria; Alpamayo~1.5 passes neither. Figure~\ref{fig:reference-P22} gives the segment geometry.}
\label{fig:measured-case-comparison}
\end{figure}
% \vspace{-5pt}
\subsection{How reliable is control across wording and task structure?}
% \vspace{-5pt}
\label{sec:wording-results}\label{sec:scope-results}
\textbf{Successful pairs do not persist across all tested language expressions.}
No CARLA system completes both arms of a PCU under all three meaning-preserving expressions.
A low overall outcome-change rate can conceal fragile successes: SimLingo changes on only 9 of 128 arms, but those nine account for 60\% of the 15 arms it ever solves.
Alpamayo~1.5's corresponding conditional rate is 75\%.
On A-010, SimLingo's arm-outcome patterns are \texttt{110} and \texttt{101}, so paired success occurs only for expression~0.
LMDrive, by contrast, completes 14 arms under every expression but never completes a pair.
Consistency on one request therefore does not establish consistent access to its alternative under the tested wordings.

\textbf{Observed long-horizon program success is confined to maintaining an invariant until an event.}
Across CARLA tracks, all seven strict successes in scope Long-horizon program occur in this family: six from LMDrive and one from SimLingo.
Ordered visits, conditional fallbacks, and reference-return programs yield no strict success in 504 rollouts.
The scope-level breakdown thus distinguishes limited success on one temporal structure from the broader ability to execute instruction-defined programs. Detailed results and cases appear in Appendices~\ref{app:analysis} and~\ref{app:case-studies}.
\subsection{Does the single--paired discrepancy extend to predicted trajectories?}
\label{sec:offline-results}
\textbf{The discrepancy also appears in open-loop settings without controller execution.}
We evaluate paired predicted trajectories on our test set developed from nuScenes, NAVSIM, and WOD-E2E~\citep{caesar2020nuscenes,dauner2024navsim,xu2026wod}, holding the recorded history context fixed while changing the instruction.
Open-loop binary success requires satisfying the requested criterion and excluding its alternative.
Figure~\ref{Fig:offline-main} reports within-dataset comparisons, with protocol and sample coverage in Appendix~\ref{app:offline}.
\begin{figure}[t]
\centering\includegraphics[width=\linewidth]{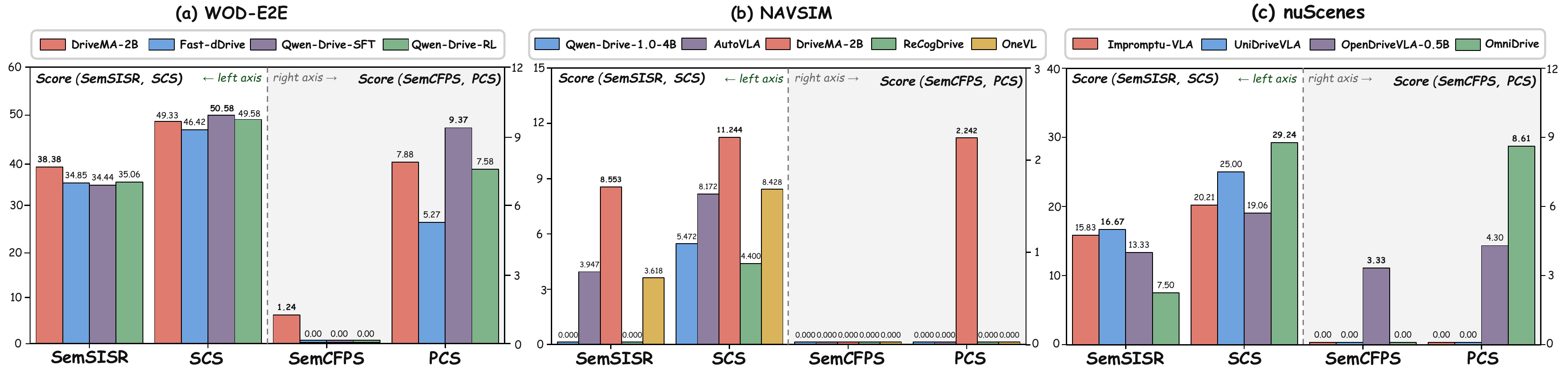}
\caption{\textbf{Open-loop instruction following track on three datasets.} Predicted trajectories are scored for exclusive binary semantic success (SemSISR/SemCFPS) and graded completion (SCS/PCS). All values use a 0--100 scale.}
\label{Fig:offline-main}
\end{figure}
On WOD-E2E, DriveMA-2B attains 38.38 SemSISR but 1.24 SemCFPS; All five NAVSIM systems have zero SemCFPS, with DriveMA-2B reaching the highest SCS/PCS of 11.244/2.242.
On nuScenes, OpenDriveVLA-0.5B reaches 3.33 SemCFPS, while all other reported systems score zero. A separate WOD-E2E geometry diagnostic characterizes weak trajectory differentiation.
Across 441 pairs, the two rated references differ by 3.962\,m mean ADE, whereas the four systems' paired predictions differ by only 0.178--0.355\,m on average.
These distances describe output sensitivity, not semantic correctness, and use a different cohort from the semantic-success analysis.
Without pooling scores or rankings across datasets, the open-loop results extend the central observation: individual completion is much more prevalent than successful completion of both requested alternatives.
\FloatBarrier
\vspace{-8pt}
\section{Related Work}
\label{sec:related}

\providecommand{\ltdyes}{\ding{51}}
\providecommand{\ltdno}{\textendash}

\begin{table}[t]
% \vspace{-10pt}
% \caption{\textbf{Comparison of driving evaluation protocols.}
% Checks indicate explicitly evaluated features, not model capabilities;
% \ltdno\ means not explicitly reported in the cited protocol.
% Action Dreaming denotes SimLingo's instruction evaluation, not its
% separate closed-loop driving tests.}
\caption{\textbf{Comparison of driving evaluation protocols.} S: synthetic worlds; R: real-log-derived contexts, not on-road tests.
 CL: sequential closed-loop execution.
 Matched: instruction variants share a scene or route context.
 Alt.: variants have distinct scored behavioral targets.
 Behavior check: instruction-specific acceptance beyond route-level metrics.
 Para.: explicit paraphrase evaluation.
 Paired success: both opposed, feasible requests must succeed.}
% }
\label{tab:benchmark-comparison}
\centering\footnotesize
\setlength{\tabcolsep}{2.2pt}
\renewcommand{\arraystretch}{1.15}
\begin{tabularx}{\linewidth}{@{}Xcc*{6}{c}@{}}
\toprule
& \multicolumn{2}{c}{Execution setting}
& \multicolumn{6}{c}{Instruction evaluation} \\
\cmidrule(lr){2-3}\cmidrule(l){4-9}
Benchmark / protocol
& World & CL
& \makecell{Language\\input}
& \makecell{Matched\\context}
& \makecell{Alt.\\targets}
& \makecell{Behavior\\check}
& \makecell{Para.\\test}
& \makecell{\textbf{Paired}\\\textbf{success}} \\
\midrule
nuPlan~\citeyear{caesar2021nuplan}
& R & \ltdyes & \ltdno & \ltdno & \ltdno
& \ltdno & \ltdno & \ltdno \\

Bench2Drive~\citeyear{jia2024bench2drive}
& S & \ltdyes & \ltdno & \ltdno & \ltdno
& \ltdno & \ltdno & \ltdno \\

NAVSIM v2~\citeyear{cao2025pseudo}
& R & \ltdno & \ltdno & \ltdno & \ltdno
& \ltdno & \ltdno & \ltdno \\
\midrule
LangAuto~\citeyear{shao2024lmdrive}
& S & \ltdyes & \ltdyes & \ltdno & \ltdno
& \ltdno & \ltdno & \ltdno \\

Action Dreaming~\citeyear{renz2025simlingo}
& S & \ltdno & \ltdyes & \ltdyes & \ltdyes
& \ltdyes & \ltdno & \ltdno \\

Drive4C~\citeyear{sohn2025drive4c}
& S & \ltdyes & \ltdyes & \ltdno & \ltdno
& \ltdyes & \ltdno & \ltdno \\

ICR-Drive~\citeyear{hamid2026icr}
& S & \ltdyes & \ltdyes & \ltdyes & \ltdno
& \ltdno & \ltdyes & \ltdno \\

% POINT~\citeyear{liu2026point}
% & R & \ltdyes & \ltdyes & \ltdno & \ltdno
% & \ltdyes & \ltdno & \ltdno \\
\midrule
\textbf{Language2Drive}
& \textbf{S+R} & \ltdyes & \ltdyes & \ltdyes & \ltdyes
& \ltdyes & \ltdyes & \ltdyes \\
\bottomrule
\end{tabularx}

\vspace{2pt}
% \begin{minipage}{\linewidth}
% \scriptsize
% S: synthetic worlds; R: real-log-derived contexts, not on-road tests.
% CL: sequential closed-loop execution.
% Matched: instruction variants share a scene or route context.
% Alt.: variants have distinct scored behavioral targets.
% Behavior check: instruction-specific acceptance beyond route-level metrics.
% Para.: explicit paraphrase evaluation.
% Paired success: both opposed, feasible requests must succeed.
% \end{minipage}
\end{table}
\textbf{VLA Models for Autonomous Driving.}
LMDrive and SimLingo connect visual observations and language to closed-loop driving, with SimLingo explicitly aligning language and actions~\citep{shao2024lmdrive,renz2025simlingo}. Recent developments include Qwen-Drive's unified perception--understanding--planning framework~\citep{zhou2026qwen}, Alpamayo~1.5's user-command-conditioned trajectory generation~\citep{wang2025alpamayo}, and Drive My Way's adaptation to persistent preferences and immediate style instructions~\citep{wang2026drive}. These methods motivate a separate evaluation question: can language reliably select and sustain the requested behavior among feasible alternatives?

\textbf{Language-Conditioned Driving Evaluation.}
nuPlan, Bench2Drive, and NAVSIM/v2 assess planning quality and safety~\citep{caesar2021nuplan,jia2024bench2drive,dauner2024navsim,cao2025pseudo}. LangAuto, and  Drive4C extend evaluation to language-guided navigation, capability-specific tasks, and open-ended instruction realization, respectively~\citep{shao2024lmdrive,sohn2025drive4c}. Action Dreaming evaluates alternative instruction-conditioned predictions within shared scenes~\citep{renz2025simlingo}, whereas ICR-Drive measures driving-score changes under matched instruction perturbations~\citep{hamid2026icr}. Complementing these protocols, \ltd\ requires \emph{both mutually exclusive, feasible requests} to satisfy their own spatial--temporal contracts in matched closed-loop runs, separating isolated completion from successful execution of both alternatives (Table~\ref{tab:benchmark-comparison}).

% \input{source/limitations}
% \vspace{-7pt}
\section{Conclusion}
% Do current driving VLAs reliably follow passenger instructions? In our evaluation, they sometimes complete individual requests, but rarely complete both feasible alternatives. Language2Drive measures this distinction with PCUs that hold the initial context fixed, machine-checkable trajectory criteria, and separate individual and paired metrics. Across closed-loop and open-loop tests, route advancement, partial completion, and visible language responses do not establish reliable instruction control. Evaluating selection, complete execution, and robustness across wording offers a concrete target for improving passenger control through language.
\vspace{2pt}
We introduce Language2Drive, a benchmark that makes reliable passenger control through language an explicit evaluation objective for autonomous-driving vision-language-action models. Its paired counterfactual framework combines matched initial conditions with machine-checkable spatial and temporal criteria, distinguishing individual instruction success from the ability to execute both requested alternatives. We instantiate this framework through an extended CARLA simulator with purpose-designed maps and curated reconstructed scenes in AlpaSim, complemented by open-loop evaluations developed from nuScenes, NAVSIM, and WOD-E2E. Our results reveal a substantial gap between driving performance and reliable instruction following. Across the six principal CARLA systems, 93.2\% of successful rollouts lack a successful paired counterpart, and strict paired success reaches at most 1.04\%. 
% Strong route advancement, partial completion, and visible responses to language frequently coexist with failures to fulfill the requested behavior. 
Success also remains fragile across paraphrases, while ordered, conditional, and reference-return programs expose persistent execution difficulties. These findings establish that producing a plausible driving trajectory is insufficient evidence of instruction following. By evaluating which alternative is selected, whether its requirements are fully executed, and whether success persists across equivalent expressions, Language2Drive provides a concrete basis for developing driving systems that passengers can direct consistently, precisely, and safely through language.

\label{main-text-end}
\FloatBarrier
\phantomsection
\addcontentsline{toc}{section}{References}
% Pre-generated bibliography: no BibTeX step is needed for this upload package.

\clearpage
\appendix
\section{Formal properties}
\label{app:theory}
\subsection{What a paired test rules out}
For PCU $i$, define the strict acceptance set for request $j$ by
\[
\mathcal A_{i,j}=\{\tau:O_{i,j}(\tau)(1-O_{i,1-j}(\tau))S_{i,j}(\tau)P_{i,j}(\tau)T_{i,j}(\tau)H_{i,j}(\tau)=1\}.
\]
The two sets are disjoint because each requires the paired-alternative criterion to fail. This remains true even if a permissive trajectory can satisfy both raw target predicates.

\begin{proposition}[Instruction-invariant execution]
\label{prop:invariant}
With identical non-language initialization and fully coupled execution randomness, an instruction-invariant system has zero strict CFPS and at most 50 SISR on balanced, mutually exclusive arm acceptances.
\end{proposition}
\begin{proof}
An instruction-invariant system receives the same effective inputs at initialization. The same control law and coupled randomness give identical actions and simulator transitions. Induction over time therefore gives the same executed trace $\tau$ under both requests. If both arms succeeded, $\tau$ would belong to the disjoint sets $\mathcal A_{i,0}$ and $\mathcal A_{i,1}$, a contradiction. At most one arm can succeed, so every pair contributes zero to CFPS and at most one half to SISR. A default behavior satisfying exactly one request in every pair attains the SISR bound.
\end{proof}

Coupling is a substantive assumption. An instruction-independent policy that independently chooses two valid behaviors with probabilities $q$ and $1-q$ in the two runs has expected paired success $q(1-q)$, up to 25\%, despite ignoring language. Coupling that behavior-selection randomness instead produces the same choice in both runs and zero CFPS. A shared seed label does not guarantee identical latent randomness across asynchronous workers or decoding implementations. The empirical scores establish execution of matched alternatives under the reported system; instruction-ablated repetitions would additionally test causal language dependence.

\subsection{Individual--paired identities and bounds}
For $x,y\in\{0,1\}$,
\[
\frac{x+y}{2}=xy+\frac{|x-y|}{2}.
\]
Averaging with the same PCU/expression weights proves Equation~\ref{eq:one-sided}. For $x,y\in[0,1]$,
\[
\min(x,y)=\frac{x+y-|x-y|}{2},
\]
which proves the SCS--PCS imbalance identity in Section~\ref{sec:metrics}. Since
\[
xy\le\frac{x+y}{2},\quad xy\ge x+y-1
\quad (x,y\in\{0,1\}),
\]
\[
\min(x,y)\le\frac{x+y}{2},\quad \min(x,y)\ge x+y-1
\quad (x,y\in[0,1]),
\]
each individual/paired metric family $(U,V)$ obeys
\begin{equation}
\max(0,2U-100)\le V\le U.
\label{eq:pair-bounds}
\end{equation}
This applies to $(\SISR,\CFPS)$, $(\SemISR,\SemCFPS)$, and $(\mathrm{SCS},\mathrm{PCS})$ when the instances and weights agree. Strict success implies the closed-loop relaxed semantic verdict, hence $\SISR\le\SemISR$ and $\CFPS\le\SemCFPS$. The graded map is a separate rule, so no analogous ordering with SemISR is assumed. PCS reduces to the corresponding binary paired rate only when its completion scores are that same binary success indicator.

\subsection{All-expression paired success}
\label{app:paraphrase}
CARLA fixes three expressions per arm and one seed. Let $b_{ir}=s_{ir,0}s_{ir,1}$. We report
\[
W=\frac{100}{N}\sum_i\min_{r\in\{0,1,2\}}b_{ir}.
\]
Both arms must succeed under every tested expression. This is not the minimum of three dataset-average CFPS values. It satisfies $W\le\CFPS$ for the uniform expression weighting used here. A zero $W$ is compatible with occasional successful pairs; RealSim's one-expression evaluation does not estimate wording robustness.

% Requires amsmath, amssymb, amsthm.
% Add \newtheorem{proposition}{Proposition} to the preamble once.
\subsection{Properties of the L2D aggregation}
\label{app:l2d-properties}
Let $a=(a_1,a_2,a_3)\in[0,100]^3$,
$M=(a_1+a_2+a_3)/3$, and
$L(a)=\bigl(\sum_{k=1}^{3}\sqrt{a_k}/3\bigr)^2$.
For a fixed driving-credit cap $b\in(0,100)$, define
\[
 F_b(a,d)=d\left[b+(100-b)\frac{L(a)}{100}\right],\qquad d\in[0,1].
\]
The reported score fixes $b=1$.

\begin{proposition}[Balance without a family-level zero veto]
For every $a\in[0,100]^3$,
\[
L(a)=M-\frac19\sum_{k<h}(\sqrt{a_k}-\sqrt{a_h})^2,
\qquad \frac{M}{3}\le L(a)\le M.
\]
Thus $L(a)=0$ if and only if all three families are zero.
For a fixed $M$, the maximum is attained precisely when all families are equal.
\end{proposition}
\begin{proof}
Expand $(\sum_k\sqrt{a_k})^2$ to obtain the identity.
Its nonnegative cross terms give $L(a)\ge\sum_k a_k/9=M/3$.
Cauchy--Schwarz gives $(\sum_k\sqrt{a_k})^2\le3\sum_k a_k$, hence $L(a)\le M$.
Equality in the upper bound requires equal square roots, proving the balance statement.
\end{proof}

\begin{proposition}[Monotonicity and weaker-family priority]
$L$ is continuous and strictly increasing in each family coordinate.
For positive $a_i,a_j$,
\[
\frac{\partial L}{\partial a_i}
=\frac{\sum_k\sqrt{a_k}}{9\sqrt{a_i}},\qquad
\frac{\partial L/\partial a_i}{\partial L/\partial a_j}
=\sqrt{\frac{a_j}{a_i}}.
\]
Thus an equal marginal increase receives greater weight in the weaker family.
At fixed $d>0$, $F_b$ preserves strict coordinatewise monotonicity.
\end{proposition}
\begin{proof}
Each square root is continuous and strictly increasing, and the final square is strictly
increasing on nonnegative inputs. Differentiation gives the displayed expressions.
The multiplier $(100-b)d/100$ is positive when $d>0$.
The continuity and monotonicity statements also hold at zero, without requiring a finite derivative there.
\end{proof}

\begin{proposition}[Bounded driving-only credit and exact separation]
The composite satisfies $0\le F_b\le100$, $F_b(a,0)=0$, and
$F_b((100,100,100),1)=100$.
If $a=0$, then $F_b(0,d)=bd\le b$; distinct driving scores remain distinguishable.
Moreover,
\[
F_b-bd=(100-b)\frac{L(a)}{100}d,
\qquad
F_b\le\min\left\{100d,\ b+\left(1-\frac b{100}\right)L(a)\right\}.
\]
In particular, $F_b>b$ implies nonzero language completion credit.
\end{proposition}
\begin{proof}
By the preceding bounds, $0\le L(a)/100\le1$.
Consequently the bracket in $F_b$ lies in $[b,100]$.
Substitution proves the boundary statements and decomposition.
The two upper bounds follow respectively from the bracket being at most $100$ and from $d\le1$.
\end{proof}

\begin{proposition}[Limited perturbation of the multiplicative score]
Let $T=L(a)d$ be the score without driving-only credit. Then
\[
F_b=\left(1-\frac b{100}\right)T+bd,
\qquad 0\le F_b-T=bd\left(1-\frac{L(a)}{100}\right)\le b.
\]
Therefore any pairwise score gap greater than $b$ under $T$ cannot be reversed by this change.
Also $|F_b-F_{b'}|\le|b-b'|$ for fixed $a,d$.
\end{proposition}
\begin{proof}
The identities follow by expansion. For systems $A,B$, write
$F_{b,A}=T_A+e_A$ and $F_{b,B}=T_B+e_B$, with $e_A,e_B\in[0,b]$.
If $T_A-T_B>b$, then $F_{b,A}-F_{b,B}\ge T_A-T_B-b>0$.
Finally $\partial F_b/\partial b=d(1-L(a)/100)\in[0,1]$ proves the cap-sensitivity bound.
\end{proof}

\begin{proposition}[Unique fusion under explicit endpoint and linearity conditions]
Let $G(\ell,d)$ obey $G(\ell,d)=dG(\ell,1)$, let $G(\ell,1)$ be affine in $\ell$,
and require $G(0,1)=b$ and $G(1,1)=100$.
Then necessarily $G(\ell,d)=d[b+(100-b)\ell]$.
\end{proposition}
\begin{proof}
Write $G(\ell,1)=u+v\ell$. The endpoint conditions imply $u=b$ and $v=100-b$.
The first condition then determines $G(\ell,d)$ for every $d$.
\end{proof}
These results establish the specified range, balance, partial-credit, and cap properties.
They do not identify a universally optimal $b$, prove causal use of language,
or replace the individual and paired success metrics.

\subsection{Complementarity and fixed evaluation scale}
Write $\ell=L(a)/100$. Differentiating $F_b=d[b+(100-b)\ell]$ gives
\[
\frac{\partial F_b}{\partial\ell}=(100-b)d,\qquad
\frac{\partial F_b}{\partial d}=b+(100-b)\ell,\qquad
\frac{\partial^2F_b}{\partial\ell\partial d}=100-b>0.
\]
Language improvement contributes more when driving quality is higher, and driving improvement contributes more when language completion is higher. Driving affects both terms; the one-point cap is not a claim that driving has only 1\% weight. The arithmetic driving component is not a hard veto on each safety metric, and L2D is not a rollout-level joint success probability. PSM is measured on independent routes.

For fixed tasks, scorer, $q=1/2$, and $b=1$, a system's score depends only on its own component metrics. Adding other systems does not change it. Nonlinear aggregation is applied \emph{after} the population metrics are computed; averaging a per-PCU nonlinear score would define a different metric. The proofs establish the chosen scoring properties, not statistical significance, causal language use, or a universal ordering of driving systems.

\clearpage
\section{Metric Definitions and Properties}
\label{app:metrics}

\subsection{Evaluation units and strict verdicts}
\label{app:metric-units}

For PCU $i$, expression $r$, and arm $j$, let $\tau_{ir,j}$
denote the executed trajectory.
The two arms share their initial non-language context but are
executed separately under their respective instructions.
Let $O_{i,j}(\tau)$ indicate complete satisfaction of the requested
task criterion.
Using the verifier's safety, progress, timing, and stability gates,
the strict verdict is
\begin{equation}
\begin{aligned}
s_{ir,j}
={}&
O_{i,j}(\tau_{ir,j})
\bigl[1-O_{i,1-j}(\tau_{ir,j})\bigr]
\\
&\times
S_{i,j}(\tau_{ir,j})
P_{i,j}(\tau_{ir,j})
T_{i,j}(\tau_{ir,j})
H_{i,j}(\tau_{ir,j}).
\end{aligned}
\label{eq:metric-strict}
\end{equation}
Each factor is binary, and the gates use the applicable
track-specific rules.
Alternative rejection makes the two strict acceptance sets disjoint.
Persistent requirements apply to a reached, nonempty interval that
must be completed; stopping before that interval does not satisfy
the requirement.
The task specifications retain the exact spatial tolerances,
deadlines, dwell requirements, and event-ordering rules.

All language scores use the expression-then-PCU averaging operator
defined in Section~\ref{sec:metrics}.
The scheduled pair set is fixed within each protocol, and
unsuccessful executions remain in its denominator.
Equal PCU weighting prevents a task with more expressions from
receiving greater weight.
It is equivalent to pooling all paired instances only when each
PCU has the same number of expressions.
The evaluated CARLA protocol has $N=64$ and $R_i=3$; RealSim has
$N=30$ and $R_i=1$.
Duplicate retries and smoke tests do not create additional
evaluation instances.

\subsection{Relaxed proxies and graded completion}
\label{app:metric-partial}

\paragraph{Relaxed semantic proxies.}
Let $\widetilde O_{i,j}$ denote a predefined task-specific
relaxation of the requested criterion.
The relaxed verdict and its aggregate rates are
\begin{equation}
\begin{aligned}
m_{ir,j}
&=
\mathbf{1}\!\left\{
O_{i,j}(\tau_{ir,j})=1
\ \text{or}\
\widetilde O_{i,j}(\tau_{ir,j})=1
\right\},
\\
\mathrm{SemISR}
&=
100\left\langle
\frac{m_{ir,0}+m_{ir,1}}{2}
\right\rangle,
\qquad
\text{Sem-CFPS}
=
100\langle m_{ir,0}m_{ir,1}\rangle.
\end{aligned}
\label{eq:metric-relaxed}
\end{equation}
These are deterministic task rules, not free-form language-model
judgments.
CARLA proxies can relax target geometry, heading, or stopped dwell;
sequence and reference proxies retain ordered event checks.
RealSim uses its own task-specific rules: some criteria relax
timeliness or dwell, whereas scoped speed tasks retain complete
entry-to-exit traversal and compliance at every active sample.
The proxies do not impose the same universal alternative-rejection
and execution-gate conjunction as Equation~\ref{eq:metric-strict}.
Passing both relaxed criteria therefore need not demonstrate
exclusive, complete execution of both requests.
A relaxed verdict neither terminates a rollout nor reduces its
opportunity to reach strict success.

\paragraph{Graded completion.}
The task-specific map $C_{i,j}$ assigns $c_{ir,j}\in[0,1]$ to
the execution of its own request.
Equation~\ref{eq:graded} averages the two arms for SCS and takes
the weaker arm for PCS.
These scores preserve partial completion without interpreting it
as a binary success probability.
The graded map is separate from the relaxed predicate: neither
$\mathrm{SCS}\leq\mathrm{SemISR}$ nor the reverse is assumed.
PCS reduces to CFPS only when the per-arm graded scores equal
the strict binary verdicts themselves.
Task-specific maps and tolerances remain part of the evaluator
specification.

\subsection{Individual--paired identities and interpretation}
\label{app:metric-identities}

For binary $x,y$,
\begin{equation}
\frac{x+y}{2}
=
xy+\frac{|x-y|}{2}.
\label{eq:metric-binary-identity}
\end{equation}
Applying the same averaging weights to both sides proves
Equation~\ref{eq:one-sided}.
The SISR--CFPS difference is thus exactly half the percentage of
one-sided paired instances.
For graded scores $x,y\in[0,1]$,
\begin{equation}
\min(x,y)
=
\frac{x+y-|x-y|}{2},
\end{equation}
which yields
\begin{equation}
\mathrm{SCS}-\mathrm{PCS}
=
50\langle|c_{ir,0}-c_{ir,1}|\rangle.
\label{eq:metric-graded-gap}
\end{equation}
This difference quantifies completion imbalance between alternatives,
not a particular internal failure mechanism.

\paragraph{Bounds.}
For each individual--paired family $(U,V)$,
\begin{equation}
\max(0,2U-100)
\leq V\leq U.
\label{eq:metric-bounds}
\end{equation}
This follows from $xy\leq(x+y)/2$ and $xy\geq x+y-1$ for
binary verdicts, and the analogous inequalities for $\min(x,y)$
on $[0,1]^2$.
It applies to
$(\mathrm{SISR},\mathrm{CFPS})$,
$(\mathrm{SemISR},\text{Sem-CFPS})$, and
$(\mathrm{SCS},\mathrm{PCS})$
when their instances and weights agree.
Strict success implies the relaxed verdict, giving
$\mathrm{SISR}\leq\mathrm{SemISR}$ and
$\mathrm{CFPS}\leq\text{Sem-CFPS}$.
These inequalities do not extend automatically to the separately
defined graded scores.

\paragraph{The denominator of one-sided success.}
The percentage of \emph{pairs} with one successful arm,
$q_{\mathrm{one}}$, differs from the percentage of
\emph{successful arms} lacking a successful counterpart.
With equal instance weights, let $n_2$ and $n_1$ count pairs
with two and one successful arms.
There are $2n_2+n_1$ successful arms, of which $n_1$ are unpaired.
More generally, under the same metric weights and
$\mathrm{SISR}>0$, the latter percentage is
\begin{equation}
100\frac{q_{\mathrm{one}}}{2\,\mathrm{SISR}}
=
100\left(
1-\frac{\mathrm{CFPS}}{\mathrm{SISR}}
\right).
\label{eq:metric-unpaired-share}
\end{equation}
This is an interpretation of the existing success decomposition,
not an additional primary metric.
It is undefined when no instruction succeeds.
For pooled counts, the numerator and denominator must be pooled
before taking the ratio; averaging per-system percentages generally
gives a different quantity.

\subsection{Consistency across expressions}
\label{app:metric-wording}

% Let $b_{ir}=s_{ir,0}s_{ir,1}$.
% Equation takes the minimum over
% expressions \emph{within each PCU} before averaging over PCUs.
% Consequently,
% \begin{equation}
% 0\leq W\leq\mathrm{CFPS}.
% \end{equation}
% $W$ is not the minimum of the three dataset-average CFPS values:
% different expressions can succeed on different PCUs.
% A zero $W$ is compatible with occasional paired successes, but
% indicates that none persists across all three tested expressions.

\paragraph{Outcome changes among successful arms.}
For arm $(i,j)$, define
\begin{equation}
\begin{aligned}
f_{i,j}
&=
\mathbf{1}\!\left\{
0<\sum_{r=0}^{2}s_{ir,j}<3
\right\},
\\
a_{i,j}
&=
\mathbf{1}\!\left\{
\sum_{r=0}^{2}s_{ir,j}>0
\right\}.
\end{aligned}
\end{equation}
We report
$100\,\frac{\sum_{i,j}f_{i,j}}{\sum_{i,j}a_{i,j}}$
as Flip/any.
The percentage is not applicable when no arm succeeds.
Conditioning on at least one success separates fragile successes
from consistently unsuccessful arms.
These statistics describe execution under the tested wordings.
With one evaluation seed, they do not isolate wording effects
from residual execution variability.

\subsection{Track-specific advancement and safety}
\label{app:metric-driving}

\paragraph{CARLA core advancement.}
For the set $\mathcal V$ of valid core traces, let
$\pi_v(t)\in[0,1]$ be normalized projection onto the task-specific
private reference path.
Then
\begin{equation}
\mathrm{DriveScore}
=
\frac{100}{|\mathcal V|}
\sum_{v\in\mathcal V}\max_t\pi_v(t).
\label{eq:metric-carla-drive}
\end{equation}
The maximum records the furthest attained reference progress;
it does not enforce target exclusivity, ordering, dwell, or the
complete instruction contract.
This valid-trace denominator differs conceptually from the
scheduled denominator for language scores.
All seven displayed CARLA configurations have 384 valid core traces.

\paragraph{RealSim advancement and safety.}
Let $\lambda_v$ be the length of the frozen arm-specific reference
path truncated at the task deadline $T_v$.
For projected arc length $a_v(t)$ and distance $e_v(t)$ to that
path, define
\begin{equation}
\begin{aligned}
g_v
&=
\operatorname{clip}\!\left(
\frac{
\max\bigl(
\{0\}\cup
\{a_v(t):t\leq T_v,\ e_v(t)\leq5\,\mathrm m\}
\bigr)
}{
\lambda_v
},
0,1
\right),
\\
\mathrm{DrivingScore}
&=
\frac{100}{60}\sum_{v=1}^{60}g_v.
\end{aligned}
\label{eq:metric-realsim-drive}
\end{equation}
Failed or unavailable slots contribute no advancement credit
rather than being removed from the fixed denominator.
Safety separately records collision/off-road-free execution
under scoring v0.2.
The strict Progress gate is not this continuous score:
RealSim uses a 10\,m threshold over the available rollout,
including the pre-instruction segment, while CARLA's default
gate requires at least 0.05 evaluator progress.
CARLA's strict safety gate also excludes red-light events;
RealSim's requested speed constraints remain language criteria.

\paragraph{Independent driving routes.}
FDQ evaluates 24 independent routes, one per ability group.
Its strict pass requires a Completed/Perfect status and no
recorded infractions.
The accompanying continuous score is
\begin{equation}
\mathrm{DS}_v
=
\mathrm{RC}_v\mathrm{IP}_v,
\qquad
\mathrm{DS}
=
\frac{1}{24}
\sum_{v=1}^{24}\mathrm{RC}_v\mathrm{IP}_v.
\label{eq:metric-fdq-driving}
\end{equation}
RC and DS use a 0--100 scale; the infraction penalty IP uses 0--1.
Multiplying mean RC by mean IP does not in general recover mean DS.
These routes supply complementary driving diagnostics, not
additional instruction PCUs.

\subsection{The L2D summary and its properties}
\label{app:metric-l2d}

\paragraph{Definition.}
L2D combines already aggregated component metrics for compact
within-track reporting.
Define the three language-family scores and their balanced aggregate:
\begin{equation}
\begin{aligned}
a_1
&=
\tfrac12(\mathrm{SISR}+\mathrm{CFPS}),
\\
a_2
&=
\tfrac12(\mathrm{SemISR}+\text{Sem-CFPS}),
\\
a_3
&=
\tfrac12(\mathrm{SCS}+\mathrm{PCS}),
\\
L_{\mathrm{bal}}
&=
\left(
\frac{\sqrt{a_1}+\sqrt{a_2}+\sqrt{a_3}}{3}
\right)^2.
\end{aligned}
\label{eq:app-metric-lbal}
\end{equation}
All $a_k$ and $L_{\mathrm{bal}}$ lie in $[0,100]$.
The track-specific driving component is
\begin{equation}
D_t
=
\begin{cases}
\tfrac12(\mathrm{PSM}+\mathrm{DriveScore}),
& t=\mathrm{CARLA},
\\
\tfrac12(\mathrm{Safety}+\mathrm{DrivingScore}),
& t=\mathrm{RealSim},
\end{cases}
\end{equation}
and the reported summary is
\begin{equation}
\mathrm{L2D}_t
=
\frac{D_t}{100}
\bigl(1+0.99L_{\mathrm{bal}}\bigr).
\label{eq:app-l2d}
\end{equation}
The square-root mean retains partial credit while favoring balance
across language families.
With zero credit in all three language families, L2D is at most
one point.
This is a summary of distinct measurements, not a substitute for
strict success or a rollout-level joint-success probability.

\paragraph{Balance and monotonicity.}
Write $L(\mathbf a)=L_{\mathrm{bal}}$ and
$M=(a_1+a_2+a_3)/3$.
Expanding the square yields
\begin{equation}
L(\mathbf a)
=
M-\frac19
\sum_{k<h}(\sqrt{a_k}-\sqrt{a_h})^2,
\qquad
\frac M3\leq L(\mathbf a)\leq M.
\label{eq:metric-lbal-properties}
\end{equation}
Nonnegative cross terms give the lower bound; Cauchy--Schwarz
gives the upper bound.
At fixed $M$, equality in the upper bound holds when all families
are equal.
Also, $L(\mathbf a)=0$ if and only if all three family scores
are zero, so a single zero family does not veto all partial credit.
The aggregate is continuous and strictly increasing in every
family score.
For positive $a_i,a_j$,
\begin{equation}
\frac{\partial L}{\partial a_i}
=
\frac{\sum_k\sqrt{a_k}}{9\sqrt{a_i}},
\qquad
\frac{\partial L/\partial a_i}{\partial L/\partial a_j}
=
\sqrt{\frac{a_j}{a_i}}.
\end{equation}
Thus an equal marginal improvement receives greater weight in
the weaker family.
Monotonicity also holds at zero, without requiring a finite
derivative there.

\paragraph{Bounded driving-only credit.}
For $d=D_t/100$ and a general cap $\beta\in(0,100)$, define
\begin{equation}
F_\beta(\mathbf a,d)
=
d\left[
\beta+(100-\beta)\frac{L(\mathbf a)}{100}
\right].
\label{eq:metric-general-fusion}
\end{equation}
Equation~\ref{eq:app-l2d} fixes $\beta=1$.
The bounds on $L$ imply
$0\leq F_\beta\leq100$,
$F_\beta(\mathbf a,0)=0$, and
$F_\beta((100,100,100),1)=100$.
Moreover,
\begin{equation}
\begin{aligned}
F_\beta(\mathbf 0,d)
&=
\beta d\leq\beta,
\\
F_\beta-\beta d
&=
(100-\beta)\frac{L(\mathbf a)}{100}d,
\\
F_\beta
&\leq
\min\left\{
100d,\,
\beta+\left(1-\frac{\beta}{100}\right)L(\mathbf a)
\right\}.
\end{aligned}
\end{equation}
A score above the driving-only cap therefore requires nonzero
language-family credit, which can include partial rather than
strict completion.

\paragraph{Cap sensitivity and the specified fusion rule.}
Relative to the multiplicative score $T=L(\mathbf a)d$,
\begin{equation}
F_\beta
=
\left(1-\frac{\beta}{100}\right)T+\beta d,
\qquad
0\leq F_\beta-T
=
\beta d\left(1-\frac{L(\mathbf a)}{100}\right)
\leq\beta.
\end{equation}
Consequently, a pairwise gap greater than $\beta$ under $T$
cannot be reversed by this bounded addition.
At fixed $\mathbf a,d$,
$|F_\beta-F_{\beta'}|\leq|\beta-\beta'|$.
Writing $\ell=L(\mathbf a)/100$, the form
$G(\ell,d)=d[\beta+(100-\beta)\ell]$
is uniquely determined by proportionality in $d$, affinity in
$\ell$ at $d=1$, and endpoints
$G(0,1)=\beta$ and $G(1,1)=100$.
These are explicit design conditions, not a claim that the cap
or fusion rule is universally optimal.

\paragraph{Interpretation and aggregation order.}
The cross-partial derivative is
$\partial^2F_\beta/(\partial\ell\,\partial d)=100-\beta>0$:
improvements in either component contribute more when the other
is larger.
The one-point cap is therefore not a claim that driving has only
one-percent weight.
The arithmetic driving component is not a hard veto on every
safety metric, and CARLA PSM is measured on separate routes.
Nonlinear aggregation is applied \emph{after} computing population
component scores; averaging nonlinear per-PCU summaries would
define a different metric.
For fixed tasks, scorer, and constants, adding other evaluated
systems does not change a system's score.
The reported tables compute L2D from displayed two-decimal
components using high-precision intermediate arithmetic and final
display rounding.
CARLA and RealSim retain separate rankings.
The properties above neither establish statistical separation
between nearby scores nor identify causal use of language.

\subsection{Open-loop semantic success}
\label{app:metric-openloop}

Open-loop evaluation checks predicted trajectories without
controller execution.
For a binary task criterion $E_{i,j}$, define exclusive semantic
success as
\begin{equation}
B_{ir,j}
=
E_{i,j}(\tau_{ir,j})
\bigl[1-E_{i,1-j}(\tau_{ir,j})\bigr].
\end{equation}
The corresponding metrics are
\begin{equation}
\mathrm{SemSISR}
=
100\left\langle
\frac{B_{ir,0}+B_{ir,1}}{2}
\right\rangle,
\qquad
\mathrm{SemCFPS}
=
100\langle B_{ir,0}B_{ir,1}\rangle.
\label{eq:metric-openloop}
\end{equation}
The names remain distinct from closed-loop SemISR and Sem-CFPS
because their acceptance rules differ: both open-loop arms
explicitly reject their alternative.
These scores evaluate semantic properties of predictions, not
realized closed-loop safety or execution.

SCS and PCS use Equation~\ref{eq:graded} with the open-loop task maps.
The graded map assigns one when all binary task constraints are
met; otherwise, native tolerances or acceptance intervals determine
linear decay, with conjunctions taking the minimum.
Lane tasks use footprint overlap and endpoint progress, while
stopping tasks combine target distance and stopping speed.
Ordered-event tasks score zero when neither event occurs, 0.5
when one occurs, and use the event interval when both occur.
Other tasks retain their speed, distance, and relative-position
conditions.
Invalid or unparseable predictions receive zero completion credit.
Scores and sample coverage are reported within each dataset
without a pooled cross-dataset rank.

\subsection{Additional driving and behavioral diagnostics}
\label{app:metric-diagnostics}

\paragraph{Supplementary FDQ measurements.}
The auxiliary per-route statistics use the evaluator's event
counts, reference curve, normalized exceedances, and recovery
quantities:
\begin{align}
\mathrm{SI}_v
&=
\exp[
-(3N_{\mathrm{ped},v}
+2N_{\mathrm{veh},v}
+N_{\mathrm{layout},v})
],
\\
\mathrm{GPE}_v
&=
r_v\exp\!\left[
-0.5\max\!\left(
0,\frac{T_v}{T_{\mathrm{ref}}(r_v)}-1
\right)
\right],
\quad
r_v=\mathrm{RC}_v/100,
\\
\mathrm{PSM}_v
&=
\frac{
\sum_t\Delta t\,
\operatorname{clip}(
\mathrm{TTC}_{\min,v}(t)/3,0,1
)
}{
\sum_t\Delta t
},
\\
\mathrm{MDQ}_v
&=
\exp\!\left[
-\frac16\sum_{q=1}^{6}E_{v,q}
\right],
\\
\mathrm{RLS}_v
&=
\frac{1}{|\mathcal J_v|}
\sum_{j\in\mathcal J_v}
u_j\exp(-t_{\mathrm{recover},j}/5).
\end{align}
These auxiliary scores are retained on their original 0--1 scales
with equal-route means; the main-table PSM is
$100\,(24^{-1}\sum_{v=1}^{24}\mathrm{PSM}_v)$.
Infinite TTC maps to one after clipping, whereas missing TTC is
not treated as infinity.
PSM uses the recorded constant-velocity oriented-box TTC model
with a 75\,m observation radius and 0.2\,m clearance.
GPE is zero at $r_v=0$, and $T_{\mathrm{ref}}$ is a
privileged-reference curve rather than a mathematical optimum.
RLS is unavailable when no recovery opportunity exists; its
reported mean is accompanied by coverage.

MDQ uses a seven-point centered mean and finite differences for
longitudinal and lateral acceleration, yaw rate, yaw acceleration,
longitudinal jerk, and jerk magnitude.
The acceleration ranges are $[-4.05,2.40]$ and
$[-4.89,4.89]\,\mathrm{m/s^2}$ for longitudinal and lateral
motion, respectively.
Yaw rate and yaw acceleration use
$[-0.95,0.95]\,\mathrm{rad/s}$ and
$[-1.93,1.93]\,\mathrm{rad/s^2}$.
Longitudinal jerk and jerk magnitude use
$[-4.13,4.13]$ and $[0,8.37]\,\mathrm{m/s^3}$, respectively.
Normalized exceedances are clipped to $[0,1]$.
These are engineering thresholds, not independently established
human-expert comfort standards.
Stationary behavior can score well on SI and MDQ, so these
measures must be read alongside advancement.
They do not change FDQ's strict pass criterion.

\paragraph{Safety-conflict diagnostics.}
Safety-Q reports behavior-based strict completion, relaxed
completion, safety, and unsafe obedience over six cases.
Fallback, refusal, and constrained-alternative completion each
use their own two-case denominator.
Unsafe obedience records the prohibited opposite-contract outcome;
it does not additionally require a separate unsafe flag.
Avoiding the prohibited action and completing a useful safe
replacement remain distinct outcomes.
No clarification-hold or benign matched requests are included,
so clarification and unnecessary-refusal rates are not applicable.
These diagnostics do not evaluate the quality of a textual
refusal explanation.

\paragraph{Trajectory-response and geometric diagnostics.}
For RealSim, complete paired trajectories are sampled at 101
common times from instruction issuance at 1.7\,s to their shared
horizon.
Mean position separation and mean absolute speed difference are
computed per pair and then summarized by the median.
Spatial statistics use complete valid pairs; speed statistics use
the specified speed-task subset, with coverage reported separately.
For WOD-E2E, the geometric criterion requires each prediction to
be closer in ADE to its own reference than to the alternative
by more than 0.5\,m.
These are diagnostics of response magnitude or reference alignment,
not replacements for instruction-specific success.
Their valid-trajectory cohorts remain separate from the scheduled
semantic-success denominators.

\subsection{What paired success establishes}
\label{app:metric-interpretation}

\paragraph{Instruction-invariant execution under coupled randomness.}
Let $\mathcal A_{i,j}$ be the strict acceptance set induced by
Equation~\ref{eq:metric-strict}.
Alternative rejection gives
$\mathcal A_{i,0}\cap\mathcal A_{i,1}=\varnothing$.
With identical non-language initialization and fully coupled
execution randomness, an instruction-invariant system produces
the same executed trajectory under both requests.
That trajectory cannot lie in both acceptance sets.
Such a system therefore has zero CFPS and at most 50 SISR on
balanced pairs.
The bound is attained when the same behavior satisfies exactly
one request in every pair.

\paragraph{Stochastic execution.}
Coupling is substantive, not merely a shared seed label.
An instruction-independent policy that independently selects the
two valid behaviors with probabilities $p$ and $1-p$ in the two
runs has paired-success probability $p(1-p)$ and expected CFPS
$100p(1-p)\leq25$.
Coupling its behavior-selection randomness instead makes it
choose the same behavior in both runs and yields zero paired
success.
Asynchronous execution and decoding can prevent a shared seed
label from coupling all relevant randomness.
Empirical CFPS therefore establishes successful matched execution
under the reported system and protocol; instruction-ablated
repetitions would additionally test causal language dependence.

\paragraph{Statistical interpretation.}
The current model comparisons use one evaluation seed per track.
CARLA expressions are clustered within PCUs, and several RealSim
PCUs share a reconstructed scene.
Uncertainty estimates must respect these units rather than treat
all arms and expressions as independent samples.
Aggregate scores alone do not provide such estimates, and small
score differences are descriptive rather than evidence of
statistical separation.
The reported metrics do not substitute for unperformed instruction
ablations, repeated-seed tests, or input/controller interventions.
\clearpage
\addtocontents{toc}{\protect\newpage}
\section{Task specification and qualification}
\label{app:maps}
\subsection{Environments, cohorts, and language coverage}
CARLA-Language2Drive adds six purpose-designed environment types to CARLA: underground parking, a port, Chinese city roads, a US city with detailed signage, a campus, and a highway. The evaluated base spans the ten map identifiers in Table~\ref{tab:maps}: five Town01--Town05 layouts and five custom non-Town assets. It is a subset of the simulator distribution, not an evaluation of every newly added map.
\begin{table}[!htbp]
\caption{Actual CARLA base map manifest. These ten evaluated layouts include five custom non-Town assets. The distribution-level six-environment description and this evaluation manifest have different scopes.}\label{tab:maps}
\centering\small
\setlength{\tabcolsep}{3.5pt}
\begin{tabular}{@{}lrr@{}}
\toprule
Map identifier & PCUs & Core rollouts\\
\midrule
American-city0428 & 11 & 66\\
Town01 & 2 & 12\\
Town02 & 1 & 6\\
Town03 & 10 & 60\\
Town04 & 4 & 24\\
Town05 & 2 & 12\\
hanghailou3 & 1 & 6\\
hubeiwenli4 & 2 & 12\\
parkinglotfish-new3 & 24 & 144\\
yangluogang1 & 7 & 42\\
\bottomrule
\end{tabular}
\end{table}

\begin{table}[!htbp]
\caption{Compiled source suite and nested CARLA cohorts. Columns count PCUs and potential Core executions; only the base single-seed cohort is the complete  comparison.}\label{tab:scales}
\centering\small
\setlength{\tabcolsep}{3.5pt}
\begin{tabular}{@{}lrrrr@{}}
\toprule
Cohort & PCUs & Arms & One seed & Three seeds\\
\midrule
Source & 128 & 256 & 768 & 2304\\
Small & 32 & 64 & 192 & 576\\
Base & 64 & 128 & 384 & 1152\\
Large (loadable) & 119 & 238 & 714 & 2142\\
\bottomrule
\end{tabular}
\end{table}

The source suite contains 128 PCUs, 32 per scope, with 256 arms and 768 arm--expression instances per seed. Nine Town06 PCUs are excluded because that map is not loadable in the evaluation runtime, leaving 119. Stable within-scope sampling selects nested 32-PCU and 64-PCU tiers. For a scope with $n_l$ loadable PCUs in source order, base selects indices $\lfloor jn_l/16\rfloor$, $j=0,\ldots,15$; small takes every second selected entry. The evaluated base has 14 development and 50 test PCUs. Development/test counts are A: 4/12, B: 3/13, C: 3/13, and D: 4/12. All arms and expressions of a PCU remain in the same split.

The split audit checks semantic entities, physical topology, and actor programs: 217, 158, and 13 recorded keys, with no cross-split conflicts in those sets. Controlled English templates substitute directions, entity names, speeds, and stage identities. The base contains 384 instruction entries and 216 distinct strings, averaging 7.34 whitespace-delimited words. Each scope contributes 16 PCUs and 96 rollouts per system. Deadlines use simulation time: A, 35--120\,s; B, 80--100\,s; C, 55--200\,s; D, 100--260\,s.
\begin{table}[!htbp]
\caption{CARLA base task inventory, with three paraphrases per arm and seed 1.}\label{tab:categories}
\centering\small
\setlength{\tabcolsep}{3.5pt}
\begin{tabular}{@{}llrr@{}}
\toprule
Stratum & Category & PCUs & Rollouts\\
\midrule
A & junction & 7 & 42\\
A & lane & 4 & 24\\
A & speed\_start & 2 & 12\\
A & waiting\_parking & 3 & 18\\
B & lane\_constraint & 4 & 24\\
B & object\_relation & 4 & 24\\
B & passing\_order & 4 & 24\\
B & speed\_constraint & 4 & 24\\
C & exact\_parking & 4 & 24\\
C & parking\_region & 4 & 24\\
C & port\_target & 4 & 24\\
C & road\_name & 4 & 24\\
D & conditional & 5 & 30\\
D & memory\_reference & 4 & 24\\
D & strict\_order & 5 & 30\\
D & until\_invariant & 2 & 12\\
\bottomrule
\end{tabular}
\end{table}

\subsection{From a physical choice to a paired specification}
The construction workflow links five stages: environment curation, geometric binding, paired instruction authoring, compilation, and validation. Directed lane graphs use 1\,m samples, refined to 0.5\,m at junctions and throughout the parking map. Road names bind to exit gates; regions and parking spaces bind to world-coordinate polygons. Topology endpoints preserve short lane sections, and successor/adjacency references must resolve. Structural labels such as walls and pillars are excluded from parking targets.

Each arm specifies language expressions, target geometry, the requested and alternative contracts, a private reference path, a deadline, and actor setup. Public records contain language, sensors, ego state, and shared neutral context. The private sidecar retains targets, acceptance predicates, future reference paths, and actor-program specifications. Pair checks enforce matching non-language initialization. Static qualification checks entity bindings, directed reachability, footprint geometry, and opposition of the requested outcomes.

The source contains 38 distinct named physical parking spaces across exact-parking branches; 33 are selected by the instantiated scene conditions. They occur in 22 A/C/D PCUs and 44 selected arms. The default exact-parking check uses a 0.2\,m boundary tolerance, $15^\circ$ heading tolerance, speed below 0.2\,m/s, and 3\,s dwell; task sidecars supply the actual parameters. Physical actors instantiate occupied bays. The maneuver checks retain actor-origin/rear-axle frames and a 2.86047\,m wheelbase model. Sustained-speed tasks use signal-free corridors and separated feasible bands; common bands are [1.9,3.9] and [5.9,7.9]\,m/s.

CARLA's live oracle validation covers 47 targeted cases. Named-road qualification checks spatial and directional agreement for 16 arms within 15\,m and $20^\circ$. These tests complement static geometric validation; they do not constitute exhaustive dynamic qualification of all 768 source instances. Task-specific deadlines are fixed in the private sidecars.

\subsection{Representative instruction pairs}
The alternatives in Table~\ref{tab:examples} share a scene entry and change the required behavior. The task inventories in Appendix~\ref{app:inventory} preserve all evaluated identities.
\begin{table}[!htbp]
\caption{Examples of the evaluated CARLA instruction pairs.}
\label{tab:examples}
\centering\small
\begin{tabularx}{\linewidth}{@{}lYY@{}}
\toprule
PCU & Arm 0 & Arm 1\\
\midrule
A-001 & Turn left at this intersection. & Go straight through this intersection.\\
B-025 & Safely pass the slow truck. & Do not overtake; remain behind the slow truck.\\
C-001 & Take Middleton Street. & Take E Gage Ave.\\
C-025 & Park in space B204. & Park in space B205.\\
D-001 & First visit B208, then go to F206. & First visit F206, then go to B208.\\
D-009 & If E101 is occupied, park in E102; otherwise use E101. & If E101 is occupied, park in C104; otherwise use E102.\\
D-026 & Remember A804; visit E101, then return to that remembered place. & Remember E101; visit A804, then return to that remembered place.\\
\bottomrule
\end{tabularx}
\end{table}
\subsection{Finite-trajectory criteria and clocks}
\label{app:task-metrics}
The verifier checks normalized, time-ordered actual ego states. Geometric and behavioral predicates represent gate entry, lane relation, speed, headway, actor order, region occupancy, and parking state. \texttt{eventually} requires a target event; \texttt{always} requires a condition throughout a reached, nonempty interval; \texttt{within} imposes a deadline; \texttt{until} maintains an invariant until a terminating event. \texttt{sequence} checks ordered stages and can reject premature, simultaneous, repeated, or oscillatory stage transitions. \texttt{if\_else} selects a branch from the instantiated scene condition, and \texttt{bind\_reference} preserves an entity for later reuse.

For example, A-001 binds its left and straight requests to distinct exit gates. D-001 reverses the ordered parking stages B208 and F206. D-026 stores a named place and requires returning to the same entity after an intermediate visit. The full strict verdict additionally applies alternative rejection and the safety, progress, timing, and stability terms.

The CARLA monitor evaluates growing prefixes of the executed trajectory. A strict-success candidate normally persists for 2\,s, and persistent tasks must reach their required horizon. Relaxed semantic success is an auxiliary scoring verdict: it does not terminate a rollout or shorten the opportunity to meet the full strict criterion. Execution ends upon strict completion, the task deadline, or an invalid out-of-world ego state. The deadline component uses satisfied-subcondition timestamps, or the trace endpoint when unavailable; it is not the negation of a \texttt{deadline} termination label. 

RealSim distinguishes recording time, rollout-relative time, and instruction time. Language is issued at 1.7\,s. Deadlines are 7.0\,s for 23 PCUs and up to 13.8\,s for the remainder, leaving 5.3--12.1\,s after instruction. Privileged reference traces extend 6\,s beyond the deadline at 0.1\,s sampling; that tail is not extra model execution time. The component validator uses a default maximum sample gap of 0.25\,s and numerical speed tolerance of $10^{-6}$\,m/s.

\subsection{Relaxed proxies, graded completion, and task progress}
CARLA SemISR accepts the full requested contract or its predefined relaxed proxy. Depending on the task, the proxy may use heading, displacement, expanded gate proximity, a lane relation with progress, an observation in an expanded speed/headway band, or relaxed region geometry. Parking proxies may accept an expanded footprint-center region and $45^\circ$ heading tolerance without stopped dwell. Sequence and reference proxies retain ordered event checks. These are deterministic partial-completion rules, not a free-form language-model judgment.

RealSim branch tasks use the first correct irreversible commitment, with timeliness relaxed in the semantic component. Lane tasks can relax dwell. Scoped-lane and speed tasks retain full entry-to-exit traversal, relevant exclusivity and trigger conditions, and, for speed, all active samples in the requested band. Target-lane dwell is 1\,s; scoped-lane tasks require a 2\,s minimum dwell. Referent triggers must occur after instruction and by scope entry, then remain latched until exit.

The task-specific maps $C_{i,j}$ define graded completion in $[0,1]$ and are aggregated by Equation~\ref{eq:graded}. The binary strict criterion, relaxed proxy, and continuous map are distinct evaluator outputs. The open-loop rules appear in Appendix~\ref{app:offline}. Within each protocol, all language metrics use the same scheduled pair set, and unsuccessful executions remain in the denominator.

Task progress is a gate in the strict-success criterion, distinct from continuous reference advancement. CARLA's default progress term requires evaluator progress of at least 0.05. RealSim checks a 10\,m threshold over the available rollout, including its pre-instruction segment. CARLA safety excludes normalized collision, off-road, and red-light events. RealSim scoring v0.2 checks collision and off-road events; requested speed remains a language criterion. These components are distinct from continuous reference advancement.

\subsection{RealSim scene mining and execution qualification}
\label{app:realsim}
Candidate records bind scene identity, recording window, initial ego pose, and four source cameras at $576\times320$. The construction records contain 120 scene-acquisition entries, 308 route opportunities, and 32 task-source lineage entries. Temporal screening retains 185 route opportunities: 94 lack sufficient preceding context and 29 lack execution time. Task authoring and review freeze 30 PCUs and retire two lineage entries. These inventories have different units and do not form one nested count funnel.

Branch topology, adjacent lanes, and scoped speed requests provide the alternatives. Authors review both-arm rendered views, scope/trigger cues, and rendering quality, then freeze public inputs and private criteria. Neutral navigation uses \texttt{UNKNOWN}, with waypoint-command routing disabled. The final 30 PCUs span 23 scenes, six in A and 24 in B. Every frozen PCU passes both-arm qualification under three seeds: 180 privileged-driver traces satisfy their requested contract and reject the alternative. 

Scope support is explicit. P01 requests 19.25--20.75 versus 21.25--22.75\,km/h. P30 requests 20.25--21.75 versus 22.25--23.75\,km/h in a 1.526\,m interval with two active samples in its qualification example. Across 144 witnesses carrying scoped-sample diagnostics, support ranges from 2 to 26 samples. The verifier evaluates these active samples over complete traversal, not a whole-rollout average speed.
\FloatBarrier

\FloatBarrier
\section{Model interfaces and execution settings}
\label{app:adapters}
\subsection{CARLA adapters}
LMDrive, SimLingo, Qwen-Drive, OneVL, Impromptu-VLA, and Senna use the model families described by \citet{shao2024lmdrive,renz2025simlingo,zhou2026qwen,lu2026xiaomi,chi2026impromptu,jiang2026senna}. Table~\ref{tab:adapters} identifies the input and output paths used in the complete CARLA comparison. Native and generic trajectory adapters differ in sensor coverage, training domain, and controller. The public/private input split does not impose identical camera rigs. Public-language conversion into a model's required command or local goal is included in the tested system identity.

\begin{table}[!htbp]
\caption{Complete CARLA system paths. A bridge denotes an interface that differs from the original end-to-end architecture.}
\label{tab:adapters}
\centering\small
\begin{tabularx}{\linewidth}{@{}p{30mm}YY@{}}
\toprule
System & Effective input & Action path / important configuration\\
\midrule
LMDrive & Four RGB views, LiDAR, ego state & Native visual encoder/LLM, waypoint head, PID; public-language local-goal adaptation\\
SimLingo & Front RGB, motion state & InternVL2-1B, epoch013; native waypoints/PID\\
Alpamayo 1.5 & Custom multi-camera temporal history & Trajectory generation and tracking; multi-camera temporal input\\
Qwen-Drive-1.0-4B & Front/left/right RGB, ego & Backbone with planner-rl; sampled plans and trajectory tracking\\
OneVL & Single front view, ego/history & NAVSIM-format latent-prefix trajectory generation\\
Impromptu-VLA & Front/right/left RGB, ego history & Released 3B AD-finetuned QA-format waypoint generation\\
Senna VLM bridge & Six RGB views & SPEED/PATH output expanded deterministically into a trajectory; no Senna-E2E planner\\
\bottomrule
\end{tabularx}
\end{table}

The generic camera rig uses yaw angles $0/-60/+60/180/-120/+120$ degrees for front/left/right/rear/back-left/back-right, typically at $896\times512$ resolution and $100^\circ$ field of view. These settings do not replace the native rigs: SimLingo uses a $1024\times512$ front image with $110^\circ$ field of view, and LMDrive and Alpamayo have their own configurations. Generic control uses meter-valued ego-frame waypoints with declared timestamps, steering based on a pure-pursuit-type tracker, and speed control. Invalid, nonfinite, or out-of-range plans can trigger braking or an error; such outcomes remain part of the system behavior.

The generic Qwen-Drive, OneVL, Impromptu, Senna, and ReCogDrive bridge configurations record model calls every 20 CARLA frames. At 20\,Hz this corresponds to a configured one-second call interval; it is not a measured wall-clock latency. Waypoint spacing is declared separately, commonly 0.5\,s for these adapters. Alpamayo records a 1\,s replanning interval, 16 ego-history steps at 0.1\,s, and four camera frames. Controller launch or fallback behavior is part of an adapter's policy and should remain attached to its exact configuration.

\subsection{RealSim model identities}
The RealSim main comparison includes SimLingo, Impromptu-VLA, ReCogDrive, OneVL, AutoVLA, Qwen-Drive, Alpamayo~1.5, and Alpamayo2-Super \citep{renz2025simlingo,chi2026impromptu,xiong2026recogdrive,lu2026xiaomi,zhou2026autovla,zhou2026qwen,wang2025alpamayo}. The main roster has separate rows for Alpamayo~1.5 and Alpamayo2-Super.

RealSim's workload/qualification baseline has four logical views: \texttt{cross\_left\_120fov}, \texttt{cross\_right\_120fov}, \texttt{front\_tele\_30fov}, and \texttt{front\_wide\_120fov}, at $576\times320$. Model-specific view selection, resizing, or history construction is recorded separately. A reconstructed multi-view input is not automatically equivalent to a model's original training rig. Repeating the earliest observed frame to fill a short history uses no future ground truth, but remains a disclosed input variation.

ReCogDrive uses a NAVSIM-transfer path in RealSim, distinct from the CARLA VLM bridge. OpenDriveVLA's RealSim path loads the Qwen+UniAD visual artifacts and performs real multi-image forward passes, distinct from the text-only CARLA head. Senna's official row remains blocked because the reviewed complete low-level planner/checkpoint path is absent. Its independently executed meta-action-to-trajectory variant is not substituted into that official row. LMDrive's official row is blocked by input-equivalence requirements, particularly LiDAR; the synthetic-null-LiDAR smoke is a separate non-equivalent variant.

\subsection{Environment and process topology}
The recorded CARLA Python environment uses Python 3.11.15, PyTorch 2.2.0+cu121, CUDA 12.1, Transformers 4.49.0, NumPy 1.26.4, and CARLA Python package 0.9.16. Model-local vendor directories can override these dependencies. The Python package version alone does not identify the custom server binary or map build. The usual execution topology uses seven model/ simulator workers on seven GPUs; SimLingo uses up to 14 workers on seven GPUs. This is route-level parallelism, not tensor-parallel inference.

RealSim fixes AlpaSim commit \texttt{affc2eab209fa43bdfa2f26c0f8d437922d78a68}, PhysicalAI-NuRec version 26.01, and NRE image version 26.04 for evaluation. The main runtime uses Python 3.12 with isolated model-worker environments. Simulation services and model workers can reside on different hosts; runtime logs and attempt identities bind their outputs. Latency, GPU-hour, and energy comparisons are outside this evaluation.
\FloatBarrier

\clearpage
% \input{source/appendices/reproducibility}
% Expanded from the supplied protocol reports on 2026-09-26.
% Drop-in appendix fragment: keep the main paper's bibliography and macros.
% Required packages: amsmath, amssymb, booktabs, tabularx, array, placeins.
% Existing labels app:provenance, app:domain, app:offline,
% tab:offline-coverage, and tab:wod-geometry are preserved.
% Source notes are comments, not new bibliography entries.

\section{Statistical interpretation and reproducibility}
\begin{table*}[t]
\centering\footnotesize
\caption{\textbf{Open-loop instruction following on three datasets.} Predicted trajectories are scored for exclusive binary semantic success (SemSISR/SemCFPS) and graded completion (SCS/PCS). All values use a 0--100 scale.}
\label{tab:offline-main-1}
\setlength{\tabcolsep}{4pt}
\renewcommand{\arraystretch}{0.95}
\begin{tabularx}{\linewidth}{@{}Xrrrr@{}}
\toprule
System & SemSISR & SemCFPS & SCS & PCS \\
\midrule
\multicolumn{5}{c}{\textit{WOD-E2E}} \\
\midrule
DriveMA-2B & \textbf{38.38} & \textbf{1.24} & 49.33 & 7.88 \\
Fast-dDrive & 34.85 & 0.00 & 46.42 & 5.27 \\
Qwen-Drive-SFT & 34.44 & 0.00 & \textbf{50.58} & \textbf{9.37} \\
Qwen-Drive-RL & 35.06 & 0.00 & 49.58 & 7.58 \\
\midrule
\multicolumn{5}{c}{\textit{NAVSIM}} \\
\midrule
Qwen-Drive-1.0-4B & 0.000 & 0.000 & 5.472 & 0.000 \\
AutoVLA & 3.947 & 0.000 & 8.172 & 0.000 \\
DriveMA-2B & \textbf{8.553} & 0.000 & \textbf{11.244} & \textbf{2.242} \\
ReCogDrive  & 0.000 & 0.000 & 4.400 & 0.000 \\
OneVL & 3.618 & 0.000 & 8.428 & 0.000 \\
\midrule
\multicolumn{5}{c}{\textit{nuScenes}} \\
\midrule
Impromptu-VLA & 15.83 & 0.00 & 20.21 & 0.00 \\
UniDriveVLA & \textbf{16.67} & 0.00 & 25.00 & 0.00 \\
OpenDriveVLA-0.5B & 13.33 & \textbf{3.33} & 19.06 & 4.30 \\
OmniDrive & 7.50 & 0.00 & \textbf{29.24} & \textbf{8.61} \\
\bottomrule
\end{tabularx}
\end{table*}
% \begin{minipage}{\linewidth}\scriptsize
% Bold compares models within each dataset; the constant-speed baseline is excluded. These binary semantic metrics enforce alternative rejection, unlike the closed-loop relaxed proxies. Fast-dDrive is the evaluated WOD-E2E system, not an AutoVLA result. nuScenes AutoVLA and OpenEMMA are unavailable and are listed as such in Appendix~\ref{app:offline}. WOD-E2E semantic and geometric diagnostics use separate evaluation sets.
% \end{minipage}
\label{app:provenance}\label{app:domain}

\subsection{Evaluation units, provenance, and interpretation}
\label{app:repro-units}
A reproducible run binds the dataset and task-source versions, PCU, instruction arm,
expression, random seed, checkpoint revision, input adapter, decoding rule, trajectory
frame and timestamps, controller when applicable, private scorer, and selected
attempt. The PCU, rather than an instruction string or a retry, is the paired
experimental unit. Repeated expressions and both arms remain grouped when computing
paired scores. Duplicate requests, retries, and integration smoke tests do not create
additional scenes. We distinguish source frames, candidate pairs, qualified PCUs,
model requests, and valid predictions throughout this appendix.

The main closed-loop comparisons use a fixed evaluation seed per track. The L2D values
in Tables~\ref{tab:carla-main} and~\ref{tab:realsim-main} are computed from the
reported component metrics with high-precision intermediate arithmetic and rounding
only for presentation; this aggregation does not alter any rollout verdict. RealSim
trajectory-response diagnostics use the complete valid pairs reported in
Table~\ref{tab:real-response}, while the main success metrics retain their scheduled
denominators. CARLA, RealSim, and the open-loop tracks are reported separately because
they instantiate different horizons, task mixtures, and scoring functions.

\subsection{Open-loop versions, coverage, and common notation}
\label{app:offline}
% SOURCES: DATASET_AND_EVALUATION_PROTOCOL.md Sections 1--3;
% BENCHMARK_SPEC_V0.1.0.md Sections 1--3;
% NAVSIM report bundle: release_manifest.json and EXPERIMENT_PROTOCOL.md.

\paragraph{Version resolution.}
The nuScenes evaluation uses \texttt{nuscenes-test-v1.0.0}. NAVSIM uses the
38-PCU \texttt{l2d-offline-navsim} 1.0.0 release identified by the September 24
experiment bundle, with 19 recording clusters and seed 42. WOD-E2E uses the frozen P0 kinematic-proxy protocol (v0.1.0); all WOD results
in this appendix are reported with the corresponding proxy metric names.

The source-specific roster includes DriveMA-2B, Fast-dDrive, and
UniDriveVLA~\citep{zheng2026drivema,zhang2026fast,li2026unidrivevla}, the model
families in Appendix~\ref{app:adapters}, and OmniDrive~\citep{wang2025omnidrive}.
Fast-dDrive is evaluated on WOD-E2E; AutoVLA~\citep{zhou2026autovla} is evaluated
in the separate NAVSIM experiment. The nuScenes comparison reports the models for
which the corresponding frozen adapter outputs are available.

\begin{table}[!htbp]
\centering\small
\caption{Open-loop coverage of the constructed Language2Drive subsets. Candidate,
qualified-PCU, and request counts are reported separately.}
\label{tab:offline-coverage}
\begin{tabularx}{\linewidth}{@{}p{0.14\linewidth}XXX@{}}
\toprule
Source & Available source / construction & Evaluation subset & Queries per model\\
\midrule
nuScenes & 150 test scenes; 6,008 anchors; 83 candidates & 75 approved PCUs in 75 scenes / 15 logs; dev 15, heldout 60 & 900 across 75 PCUs; 720 heldout, including controls\\
NAVSIM & OpenScene test; 40 reviewed candidates & 38 released PCUs / 19 original recordings & 304: two arms $\times$ four wordings $\times$ 38 PCUs\\
WOD-E2E & 38/93 validation shards; 43,394 frames; 203 rated anchors & 563 complete-window candidates; 441 with P0 drafts; 241 motion-proxy pairs & 882 draft requests; 482 belong to the proxy subset\\
\bottomrule
\end{tabularx}
\end{table}

\begin{table}[!htbp]
\centering\small
\caption{Prediction grids and dataset-specific anchor reference points. Future points
exclude $t=0$.}
\label{tab:repro-time-frame}
\begin{tabularx}{\linewidth}{@{}lcccX@{}}
\toprule
Source & Horizon & Interval & Points & Output and fixed anchor reference\\
\midrule
nuScenes & 3.0 s & 0.25 s & 12 & $(x,y)$; \path{ego_pose_origin}\\
NAVSIM & 4.0 s & 0.50 s & 8 & $(x,y,\psi)$; rear-axle origin\\
WOD-E2E & 5.0 s & 0.25 s & 20 & $(x,y)$; rear-axle midpoint\\
\bottomrule
\end{tabularx}
\end{table}

\paragraph{Coordinates and resampling.}
Positions use meters and an anchor-fixed local plane with $+x$ forward and $+y$
left. Velocities and accelerations use m/s and m/s$^2$. For an anchor with global
planar position $(X_0,Y_0)$ and yaw $\psi_0$, the planar conversion is
\begin{equation}
\begin{bmatrix}X_k\\Y_k\end{bmatrix}
=
\begin{bmatrix}X_0\\Y_0\end{bmatrix}
+
\begin{bmatrix}\cos\psi_0&-\sin\psi_0\\\sin\psi_0&\cos\psi_0\end{bmatrix}
\begin{bmatrix}x_k\\y_k\end{bmatrix}.
\label{eq:repro-se2}
\end{equation}
The reference-point definition is part of the protocol and is matched by every
adapter. NAVSIM uses an SE(2) transform based on $(X,Y,\psi)$ and verifies that the
current lane covers the local origin. Native trajectories are converted to the
source-specific evaluation grids before scoring; private targets are never used to
select among multiple model outputs.

\paragraph{Binary and graded paired scores.}
For PCU $i$, wording $r$, and arm $a\in\{0,1\}$, let $\tau_{ir,a}$ be the submitted
trajectory, $V(\tau)$ its format-validity indicator, and $E_{i,a}(\tau)$ its native
binary task predicate. Define
\begin{equation}
 B_{ir,a}=V(\tau_{ir,a})E_{i,a}(\tau_{ir,a})
 \bigl[1-E_{i,1-a}(\tau_{ir,a})\bigr],\qquad
 c_{ir,a}=V(\tau_{ir,a})C_{i,a}(\tau_{ir,a}).
\label{eq:repro-exclusive}
\end{equation}
Here $C_{i,a}\in[0,1]$ is the \emph{source-specific} graded completion function.
For a declared wording set $\mathcal R_i$, write
$\langle z\rangle=N^{-1}\sum_i |\mathcal R_i|^{-1}\sum_{r\in\mathcal R_i}z_{ir}$.
The four reported aggregates are
\begin{align}
\mathrm{SemSISR}&=100\left\langle\frac{B_{ir,0}+B_{ir,1}}2\right\rangle,
&\mathrm{SemCFPS}&=100\langle B_{ir,0}B_{ir,1}\rangle,\nonumber\\
\mathrm{SCS}&=100\left\langle\frac{c_{ir,0}+c_{ir,1}}2\right\rangle,
&\mathrm{PCS}&=100\langle\min(c_{ir,0},c_{ir,1})\rangle.
\label{eq:repro-four-metrics}
\end{align}
Missing or invalid predictions receive zero binary and completion credit without
shrinking the fixed eligible denominator. A source may reject an incomplete result
file before scoring, as specified below. Candidate-qualification failures are not
model failures and are excluded before fixing that denominator.

The open-loop reports use the same paired aggregation while preserving each
dataset's native primitive scorer: nuScenes uses temporal coverage with a displacement
gate, NAVSIM uses tolerance-based task completion with minimum conjunction, and
WOD-E2E uses the kinematic proxy functions defined below. WOD values are denoted
\emph{ProxySemSISR}, \emph{ProxySemCFPS}, \emph{ProxySCS}, and \emph{ProxyPCS}.

\subsection{nuScenes: construction, qualification, and exact scoring}
\label{app:repro-nuscenes}
% SOURCE: DATASET_AND_EVALUATION_PROTOCOL.md, Sections 1--12.

\paragraph{Source inventory and split.}
Construction uses local nuScenes \texttt{v1.0-test}, not \texttt{v1.0-trainval}.
The catalog contains 150 scenes and 6,008 anchors: 3,318 in Boston, 796 in
Singapore-onenorth, and 1,894 in Singapore-queenstown. At most one PCU is retained per scene.
Logs are grouped by location, sorted by the SHA-256 of their log token, and the
first $\max(1,\operatorname{round}(0.25n_{\mathrm{log}}))$ logs per location are
assigned to development, using Python's \texttt{round}. The split is then intersected
with the approved set. This yields 15 dev PCUs in four logs and 60 heldout PCUs in
11 logs, with neither scene nor log overlap. The resulting split contains 15 dev and 60 heldout PCUs.

\begin{table}[!htbp]
\centering\small
\caption{Released nuScenes task and geographic distributions. The paired lane
alternatives are change/keep, not left/right within the same PCU.}
\label{tab:repro-nusc-distribution}
\begin{tabularx}{\linewidth}{@{}Xrrr@{}}
\toprule
Category & Released & Dev & Heldout\\
\midrule
Lane change / keep & 22 & 5 & 17\\
\quad Left-change target & 12 & 3 & 9\\
\quad Right-change target & 10 & 2 & 8\\
Scoped speed: approximately 15 / 30 km/h & 53 & 10 & 43\\
\midrule
Boston-seaport & 41 & 7 & 34\\
Singapore-onenorth & 12 & 4 & 8\\
Singapore-queenstown & 22 & 4 & 18\\
\midrule
Total PCUs & 75 & 15 & 60\\
\bottomrule
\end{tabularx}
\end{table}
The lane/speed proportions are 29.33/70.67\% in the release and 28.33/71.67\% in
heldout. The released nuScenes track focuses on lane-choice and scoped-speed PCUs.

\paragraph{Public observations and coordinate metadata.}
Public inputs reference six cameras and the \path{LIDAR_TOP} keyframe, retaining
sensor timestamps and calibration tokens, plus up to four historical keyframes and
ego poses. Historical-frame counts of 4/3/2/1 occur in 67/3/1/4 PCUs, with observation offsets
spanning 0.400365--2.100166 s before the anchor. The output declares
\path{coordinate_frame=ego_at_anchor} and
\path{reference_point=ego_pose_origin}. The anchor transform uses the translation and yaw of the ego pose attached to
\path{LIDAR_TOP}; quaternion storage is $(w,x,y,z)$.

The four evaluated adapters first produce six future positions spaced by 0.5 s.
The known anchor $(0,0)$ is prepended and piecewise-linear interpolation yields the
12-point grid in Table~\ref{tab:repro-time-frame}. The adapters resample the submitted future trajectory to the fixed 3 s grid; the
Impromptu adapter additionally constructs its required observed-history input to $-3$ s.

\paragraph{Candidate mining and map binding.}
Initial speed is estimated from the current and preceding ego poses using their
actual time separation, which must lie in $[0.4,0.6]$ s. Candidates bind to a
\texttt{CAR} lane queried within 2 m, with lane-direction error at most $30^\circ$
and polygon distance strictly below 1 m. The local tangent uses points 0.5 m before
and after the projected arc length. Lane paths are extended toward 80 m through
connections with endpoint gap at most 3 m and turn angle strictly below $35^\circ$;
the next connection minimizes $\mathrm{turn}_{\mathrm{rad}}+0.02\,
\mathrm{gap}_{\mathrm m}$. A disconnected path can end earlier, but each witness
must remain at least 0.5 m before its endpoint. The path representation is the
polyline through arcline-segment endpoints, not a densely sampled exact curve.
Lane candidates are built first; speed candidates skip occupied scenes. The first
qualified anchor in catalog order is retained per scene. The lane-builder limit of
80 was not reached: 22 lane candidates were found.

\begin{table}[!htbp]
\centering\small
\caption{nuScenes target-construction parameters. These mining and witness
parameters are distinct from the final scoring tolerances.}
\label{tab:repro-nusc-construction}
\begin{tabularx}{\linewidth}{@{}p{0.30\linewidth}X@{}}
\toprule
Parameter & Value or exact rule\\
\midrule
Lane-task initial speed & $4.5\le v_0\le9.0$ m/s; released range 4.745680--8.878580\\
Adjacent-lane query & 7 m radius; polygon gap $\le0.8$ m\\
Divider / direction & Divider exists and every segment type contains \texttt{DASHED}; target yaw error $\le20^\circ$\\
Relative target position & Lateral distance 2.4--5.0 m on the requested side; absolute forward offset $\le5.0$ m\\
Candidate ranking & $|d_{\mathrm{lat}}-3.5|+0.1|d_{\mathrm{fwd}}|$; search left before right\\
Lane witness blend & Start 0.4 s; duration 2.0 s; initial centerline offset decays over 1.0 s\\
Speed-task initial speed & $5.5\le v_0\le6.5$ m/s; released range 5.521365--6.491924\\
Speed witness targets & 4.0 and 8.0 m/s, reached after a 1.0 s transition\\
\bottomrule
\end{tabularx}
\end{table}

\paragraph{Counterfactual witnesses.}
Both arms share the observation, ego history, and public-input hash; the model
receives \path{predict(public_input, instruction, seed)}, without arm labels,
target lane tokens, contracts, witnesses, or review decisions. Targets are map-bound
sets of acceptable behavior, not two fitted copies of the single recorded future.
For lane change, the current-lane and target-lane paths are mixed with
$u=\operatorname{clip}((t-0.4)/2,0,1)$ and $w=u^2(3-2u)$. Full mixing occurs at
2.4 s, while the contract requires entry into the target corridor from 2.25 s;
entry does not require reaching its centerline. The keep witness follows the
current-lane path. For speed tasks, both witnesses follow the same extended path,
with arc-length schedule
\begin{equation}
 s(t)=\begin{cases}
 v_0t+\tfrac12(v_*-v_0)t^2,&0\le t\le1,\\
 \tfrac12(v_0+v_*)+v_*(t-1),&t>1,
 \end{cases}\qquad v_*\in\{4,8\}\;\mathrm{m/s}.
\label{eq:repro-nusc-speed-witness}
\end{equation}
These witnesses support local feasibility during construction; they are not unique
reference trajectories that a prediction must imitate.

\paragraph{Qualification and human review.}
Construction produces 83 candidates (22 lane and 61 speed). Automated witness,
map, motion, and swept-footprint checks retain 82 candidates, after which manual
review freezes 75 PCUs (22 lane and 53 speed). Both arms must have feasible witnesses,
satisfy their requested predicate, reject the paired alternative, and pass the
shared geometric and motion checks.

The swept-footprint screen uses a $5.0\times2.0$ m rectangular vehicle approximation
under two reference-point hypotheses and checks coverage against buffered drivable
polygons. LiDAR-assisted review additionally summarizes points near each witness
path over the 1--3 s interval. The review interface presents six cameras, the local
map, both instructions and witnesses, and the map/LiDAR diagnostics. Final approval
requires \texttt{observable}, \path{arm_0_feasible},
\path{arm_1_feasible}, \path{mutually_exclusive},
\path{traffic_clearance}, and \path{map_binding} to be marked
\texttt{yes}; the decision, reviewer identifier, rationale, timestamp, and source hash
are recorded in the frozen review table.

\paragraph{Output validity and binary contracts.}
Both \path{timestamps_s} and \path{xy_m} must contain 12 entries, with two
finite coordinates per point and the exact frame/reference strings above.
Timestamp error must be at most $10^{-5}$ s. Missing points, NaN, or infinity are invalid, and the scorer does not fill missing
future points.
For $p_0=(0,0)$ and $\Delta t=0.25$ s, define
\begin{equation}
 v_k=\|p_k-p_{k-1}\|/\Delta t,\qquad D=\|p_{12}\|.
\label{eq:repro-nusc-features}
\end{equation}
Every arm predicate in Table~\ref{tab:repro-nusc-binary} additionally requires
$D\ge8$ m.

\begin{table}[!htbp]
\centering\small
\caption{Complete nuScenes binary semantic conditions, in addition to $D\ge8$ m.
Every required sample must satisfy its condition.}
\label{tab:repro-nusc-binary}
\begin{tabularx}{\linewidth}{@{}p{0.20\linewidth}Xp{0.32\linewidth}@{}}
\toprule
Arm & Predicate & Required samples\\
\midrule
Lane-change arm 0 & Reference point lies in the target lane corridor & $t=2.25,2.50,2.75,3.00$ s; 4 points\\
Lane-keep arm 1 & Reference point lies in the current lane corridor & $t=0.25,0.50,\ldots,3.00$ s; 12 points\\
Low-speed arm 0 & $3.2\le v_k\le4.8$ m/s & Interval ends $t_k=1.25,\ldots,3.00$ s; 8 intervals\\
High-speed arm 1 & $7.2\le v_k\le8.8$ m/s & Same 8 intervals\\
\bottomrule
\end{tabularx}
\end{table}

Each corridor comprises the available \texttt{CAR} lane polygons in its private
token set, individually buffered by 0.15 m. Boundary points count as inside through
\texttt{covers}. The speed window covers 1.00--3.00 s, with the first interval spanning
1.00--1.25 s. Binary scoring tests the requested predicate and the paired-alternative
predicate as in Eq.~\eqref{eq:repro-exclusive}.

\paragraph{Separate motion and map-quality checks.}
Motion and map-quality diagnostics are reported alongside semantic success. The
thresholds are maximum speed $\le17$ m/s, $|v_1-v_{\mathrm{initial}}|\le3$ m/s,
maximum absolute speed-difference acceleration $\le5$ m/s$^2$, and maximum absolute
wrapped heading-change rate $\le1.0$ rad/s. Heading differences use
$\operatorname{atan2}(\sin\Delta\psi,\cos\Delta\psi)/0.25$. Drivable-center requires
all 12 reference points to lie in, or within 0.2 m of, the drivable region. Both
rates use all scheduled arms, with invalid outputs receiving zero credit.

\paragraph{Partial completion and response metrics.}
The supplementary unified semantic metrics use the arm's required
index set $A_a$ in Table~\ref{tab:repro-nusc-binary}, let $I_a(k,\tau)$ denote its
lane-membership or speed-band test. Then
\begin{equation}
 C_a(\tau)=\min(1,D/8)\frac{1}{|A_a|}\sum_{k\in A_a}I_a(k,\tau).
\label{eq:repro-nusc-continuous}
\end{equation}
Invalid outputs have $C_a=0$. There is no opposite-target penalty or motion/map
multiplier. Thus SCS/PCS award partial temporal coverage; spatial membership itself
remains binary.

SRG uses the common window: the four terminal lane points or the eight speed
intervals. At each point, $d_a$ is distance to the corresponding corridor in meters,
or $\max(\ell_a-v,0,v-h_a)$ for a speed band $[\ell_a,h_a]$ in m/s. Define
\begin{equation}
 p_0(k,\tau)=
 \begin{cases}(d_1-d_0)/(d_0+d_1),&d_0+d_1>0,\\0,&d_0+d_1=0,\end{cases}
 \quad
 \mathrm{SRG}_i=\frac{\overline p_0(\tau_{i,0})-
 \overline p_0(\tau_{i,1})}{2}.
\label{eq:repro-nusc-srg}
\end{equation}
The ratio is computed before time averaging; no epsilon, temperature, clipping, or
progress gate is added. SRG lies in $[-1,1]$ and measures directional response, not
absolute completion. NBD instead uses all 12 points:
\begin{equation}
 \mathrm{NBD}_i=
 \frac{\sum_{k=1}^{12}\|p_{i,0,k}-p_{i,1,k}\|}
 {\sum_{k=1}^{12}(\|p_{i,0,k}\|+\|p_{i,1,k}\|)}.
\label{eq:repro-nusc-nbd}
\end{equation}
A zero denominator gives zero without an epsilon. Each PCU's ratio is computed
before averaging. Displayed SRG and NBD are multiplied by 100, and paired ADE in meters is retained
as an absolute-scale companion to NBD.

SCS/PCS retain every scheduled PCU with invalid completion zero. SRG, NBD, and paired
ADE average only valid pairs, record invalid values as null, and report coverage;
an empty valid set is undefined, not zero. On canonical heldout predictions,
OmniDrive has 116/120 valid arms and 56/60 valid pairs, while the other three
systems have 120/120 and 60/60. PCU-weighted averages use the 17/60 and 43/60 lane/speed
weights. A separately labeled family-macro average assigns each family weight
$1/2$; valid-pair filtering is applied within each family for response metrics.
The stored \path{SRG_coverage} field is a separate coverage-based diagnostic;
Eq.~\eqref{eq:repro-nusc-srg} defines the distance-based SRG used here.

\begin{table}[!htbp]
\centering\footnotesize
\caption{Additional nuScenes diagnostics and their full numeric conventions.
Conditional means always accompany validity coverage.}
\label{tab:repro-nusc-diagnostics}
\begin{tabularx}{\linewidth}{@{}p{0.24\linewidth}Xp{0.20\linewidth}@{}}
\toprule
Diagnostic & Exact definition / threshold & Invalid handling\\
\midrule
Speed-target MAE & Mean $|3.6v-v_*|$ on 8 intervals, $v_*=15$ or 30 km/h & Valid-arm mean\\
Speed-band violation & $3.6\,\mathrm{mean}\max(\ell-v,0,v-h)$, km/h & Valid-arm mean\\
Speed coverage & In-band fraction over 8 intervals, without displacement gate & Zero for invalid arms\\
Paired speed change & $\Delta v=3.6[\mathrm{mean}(v_{\mathrm{high}})-\mathrm{mean}(v_{\mathrm{low}})]$ & Valid-pair mean\\
Lane distance / coverage & Mean distance to own corridor; inside fraction over 4 change or 12 keep points & Distance conditional; coverage zero if invalid\\
Lane preference & $q=\mathrm{mean}(d_{\mathrm{current}}-d_{\mathrm{target}})$ on the common window; $\Delta q=q_{\mathrm{change}}-q_{\mathrm{keep}}$ & Valid arm / pair\\
Directional response & Speed: $\Delta v>1.0$ km/h; lane: $\Delta q>0.2$ m; no-dead-zone version uses $>10^{-6}$ & All family PCUs; invalid fails\\
Semantic margins & Speed: $3.6\min_k\min(v_k-\ell,h-v_k)$; lane: minimum signed boundary distance, inside positive & Valid-arm mean\\
Opposite-arm ADE & Mean pointwise separation over 12 points & Valid pairs\\
Paraphrase ADE & Canonical versus each same-arm rewrite; 360 planned heldout comparisons & Valid comparisons\\
Masked / shuffled ADE & Canonical versus same-arm control; 120 comparisons per heldout condition & Valid comparisons\\
\bottomrule
\end{tabularx}
\end{table}
The language targets of 15/30 km/h differ slightly from the witness speeds of
14.4/28.8 km/h and the scoring-band centers. Target MAE uses the language targets;
binary compliance uses the frozen task-specific bands. Semantic margins summarize
the active speed-band or lane-boundary constraint.

\paragraph{Language conditions and inference seed.}
Each arm has one canonical instruction and three fixed English paraphrases.
Together they form 600 sample-indexed text instances over the 75-PCU release;
templates can repeat across scenes. Masked and shuffled controls yield six
conditions in total and 900 outputs per model (720 heldout), all with seed 0.
The two arms already provide the opposite-instruction comparison; there is no
additional seventh condition. Masked passes an empty string to the runner, which
an adapter may render as a neutral prompt. Shuffled uses the opposite-arm text from
the next PCU in the same family after sorting by PCU ID; unchanged text falls back
to the current PCU's opposite arm. The executed donor pool is all 75 PCUs, followed by heldout filtering.
Masked/shuffled scores are evaluated against the original arm target and serve as
controlled language ablations.

The formal scoring wrapper rejects files missing requested conditions, seeds, or
arms; the underlying scorer assigns zero to a missing arm within an already present
condition. \path{Paraphrase_Worst_PCU} requires both arms to succeed under all
three paraphrases within the same PCU, rather than taking the minimum of three
aggregate CFPS values. Masked and shuffled gaps subtract the corresponding control
CFPS from canonical CFPS.
\FloatBarrier

\subsection{NAVSIM: released portfolio and frozen evaluation}
\label{app:repro-navsim}
% SOURCES: l2d_offline_navsim_report_20260924/{L2D_OFFLINE_NAVSIM_REPORT.md,
% artifacts/PAPER_CODEX_HANDOFF.md, artifacts/EXPERIMENT_PROTOCOL.md,
% artifacts/METRIC_SPEC.md, artifacts/release_manifest.json,
% artifacts/benchmark_release_validation.json, artifacts/scores_*.json}.

\paragraph{Release size and task composition.}
The released NAVSIM/OpenScene test portfolio contains 38 PCUs from 19 original
recordings, after reviewing 40 candidates and rejecting two. Each PCU has two arms
and four approved English wordings per arm: canonical and three paraphrases.
The evaluation therefore contains 304 queries per model and 1,520 queries across
five models. All 1,520 outputs were parsed into the declared trajectory schema.
Recordings define the coarsest reported cluster structure, and multiple PCUs may
share a recording.

\begin{table}[!htbp]
\centering\small
\caption{NAVSIM 1.0.0 task composition and documented counterfactual targets.
Counts and percentages use the 38-PCU release, not the earlier 12-candidate pool.}
\label{tab:repro-navsim-distribution}
\begin{tabularx}{\linewidth}{@{}p{0.23\linewidth}crrX@{}}
\toprule
Family & Scope & PCUs & Share & Counterfactual alternatives\\
\midrule
Lane choice & A & 12 & 31.58\% & Keep versus change one lane; eight left-target and four right-target PCUs\\
Junction branch & A & 1 & 2.63\% & Left turn versus straight at a common junction\\
Pass / follow & A & 1 & 2.63\% & Follow the same actor versus pass using an adjacent lane\\
Speed band & B & 7 & 18.42\% & Two disjoint speed bands over the declared interval\\
Object-relative speed & B & 2 & 5.26\% & Two speed requirements while following the same actor\\
Landmark stop & C & 10 & 26.32\% & Distinct stopping locations grounded in visible landmarks\\
Local goal & C & 1 & 2.63\% & Adjacent left-lane versus right-lane target\\
Event order & D & 4 & 10.53\% & Slow then change lane versus change lane then slow\\
\midrule
Total & A--D & 38 & 100.00\% & Scope counts A/B/C/D: 14/9/11/4\\
\bottomrule
\end{tabularx}
\end{table}

\paragraph{Counterfactual construction and qualification.}
Candidates are mined using observed ego state, camera/LiDAR history, map geometry,
and logged actors. Each arm receives a private executable contract and feasibility
witness. Lane and local-goal tasks bind different lane regions; branch tasks use a
junction goal region; landmark stopping binds a visible named or described landmark
to a stopping region. Actor-conditioned tasks bind the same actor across both arms.
Order tasks retain the two events while exchanging their required order. These
contracts define acceptable behavior sets, rather than forcing the model to match
the witness path point by point.

Release gates check sensor availability, source-split isolation, public/private PCU
ID agreement, private-label leakage, own/opposite contract behavior, witness replay,
map footprint, logged actors, and collision screening on a 0.1 s grid. Human review
records a reviewer, pass/reject decision, and rationale. All 38 released PCUs pass
the frozen release checks, and oracle replay gives OL-SemSISR/OL-SemCFPS of 1.0 as a
construction consistency test.

The public package contains observation references and English language, while
source tokens, future trajectories, map targets, witnesses, and contracts remain
private. Camera/LiDAR paths are relative to the OpenScene v1.1 test sensor root.
The release manifest identifies three deterministic archives: public inputs,
private evaluation assets, and evaluator code. Repeated public/private/evaluator builds are reported as byte-identical.

\paragraph{Common output, parsing, and quality.}
Predictions contain eight rear-axle poses in the anchor ego frame at
$0.5,1.0,\ldots,4.0$ s. Model headings are retained when supplied; otherwise headings
are computed from successive positions. Parse failures remain invalid and receive
zero completion. The released motion-quality summary uses maximum absolute
acceleration $\le3.5$ m/s$^2$ and maximum yaw rate $\le0.8$ rad/s. These quality checks are reported separately from the semantic contracts.

\paragraph{Binary and graded semantics.}
The released binary result uses own-contract success and opposite-contract failure
as in Eq.~\eqref{eq:repro-exclusive}. The graded function is implemented by
\path{task_completion} in \texttt{verify.py}. A trajectory satisfying the native
binary contract has completion 1; outside the accepted set, primitive scores decay
linearly over one native tolerance or accepted interval width. Conjunctions take
the minimum. Lane completion combines minimum vehicle-footprint overlap with
terminal progress. Stop completion combines the target-region requirement and
stopping speed. Event-order completion is 0 if neither event occurs, 0.5 if exactly one occurs,
and otherwise is graded by the event gap using the frozen verifier.
The evaluator records the full $2\times2$ matrix $C_a(\tau_b)$, but SCS and PCS use
only the own-arm entries; opposite rejection is a binary requirement, not a
universal multiplier on graded completion.

\begin{table}[!htbp]
\centering\footnotesize
\caption{NAVSIM task-specific scoring components. Global rules are shown here;
per-PCU geometry and thresholds are read from the frozen private sidecars.}
\label{tab:repro-navsim-parameters}
\begin{tabularx}{\linewidth}{@{}p{0.21\linewidth}X@{}}
\toprule
Component & Frozen scoring rule / parameter source\\
\midrule
Common grid / quality & 8 poses at 0.5--4.0 s in the rear-axle frame; maximum acceleration 3.5 m/s$^2$ and maximum yaw rate 0.8 rad/s\\
Lane / local goal & Vehicle-footprint overlap, terminal progress, and target polygons from the per-PCU private contract\\
Junction goal & Goal region, active interval, and deadline stored in the per-PCU contract\\
Landmark stop & Landmark-grounded stop region, stopping threshold, and hold requirement stored in the private contract\\
Speed band & Per-PCU lower/upper speed bands and active time range; primitive completion decays linearly over the native tolerance\\
Object-relative speed & Speed-band requirement combined with the bound lead-actor geometry and relative-position constraints\\
Pass / follow & Shared actor binding with lane/gap/overtaking constraints and task-specific active windows\\
Event order & Ordered event predicates; 0 if neither occurs, 0.5 if one occurs, otherwise event-gap grading\\
Collision audit & Logged-actor/ego-footprint overlap screening on the frozen 0.1 s audit grid\\
\bottomrule
\end{tabularx}
\end{table}

Exact per-PCU polygons, actor bindings, event thresholds, speed bands, and grading
tolerances are loaded from the frozen private sidecars by the released evaluator;
Table~\ref{tab:repro-navsim-parameters} summarizes how they enter scoring.

\paragraph{Aggregation and reporting.}
All headline metrics average four wordings equally within each PCU, followed by
an equal average across 38 PCUs. Families are not equally weighted. The paired
imbalance is checked against the algebraic identity
\begin{equation}
 \mathrm{SCS}-\mathrm{PCS}
 =50\left\langle|c_{ir,0}-c_{ir,1}|\right\rangle.
\label{eq:repro-imbalance}
\end{equation}
The recorded maximum absolute residual across models is $1.07\times10^{-14}$.
Score JSON top-level binary fields are canonical-only and lie in $[0,1]$;
\path{equal_weight_rewrites} contains the four-wording headline scores on the
0--100 scale. These fields are not interchangeable. A separate 3 s common-horizon diagnostic covers 22 PCUs from 13 recording clusters
and is reported independently from the full 4 s, 38-PCU result.

\paragraph{Model adapters and randomness.}
The five models are Qwen-Drive-1.0-4B, AutoVLA, DriveMA-2B, ReCogDrive, and OneVL.
Seed 42 is reset for every query. Deterministic generation is used where the released
interface supports it; sampling-based interfaces retain their released sampling
procedure. Required non-language command tensors in Qwen-Drive and ReCogDrive are
fixed to neutral/straight for all PCUs, arms, and wordings. Full language is the only
candidate-dependent input, although sensor coverage and history representation
differ across models.

\begin{table}[!htbp]
\centering\footnotesize
\caption{Frozen NAVSIM adapters. Input differences are retained and disclosed rather
than hidden by the common output grid.}
\label{tab:repro-navsim-adapters}
\begin{tabularx}{\linewidth}{@{}p{0.18\linewidth}XX@{}}
\toprule
Model & Released inputs and instruction placement & Common-grid conversion\\
\midrule
Qwen-Drive & Four frames of front/front-left/front-right; 16-point interpolation of observed 1.5 s history; full instruction in scene prompt & One reasoning-mode RL-planner sample; 10 Hz future points numbered 5,10,$\ldots$,40\\
AutoVLA & Four frames of three forward cameras; current velocity/acceleration; driving-instruction field & First eight released 0.5 s action-token poses\\
DriveMA-2B & Current three forward cameras; observed 1.5 s history and state; high-level-intent field & Released two-turn decision/trajectory protocol; 1 Hz positions linearly resampled to 0.5 s\\
ReCogDrive & Current front image; four observed poses and state; text replaces discrete command phrase & Released eight-pose DDIM planner output\\
OneVL & Current front image; three historical poses and state; full Command field & Released latent-chain-of-thought generation of eight 0.5 s poses\\
\bottomrule
\end{tabularx}
\end{table}

\paragraph{Implementation identity.}
The supplied protocol records PyTorch 2.5.1+cu124 on NVIDIA A100 40 GB GPUs;
Transformers 5.14.1 for Qwen-Drive, 4.57.0 for OneVL, and a released-compatible 4.49
development build for legacy adapters. These are the recorded experiment versions. Table~\ref{tab:repro-navsim-revisions} records the frozen
code and primary checkpoint revisions. Additional Qwen2.5-VL and ReCogDrive VLM
component revisions remain in \path{EXPERIMENT_PROTOCOL.md}.

\begin{table}[!htbp]
\centering\footnotesize
\caption{NAVSIM model revision identifiers from the frozen experiment protocol.
Line breaks inside identifiers are typographic only.}
\label{tab:repro-navsim-revisions}
\begin{tabularx}{\linewidth}{@{}lXX@{}}
\toprule
Model & Repository commit & Primary checkpoint revision\\
\midrule
Qwen-Drive & \texttt{28091c15\allowbreak 32e869bc\allowbreak 7aee91fc\allowbreak 0aef6b3e\allowbreak 6fd0b2e0} & \texttt{28484089\allowbreak a7cc8c33\allowbreak 5cf5089f\allowbreak b0745cf7\allowbreak c49b6eaa}\\
AutoVLA & \texttt{ba34eed7\allowbreak 4ce6729e\allowbreak 7986592d\allowbreak 0e66cbac\allowbreak a397b4fa} & \texttt{a7d7ba3e\allowbreak d7529b24\allowbreak 8d2694c2\allowbreak defa31b3\allowbreak 5208340f}\\
DriveMA-2B & \texttt{de20e6c6\allowbreak 4878bc34\allowbreak 13d26d69\allowbreak 362c3be9\allowbreak b3a91533} & \texttt{d5f4905d\allowbreak 2032da94\allowbreak 3a7ae8f2\allowbreak 2d5b2b8e\allowbreak fbe4b0f7}\\
ReCogDrive & \texttt{6b8d8f5e\allowbreak 01346c71\allowbreak 094651c8\allowbreak 1dcaf664\allowbreak 05dbc04e} & \texttt{7346ede7\allowbreak 3f03ed57\allowbreak 7811024f\allowbreak 1a5d16fb\allowbreak 063a47b9}\\
OneVL & \texttt{f4e6f055\allowbreak 72b2f666\allowbreak 11259bcd\allowbreak 55635298\allowbreak 33f30469} & \texttt{c1e893e9\allowbreak a03b8312\allowbreak f9328d2b\allowbreak 51267836\allowbreak a9f7bf8a}\\
\bottomrule
\end{tabularx}
\end{table}
\FloatBarrier

\subsection{WOD-E2E: constructed tasks and semantic trajectory evaluation}
\label{app:repro-wod}

\paragraph{Available data and evaluation denominators.}
The WOD-E2E construction scans 38 of the 93 numbered validation TFRecord
shards available locally. The resulting source contains 43,394 frames,
deduplicated by frame name and timestamp, and 203 rated anchors/sequences,
each associated with three candidate future trajectories.
Pairing the available trajectories produces 609 raw candidate pairs.
After requiring a complete 5\,s prediction window, 563 candidate pairs remain.
Among them, 441 pairs across 189 anchors admit two distinct English
instructions, resulting in 882 model requests.

The final semantic-evaluation set contains 241 PCUs and 482 requests across
126 anchors. Each PCU contains two mutually exclusive instructions defined
over the same recorded driving context. The evaluated set contains 165
\path{SPEED_CHOICE} PCUs and 76 \path{STOP_GO} PCUs.

\begin{table}[!htbp]
\centering\small
\caption{WOD-E2E P0 task-family distribution. Candidate construction,
language-pair construction, and the final semantic-evaluation set are
reported with separate denominators.}
\label{tab:repro-wod-distribution}
\begin{tabularx}{\linewidth}{@{}Xrrrrr@{}}
\toprule
Task family & Pairs & \% of 563 & \% of 441 & Requests & Evaluated PCUs\\
\midrule
\path{SPEED_CHOICE}
& 175 & 31.08 & 39.68 & 350 & 165\\
\path{STOP_GO}
& 112 & 19.89 & 25.40 & 224 & 76\\
\path{LATERAL_CHOICE}
& 112 & 19.89 & 25.40 & 224 & 0\\
\path{TURN_WAIT}
& 28 & 4.97 & 6.35 & 56 & 0\\
\path{TURN_STRAIGHT}
& 14 & 2.49 & 3.17 & 28 & 0\\
\texttt{UNRESOLVED}
& 122 & 21.67 & -- & 0 & 0\\
\midrule
Total & 563 & 100.00 & 100.00 & 882 & 241\\
\bottomrule
\end{tabularx}
\end{table}

\paragraph{Prediction horizon and coordinate convention.}
Each prediction is evaluated on the fixed grid
\begin{equation}
    t_k = 0.25k~\mathrm{s},
    \qquad k=1,\ldots,20,
\end{equation}
covering the future interval $(0,5]$\,s.
The implicit anchor $p_0=(0,0)$ represents the current ego position and is
not included among the submitted future points.
Positions are expressed in an anchor-fixed right-handed ego frame whose
origin is the rear-axle midpoint, with $+x$ forward, $+y$ left, and $+z$
upward. The semantic evaluator operates on the planar trajectory
$p_k=(x_k,y_k)$.

All reference and prediction trajectories are normalized to this common
20-point representation before scoring. DriveMA-2B and Fast-dDrive produce
five future positions at 1, 2, 3, 4, and 5\,s; the known current position is
prepended and linear interpolation produces the 4\,Hz evaluation sequence.
Qwen-Drive predicts a denser trajectory over the same 5\,s horizon, which is
resampled onto the same 0.25\,s evaluation grid. No private task target is
used to select among multiple predictions.

\paragraph{Observed context and model inputs.}
The recorded ego-state history contains 16 samples at 4\,Hz over the
preceding 4\,s window. The source data provide eight camera views, while the
evaluated adapters use the forward-facing views required by their respective
model interfaces. Temporal observations, ego history, calibration, and
high-level driving context are fixed within each PCU.

For the two arms of a PCU, the model receives the same scene identifier,
visual observations, state history, calibration, and non-language context.
Only the natural-language instruction changes. All task-specific semantic
contracts and scoring parameters remain private to the evaluator.

\paragraph{Counterfactual task qualification.}
The reported WOD-E2E evaluation focuses on the two task families for which
the future trajectory directly exposes a complete and mutually exclusive
behavioral distinction: \path{STOP_GO} and \path{SPEED_CHOICE}.

Each PCU contains two frozen semantic contracts,
$P_0$ and $P_1$. Qualification requires that both target behaviors form
valid opposed alternatives under the same recorded scene context. In
particular, each reference arm must satisfy its own semantic contract while
rejecting the paired alternative:
\begin{equation}
    P_a(\tau^{\mathrm{ref}}_a)=1,
    \qquad
    P_{1-a}(\tau^{\mathrm{ref}}_a)=0,
    \qquad a\in\{0,1\}.
\end{equation}
Applying the frozen qualification procedure yields 76 \path{STOP_GO}
PCUs and 165 \path{SPEED_CHOICE} PCUs, for a fixed evaluation denominator
of 241 PCUs and 482 instruction-conditioned predictions.

\paragraph{Trajectory features and valid-output gate.}
For the submitted future trajectory
$\tau=(p_1,\ldots,p_{20})$ and the implicit current position $p_0=(0,0)$,
we define
\begin{equation}
\begin{aligned}
v_k
&=
\frac{\|p_k-p_{k-1}\|_2}{0.25},
\\
L
&=
0.25\sum_{k=1}^{20}v_k,
\\
v_{\mathrm{tail}}
&=
\frac{1}{4}\sum_{k=17}^{20}v_k,
\\
v_{\mathrm{late}}
&=
\frac{1}{9}\sum_{k=12}^{20}v_k .
\end{aligned}
\label{eq:repro-wod-features}
\end{equation}

Here $v_{\mathrm{tail}}$ summarizes the final second of prediction, while
$v_{\mathrm{late}}$ summarizes the later portion of the 5\,s horizon used by
the speed-choice contracts.

A valid submission must contain exactly 20 future positions, with exactly
two finite numeric coordinates per point. We additionally require
\begin{equation}
    \max_k v_k \le 50~\mathrm{m/s},
    \qquad
    \max_{k=2,\ldots,20}
    \frac{|v_k-v_{k-1}|}{0.25}
    \le 20~\mathrm{m/s^2}.
\end{equation}
Missing or invalid predictions receive zero binary and graded completion
credit and remain in the fixed 241-PCU denominator.

\paragraph{Binary semantic criteria and graded completion.}
For $a<b$, define the clipped linear ramps
\begin{equation}
\begin{aligned}
U(z;a,b)
&=
\min\left\{
1,
\max\left\{
0,
\frac{z-a}{b-a}
\right\}
\right\},
\\
D_{\downarrow}(z;a,b)
&=
1-U(z;a,b).
\end{aligned}
\label{eq:repro-wod-ramps}
\end{equation}

For each \path{STOP_GO} PCU, let
$q_{\mathrm{stop}}$ denote the frozen stopping target and
\begin{equation}
    d_{\mathrm{stop}}
    =
    \|p_{20}-q_{\mathrm{stop}}\|_2.
\end{equation}
The binary semantic predicates are
\begin{equation}
\begin{aligned}
P_{\mathrm{stop}}(\tau)
&=
[v_{\mathrm{tail}}\le0.6]
\land
[d_{\mathrm{stop}}\le2],
\\
P_{\mathrm{go}}(\tau)
&=
[v_{\mathrm{tail}}\ge2]
\land
[x_{20}\ge8],
\end{aligned}
\label{eq:repro-wod-stop-go-binary}
\end{equation}
and the corresponding graded completion scores are
\begin{equation}
\begin{aligned}
C_{\mathrm{stop}}(\tau)
&=
\min\{
D_{\downarrow}(v_{\mathrm{tail}};0.6,2),
D_{\downarrow}(d_{\mathrm{stop}};2,5)
\},
\\
C_{\mathrm{go}}(\tau)
&=
\min\{
U(v_{\mathrm{tail}};0.6,2),
U(x_{20};0,8)
\}.
\end{aligned}
\label{eq:repro-wod-stop-go}
\end{equation}

For each \path{SPEED_CHOICE} PCU, the private evaluator stores two
pair-specific speed boundaries
\begin{equation}
    \ell_i < h_i ,
\end{equation}
which define the two mutually exclusive requested speed behaviors for PCU
$i$. The binary predicates are
\begin{equation}
\begin{aligned}
P_{\mathrm{slow}}(\tau)
&=
[v_{\mathrm{late}}\le\ell_i]
\land
[L\ge2],
\\
P_{\mathrm{normal}}(\tau)
&=
[v_{\mathrm{late}}\ge h_i]
\land
[L\ge2],
\end{aligned}
\label{eq:repro-wod-speed-binary}
\end{equation}
with graded completion
\begin{equation}
\begin{aligned}
C_{\mathrm{slow}}(\tau)
&=
\min\{
D_{\downarrow}(v_{\mathrm{late}};\ell_i,h_i),
U(L;0,2)
\},
\\
C_{\mathrm{normal}}(\tau)
&=
\min\{
U(v_{\mathrm{late}};\ell_i,h_i),
U(L;0,2)
\}.
\end{aligned}
\label{eq:repro-wod-speed}
\end{equation}

\begin{table}[!htbp]
\centering\small
\caption{WOD-E2E semantic scoring parameters. Speed and distance are
reported in m/s and m, respectively. The speed-choice boundaries
$\ell_i$ and $h_i$ are frozen per-PCU evaluator parameters.}
\label{tab:repro-wod-semantic-parameters}
\begin{tabularx}{\linewidth}{@{}Xlll@{}}
\toprule
Criterion & Binary success & Completion $1$ & Completion $0$\\
\midrule
Stop tail speed
& $v_{\mathrm{tail}}\le0.6$
& $\le0.6$
& $\ge2$\\

Stop endpoint distance
& $d_{\mathrm{stop}}\le2$
& $\le2$
& $\ge5$\\

Go tail speed
& $v_{\mathrm{tail}}\ge2$
& $\ge2$
& $\le0.6$\\

Go forward endpoint
& $x_{20}\ge8$
& $\ge8$
& $\le0$\\

Slow late speed
& $v_{\mathrm{late}}\le\ell_i$
& $\le\ell_i$
& $\ge h_i$\\

Normal late speed
& $v_{\mathrm{late}}\ge h_i$
& $\ge h_i$
& $\le\ell_i$\\

Speed-task progress
& $L\ge2$
& $\ge2$
& $\le0$\\
\bottomrule
\end{tabularx}
\end{table}

\paragraph{Semantic success and paired aggregation.}
For instruction arm $a$, binary semantic success requires satisfying the
requested criterion while rejecting the paired alternative:
\begin{equation}
    B_a
    =
    P_a(\tau_a)
    \land
    \neg P_{1-a}(\tau_a).
\end{equation}
The corresponding continuous value is the own-arm completion
$C_a(\tau_a)\in[0,1]$.

Using the common open-loop aggregation defined in
Section~\ref{sec:metrics}, WOD-E2E reports
\begin{equation}
\begin{aligned}
\mathrm{SemSISR}
&=
100
\left\langle
\frac{B_0+B_1}{2}
\right\rangle,
\\
\mathrm{SemCFPS}
&=
100
\left\langle
B_0B_1
\right\rangle,
\\
\mathrm{SCS}
&=
100
\left\langle
\frac{C_0+C_1}{2}
\right\rangle,
\\
\mathrm{PCS}
&=
100
\left\langle
\min(C_0,C_1)
\right\rangle .
\end{aligned}
\label{eq:repro-wod-metrics}
\end{equation}
All 241 qualified PCUs receive equal weight. Missing or invalid model
outputs contribute zero rather than reducing the denominator. This
definition is identical in form to the semantic paired evaluation used by
the other Language2Drive open-loop tracks, while retaining WOD-E2E-specific
task predicates.

\paragraph{Geometric response diagnostic.}
We additionally report a trajectory-alignment diagnostic over the 441
instruction-paired samples. For prediction $\tau$ and reference $R$, define
\begin{equation}
    \operatorname{ADE}(\tau,R)
    =
    \frac{1}{20}
    \sum_{k=1}^{20}
    \|p_k-R_k\|_2 .
\end{equation}
For each arm,
\begin{equation}
    m
    =
    \operatorname{ADE}(\tau,R_{\mathrm{other}})
    -
    \operatorname{ADE}(\tau,R_{\mathrm{own}}).
\end{equation}
An arm is counted as geometrically aligned when
$m>0.5$\,m, and a pair is aligned only when both valid arms satisfy this
criterion. This diagnostic measures whether the two instruction-conditioned
predictions move toward their corresponding alternatives and is reported
separately from SemSISR, SemCFPS, SCS, and PCS.

\begin{table}[!htbp]
\centering\small
\caption{WOD-E2E geometric diagnostics on the 441 instruction-paired
samples. ADE is measured in meters.}
\label{tab:wod-geometry}
\begin{tabularx}{\linewidth}{@{}Xrrr@{}}
\toprule
System & Valid outputs & Own-reference ADE $\downarrow$ & Aligned pairs\\
\midrule
DriveMA-2B
& 882/882 & 2.520 & 6/441\\
Fast-dDrive
& 882/882 & 2.817 & 5/441\\
Qwen-Drive-SFT
& 882/882 & 2.959 & 3/441\\
Qwen-Drive-RL
& 876/882 & 2.924 & 2/436\\
\bottomrule
\end{tabularx}
\end{table}

The paired reference trajectories have a mean mutual ADE of
3.962\,m, whereas the mean separation between paired model predictions is
0.178--0.355\,m across the four evaluated systems. These geometric
measurements complement the semantic success metrics by characterizing the
magnitude of the trajectory response to different instructions.

\paragraph{Reporting scope.}
WOD-E2E results are reported as protocol-specific open-loop
instruction-following measurements on the frozen evaluation set.
They are interpreted together with the nuScenes and NAVSIM tracks using the
common Language2Drive open-loop metrics
SemSISR, SemCFPS, SCS, and PCS. Dataset-specific task predicates remain
separate, and scores are not pooled across datasets.

\subsection{Language interventions and statistical uncertainty}
\label{app:repro-controls-statistics}
% SOURCES: REPORT_ZH.md; DATASET_AND_EVALUATION_PROTOCOL.md Section 10;
% NAVSIM artifacts/scores_*.json; original reproducibility.tex.

\paragraph{nuScenes intervals.}
The original binary SemISR and Sem-CFPS intervals use 1,000 scene-cluster bootstrap
replicates with Python \texttt{random.Random(2026)}. After sorting the replicate
values, the endpoints have indices $\lfloor0.025(B-1)\rfloor$ and
$\lfloor0.975(B-1)\rfloor$, with $B=1000$. Supplementary continuous metrics use
2,000 replicates with NumPy \path{default_rng(2026)} and linearly interpolated
\texttt{quantile([.025,.975])}. Family-macro intervals resample scenes independently
inside each family before averaging the family means; replicates with an undefined
valid-output denominator are excluded. A scene currently equals one PCU, with its
paired arms retained together. The paired arms of each scene are retained together throughout resampling.

\paragraph{NAVSIM intervals.}
Stored canonical binary intervals use 1,000 bootstrap replicates over 19 original
recordings (13 for the 22-PCU common-3 s subset), retaining the PCUs within each
selected recording. The canonical NAVSIM intervals are computed over recording clusters; WOD-E2E proxy
results are reported deterministically for the frozen prediction set.

\paragraph{Completed RealSim intervention subset.}
The September 26 supplemental report adds 42 successful runtime conditions:
two SimLingo versions $\times$ three PCUs $\times$ seven conditions. These conditions
are normal, two same-instruction repeats, opposite, empty, neutral, and shuffled.
The selected PCUs are the first speed, lane, and route PCUs in frozen source order.
Each condition independently resets the closed-loop session, while the four-camera,
pose, route, and other non-language prefixes match before intervention. Normal and
its two repeats share seed 1701.

Same-instruction NBD averages all three pairwise comparisons among the normal
runs. Opposite-instruction NBD and SRG average comparisons against those three runs.
Opposite-minus-repeat differences retain negative values. Shuffled cyclically
reassigns the complete normal instruction across the three scenes, rather than
shuffling words. Empty removes only the task-instruction string, retaining the
model's state and task template. Empty/neutral/shuffled completion is scored against
the original target solely as an ablation. Zero common semantic samples follow the scorer's fixed zero convention and are
accompanied by active-sample counts and an informative flag. These controls are
reported as supplemental diagnostics alongside the full 30-PCU RealSim comparison.

\subsection{Reproduction assets and checksums}
\label{app:repro-artifacts}
% SOURCE: nuScenes protocol Section 12; NAVSIM release manifest; WOD spec Section 9.

\paragraph{Evidence layers.}
Exact replay is organized at three levels. Aggregate score JSON/CSV files reproduce
the reported tables; normalized predictions, private PCU contracts, and the frozen
scorer reproduce trajectory-level metrics; rerunning inference additionally uses
the corresponding upstream sensors, public manifests, checkpoints, adapters, and
recorded software environments. Release manifests and SHA-256 hashes bind these
assets to the experiment version used in this appendix.

The nuScenes protocol exports construction statistics, scoring parameters, and source
hashes. NAVSIM provides public/private/evaluator archive identities together with
score matrices, protocol tables, and release-validation records. WOD-E2E records the
frozen P0 proxy manifest, per-instance proxy contracts, adapter outputs, and scoring
artifacts. These assets preserve the mapping from each reported aggregate back to
its PCU-level prediction and scoring configuration.

\begin{table}[!htbp]
\centering\footnotesize
\caption{Recorded release identities used to bind the reported experiments to
their public, private, evaluator, and scoring artifacts.}
\label{tab:repro-hashes}
\begin{tabularx}{\linewidth}{@{}p{0.30\linewidth}X@{}}
\toprule
Artifact & Recorded SHA-256\\
\midrule
nuScenes public inputs & \texttt{b2d3a6e2\allowbreak deccb91e\allowbreak 1e7d1694\allowbreak ff9589d5\allowbreak 577eb89e\allowbreak 834c674d\allowbreak c36768bd\allowbreak 213962d5}\\
nuScenes private sidecars & \texttt{f2506b15\allowbreak 8aa34440\allowbreak 999ef621\allowbreak 5fe21dc0\allowbreak 8fccd513\allowbreak 8d6db5a0\allowbreak ac4fdc5e\allowbreak 8d9ea7ce}\\
nuScenes review sheet & \texttt{6744858c\allowbreak f01921f4\allowbreak 3ba6fca1\allowbreak 1ef7d63d\allowbreak 6f3a4512\allowbreak 03f897c8\allowbreak 7e289e4a\allowbreak 6301b30f}\\
nuScenes binary scorer & \texttt{436bf1fc\allowbreak b8393b2b\allowbreak 2da27974\allowbreak c363eef8\allowbreak 13dc7384\allowbreak 9a7063b1\allowbreak aeb947c2\allowbreak 9fef1eeb}\\
NAVSIM public archive & \texttt{0e9f010a\allowbreak e39f9dec\allowbreak c049ff4f\allowbreak f9f79109\allowbreak e92cd27d\allowbreak f27f55b5\allowbreak cc80a18f\allowbreak 35d0e799}\\
NAVSIM private archive & \texttt{e520a6f4\allowbreak 455b82e1\allowbreak e6dbfe3c\allowbreak 1de30b07\allowbreak 86153f5a\allowbreak 3dbedf85\allowbreak 9aae7408\allowbreak 12bfea6f}\\
NAVSIM evaluator archive & \texttt{720b3bfd\allowbreak 3c899933\allowbreak 7a5b4b08\allowbreak 7426c0ba\allowbreak 95414b56\allowbreak dc8c0b27\allowbreak 371d3a32\allowbreak a46ebc04}\\
NAVSIM scorer & \texttt{f86ecc33\allowbreak 1facc313\allowbreak 7fe20f4a\allowbreak d2300bc6\allowbreak c035b451\allowbreak 6aecb5bb\allowbreak edc35684\allowbreak 2b3f8032}\\
NAVSIM verifier & \texttt{ecde4152\allowbreak 2adb70df\allowbreak 98487031\allowbreak 6ce2a30b\allowbreak 548f0101\allowbreak eeeeaf80\allowbreak 9c97f499\allowbreak 19195fd1}\\
\bottomrule
\end{tabularx}
\end{table}

\paragraph{Minimum portable replay contract.}
Each private record must retain its PCU/source cluster ID, frozen arm texts and
rewrites, public-input hash, coordinate/reference metadata, time grid, own and
opposite executable predicates, all instance geometry and numeric tolerances,
continuous-kernel version, feasibility witnesses, and qualification decision.
A prediction record must bind the condition, seed, checkpoint, adapter, original
output, normalized trajectory, validity, and selected attempt. Scoring must verify
ID coverage, paired input equality, time/frame consistency, private-label isolation,
and the declared eligible denominator before evaluating both targets on each output.
The scorer verifies ID coverage and the fixed eligible denominator before computing
the paired metrics.

For nuScenes, the referenced reproduction entry points are
\path{model_eval/REPRODUCE.md}, \path{protocol/evaluation_parameters.json},
\path{protocol/dataset_statistics.json}, and
\path{protocol/protocol_evidence.json}. NAVSIM replay uses the public/private/evaluator archives identified by the release
manifest, including the frozen geometry dependencies. WOD-E2E replay uses the frozen
P0 manifest, \path{proxy_contracts.jsonl}, the
\path{p0_kinematic_proxy_v0.1.0} scorer, normalization adapters, and
per-request protocol records. Together with the nuScenes resources above, these files
fully specify the reported open-loop evaluation interfaces and aggregation rules.
\FloatBarrier

\clearpage
\section{Extended quantitative results}
\label{app:analysis}\label{app:split}
\subsection{Task scope}
% Table~\ref{tab:success-decomposition} reports counts behind the single--paired gap. If $n_2$ pairs have two strict successes and $n_1$ have one, successful arms total $2n_2+n_1$. Across the six principal CARLA systems, $n_2=4$ and $n_1=110$, so 110/118 successful arms (93.22\%) have no successful counterpart. Including the separately labeled Senna bridge would give 157/165 (95.15\%); the primary analysis keeps that configuration separate.
% \input{source/tables/success_decomposition}

% The closed-loop relaxed proxy and strict execution assess different completion requirements. SimLingo has 187 relaxed-proxy successes: 123 overlap safety violations and 64 are safe. Of those 64, 30 succeed strictly and 34 do not. Its full requested-criterion count is also 30, so all 157 non-strict relaxed successes already fail full requested completion. These overlapping components are not an additive attribution of independent failure causes.
\begin{table}[!htbp]
\centering\small
\caption{Strict CARLA instruction success by scope. Each scope has 16 PCUs, 96 rollouts, and 48 paired instances per system. A: atomic; B: persistent; C: grounded; D: program.}
\label{tab:scope-results}\label{tab:scope-sisr}
\setlength{\tabcolsep}{10pt}
\begin{tabularx}{\linewidth}{@{}Xrrrr@{}}
\toprule
System & A & B & C & D\\
\midrule
LMDrive & \textbf{13.54} & \textbf{22.92} & 6.25 & \textbf{6.25}\\
Alpamayo 1.5 & 3.12 & 16.67 & \textbf{12.50} & 0.00\\
SimLingo & \textbf{13.54} & 12.50 & 4.17 & 1.04\\
Qwen-Drive-1.0-4B & 0.00 & 4.17 & 3.12 & 0.00\\
OneVL & 0.00 & 1.04 & 1.04 & 0.00\\
Impromptu-VLA & 1.04 & 0.00 & 0.00 & 0.00\\
\bottomrule
\end{tabularx}
\end{table}

Scope D contains five ordered-visit PCUs, five conditional-fallback PCUs, four reference-memory PCUs, and two until-invariant PCUs. Across six systems these contribute 180, 180, 144, and 72 rollouts. The first three families have no strict successes; all seven D successes occur in until-invariant tasks. Across A--C the leading strict system changes. Persistent scope B exceeds atomic scope A for LMDrive (22.92 versus 13.54) and Alpamayo~1.5 (16.67 versus 3.12), confirming that the scopes are not a monotonic difficulty ordering.
\begin{figure}[!htbp]
\centering\includegraphics[width=0.88\linewidth]{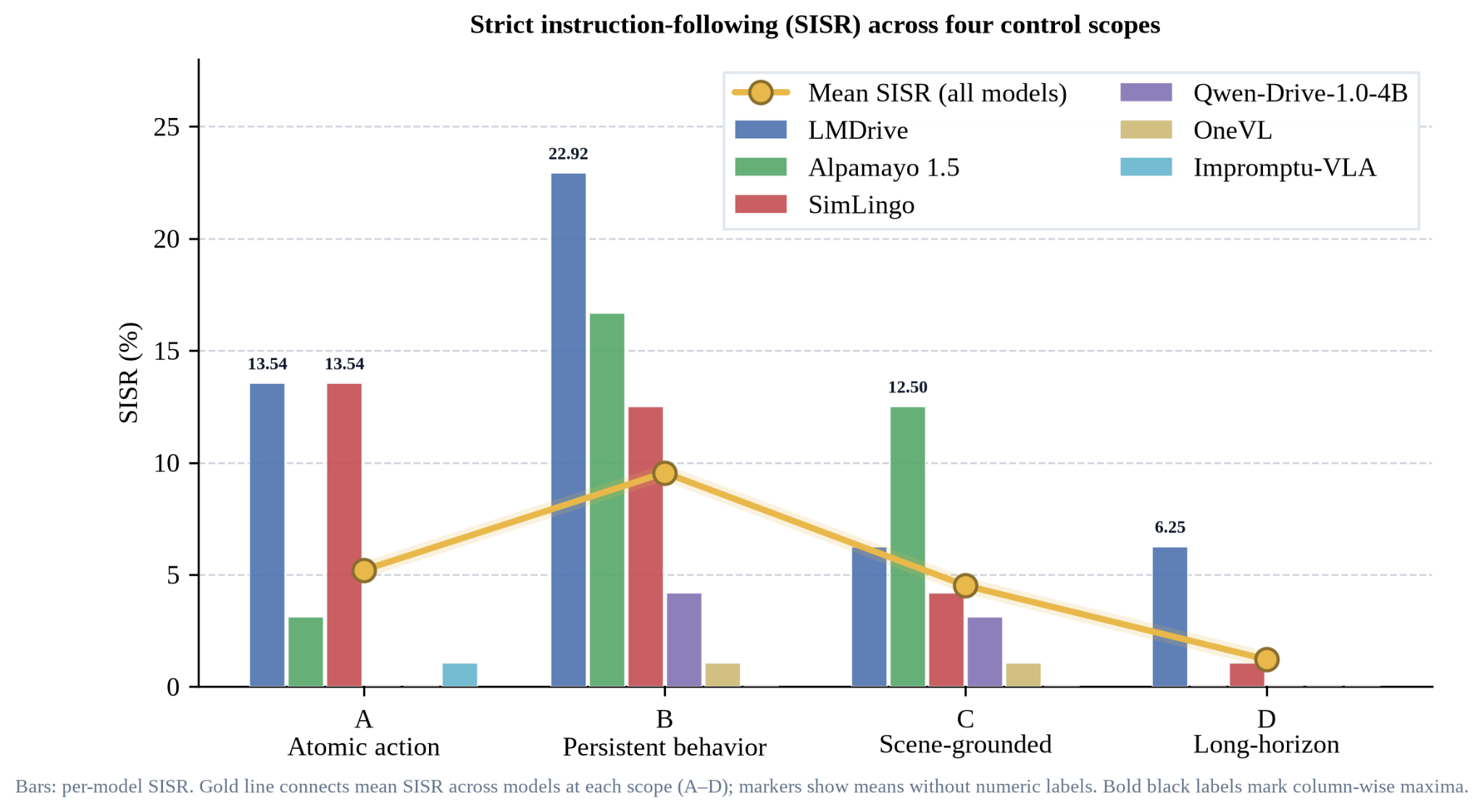}
\caption{Strict CARLA success by control scope, with 96 rollouts per system and scope. Gold markers average six systems. D has mean SISR 1.22\%; its seven strict successes all maintain an invariant until an event. Connecting lines join categorical scopes rather than a continuous difficulty axis.}
\label{fig:scope-results}
\end{figure}

\subsection{Development/test split and wording robustness}
% Table~\ref{tab:split} partitions the frozen CARLA cohort into 84 development and 300 test rollouts per system. The two subsets partition the same evaluation cohort.
% \input{source/tables/split}

For one PCU/arm, its strict outcomes form a three-bit expression pattern. A changed outcome has one or two successes; an all-expression success has three. Conditioning the changed-outcome count on arms with at least one success separates fragile successes from consistent failures. SimLingo changes on 9 of 15 sometimes-successful arms (60\%), Alpamayo~1.5 on 12 of 16 (75\%), and LMDrive on three of 17 (17.65\%).
\begin{table}[!htbp]
\caption{Instruction sensitivity across three CARLA expressions. Counts use 128 semantic arms per system. ``All three requires the same arm to succeed under every expression; ``any requires at least one. Flip/any conditions on arms with some strict success.}
\label{tab:sensitivity}
\centering\small
\setlength{\tabcolsep}{4pt}
\begin{tabular}{@{}lrrrrr@{}}
\toprule
 & \multicolumn{4}{c}{Strict outcomes} & Semantic\\
\cmidrule(lr){2-5}
System & Flip & All three & Any & Flip/any (\%) & Flip\\
\midrule
LMDrive & 3 & 14 & 17 & 17.65 & 5\\
Alpamayo 1.5 & 12 & 4 & 16 & 75.00 & 8\\
SimLingo & 9 & 6 & 15 & 60.00 & 7\\
Qwen-Drive-1.0-4B & 2 & 1 & 3 & 66.67 & 9\\
OneVL & 2 & 0 & 2 & 100.00 & 5\\
Impromptu-VLA & 1 & 0 & 1 & 100.00 & 15\\
Senna bridge & 1 & 15 & 16 & 6.25 & 4\\
\bottomrule
\end{tabular}
\end{table}

\begin{table}[!htbp]
\caption{CARLA paired and paraphrase diagnostics. A--D: successful pairs out of 48. $I_s/I_m$: numbers of arm groups with inconsistent strict/semantic outcomes across paraphrases, out of 128. All-3: arms strictly successful for all three paraphrases, out of 128. All systems have zero worst-paraphrase paired success out of 64 PCUs.}\label{tab:pairs}
\centering\small
\setlength{\tabcolsep}{3.5pt}
\begin{tabular}{@{}lrrrrrrr@{}}
\toprule
System & A & B & C & D & $I_s$ & $I_m$ & All-3\\
\midrule
LMDrive & 0 & 0 & 0 & 0 & 3 & 5 & 14\\
SimLingo & 2 & 0 & 0 & 0 & 9 & 7 & 6\\
Alpamayo 1.5 & 0 & 2 & 0 & 0 & 12 & 8 & 4\\
Qwen-Drive-1.0-4B & 0 & 0 & 0 & 0 & 2 & 9 & 1\\
OneVL & 0 & 0 & 0 & 0 & 2 & 5 & 0\\
Impromptu-VLA & 0 & 0 & 0 & 0 & 1 & 15 & 0\\
Senna VLM bridge & 0 & 0 & 0 & 0 & 1 & 4 & 15\\
\bottomrule
\end{tabular}
\end{table}

\begin{table}[!htbp]
\caption{CARLA per-paraphrase outcomes. S/M entries give strict/semantic success counts out of 128; pair columns count successful pairs out of 64.}\label{tab:phrases}
\centering\small
\setlength{\tabcolsep}{3.5pt}
\begin{tabular}{@{}lrrrrrr@{}}
\toprule
System & p0 S/M & p1 S/M & p2 S/M & p0 pair & p1 pair & p2 pair\\
\midrule
LMDrive & 15/36 & 17/38 & 15/38 & 0 & 0 & 0\\
SimLingo & 13/61 & 8/64 & 9/62 & 2 & 0 & 0\\
Alpamayo 1.5 & 11/52 & 11/55 & 9/55 & 1 & 1 & 0\\
Qwen-Drive-1.0-4B & 2/20 & 3/19 & 2/21 & 0 & 0 & 0\\
OneVL & 0/31 & 2/33 & 0/32 & 0 & 0 & 0\\
Impromptu-VLA & 1/25 & 0/23 & 0/28 & 0 & 0 & 0\\
Senna VLM bridge & 16/35 & 16/34 & 15/37 & 0 & 0 & 0\\
\bottomrule
\end{tabular}
\end{table}

SimLingo's A-006 left and straight patterns are \texttt{100} and \texttt{111}; A-010's two patterns are \texttt{110} and \texttt{101}. Each PCU completes both alternatives only at expression~0. Alpamayo~1.5's successful pairs occur on B-011/p0 and B-027/p1. No system achieves all-expression paired success $W>0$. These observations describe the evaluated systems under the tested wordings; single-seed evaluation does not isolate the contribution of execution variability.

\subsection{Graded completion and the L2D summary}
\begin{table}[!htbp]
\centering\small
\caption{L2D components. The language and driving contributions sum to the final score before rounding. Calculations use the displayed component metrics in Tables~\ref{tab:carla-main} and~\ref{tab:realsim-main}.}
\label{tab:l2d-components}
\setlength{\tabcolsep}{5pt}
\begin{tabularx}{\linewidth}{@{}Xrrrrr@{}}
\toprule
System & $L_{\rm bal}$ & $D_t$ & Driving credit & Language credit & L2D\\
\midrule
\multicolumn{6}{@{}l}{\textit{CARLA}}\\
Alpamayo 1.5 & 17.56 & 66.40 & 0.66 & 11.54 & 12.20\\
SimLingo & 18.26 & 63.67 & 0.64 & 11.51 & 12.15\\
LMDrive & 14.09 & 69.48 & 0.69 & 9.69 & 10.39\\
OneVL & 6.66 & 64.21 & 0.64 & 4.23 & 4.87\\
Qwen-Drive-1.0-4B & 6.23 & 60.18 & 0.60 & 3.71 & 4.31\\
Impromptu-VLA & 5.54 & 59.73 & 0.60 & 3.27 & 3.87\\
Senna VLM bridge & 12.37 & 74.62 & 0.75 & 9.14 & 9.89\\
\midrule
\multicolumn{6}{@{}l}{\textit{RealSim}}\\
Alpamayo2-Super & 8.77 & 95.22 & 0.95 & 8.26 & 9.22\\
Alpamayo 1.5 & 7.57 & 98.59 & 0.99 & 7.39 & 8.38\\
SimLingo & 9.10 & 75.71 & 0.76 & 6.82 & 7.58\\
AutoVLA & 4.72 & 59.31 & 0.59 & 2.77 & 3.36\\
Qwen-Drive-1.0-4B & 0.64 & 56.25 & 0.56 & 0.36 & 0.92\\
Impromptu-VLA & 0.58 & 49.43 & 0.49 & 0.28 & 0.78\\
OneVL & 0.00 & 43.58 & 0.44 & 0.00 & 0.44\\
ReCogDrive & 0.00 & 38.45 & 0.38 & 0.00 & 0.38\\
\bottomrule
\end{tabularx}
\end{table}

The SCS--PCS gap is half the average absolute arm-completion difference on the reporting scale. For CARLA Alpamayo~1.5 it is 13.68 points; for LMDrive, 12.19; for SimLingo, 11.57. In RealSim the gap is 21.43 for Alpamayo2-Super, 17.01 for Alpamayo~1.5, and 6.67 for SimLingo. These gaps use rounded aggregate scores; exact trajectory-level gaps should be exported before rounding.

The L2D formula retains partial completion while limiting independent driving credit. SimLingo has the largest RealSim language component, 9.1004, but Alpamayo2-Super has the largest composite, 9.22, after the driving components are included. All six language metrics are zero for RealSim OneVL and ReCogDrive; their unrounded totals are exactly 0.43575 and 0.38450. Positive driving-only credit is not counted as instruction success. The CARLA top-two unrounded difference is about 0.05834 using the displayed component metrics; no significance claim follows from that small gap.

\subsection{Reconstructed-scene scope and response}
\begin{table}[!htbp]
\caption{RealSim relaxed semantic outcomes by scope for five systems. Counts use 12 atomic arms and 48 persistent arms; overall rates agree with Table~\ref{tab:realsim-main}.}
\label{tab:real-layers}
\centering\small
\begin{tabular}{@{}lrrr@{}}
\toprule
System & A: atomic action & B: persistent behavior & Overall\\
\midrule
Alpamayo 1.5 & 4/12 (33.33) & 5/48 (10.42) & 9/60 (15.00)\\
Alpamayo2-Super & 5/12 (41.67) & 6/48 (12.50) & 11/60 (18.33)\\
SimLingo & 3/12 (25.00) & 5/48 (10.42) & 8/60 (13.33)\\
OneVL & 0/12 (0.00) & 0/48 (0.00) & 0/60 (0.00)\\
ReCogDrive & 0/12 (0.00) & 0/48 (0.00) & 0/60 (0.00)\\
\bottomrule
\end{tabular}
\end{table}
\label{app:sensitivity}
In the RealSim trajectory-response analysis, each complete pair is sampled on 101 common times from instruction at 1.7\,s to its shared horizon. Mean Euclidean position separation and mean absolute speed difference are computed per pair, then summarized by the median. Table~\ref{tab:real-response} reports spatial separation over all complete pairs and speed separation over the 21 B1/B3 speed-task pairs. Per-system valid-pair counts specify the diagnostic sample coverage.
Across 30 pairs, median position separation is 0.15\,m for Alpamayo~1.5, 0.23\,m for Alpamayo2-Super, and 2.61\,m for SimLingo. Their median speed differences on speed tasks are 0.21, 0.46, and 3.51\,km/h. Greater response does not by itself establish the requested direction, completed scope, or safe execution. P11 and P22 in Figure~\ref{fig:measured-case-comparison} connect those measurements to the scorer's individual and paired verdicts.

\begin{table}[!htbp]
\caption{RealSim response to opposed instructions in the trajectory-response analysis. Spatial values are medians over complete valid pairs; speed values use the B1/B3 subset. Each pair first averages differences over 101 matched post-instruction time points. These describe response magnitude, with correctness evaluated separately by CFPS.}
\label{tab:real-response}
\centering\small
\begin{tabular}{@{}lrrrr@{}}
\toprule
System & Pairs & Separation (m) & Speed pairs & Speed diff. (km/h)\\
\midrule
Alpamayo2-Super & 30 & 0.23 & 21 & 0.46\\
Alpamayo 1.5 & 30 & 0.15 & 21 & 0.21\\
SimLingo & 30 & 2.61 & 21 & 3.51\\
OneVL & 30 & 0.36 & 21 & 0.25\\
ReCogDrive & 30 & 0.10 & 21 & 0.07\\
Qwen-Drive-1.0-4B & 29 & 0.31 & 20 & 0.48\\
Impromptu-VLA & 26 & 3.07 & 21 & 4.36\\
AutoVLA & 20 & 2.25 & 16 & 3.54\\
\bottomrule
\end{tabular}
\end{table}

\FloatBarrier

\section{Diagnostic experiments}
\label{app:fdq}\label{app:driving}
\subsection{Core progress versus independent driving routes}
CARLA DriveScore averages the maximum normalized projection onto each task-specific private path:
\begin{equation}
\mathrm{DriveScore}=\frac{100}{|\mathcal V|}\sum_{v\in\mathcal V}\max_t\pi_v(t),\qquad \pi_v(t)\in[0,1].
\label{eq:drivescore}
\end{equation} All seven displayed CARLA configurations have 384 valid core traces. The valid-trajectory denominator is conceptually distinct from the fixed scheduled denominator used for instruction metrics. Progress need not be monotonic in time; the maximum records the furthest attained point. A high score does not enforce correct stage order, target exclusivity, dwell, or the entire instruction contract.

RealSim DrivingScore uses the frozen arm-specific Oracle plan, truncated at the task deadline. Each actual ego position is projected onto the reference polyline; a projection is accepted within a 5\,m corridor. The largest accepted arc length, divided by reference length and clipped at one, gives the rollout score. Summed contributions use the fixed 60-slot denominator. This measures short-window route advancement; the binary 10\,m Progress test includes the interval before the 1.7\,s instruction. Safety is reported separately. With reference length $\lambda_v$, projected arc length $a_v(t)$, and distance $e_v(t)$ to the path, the accepted advancement is
\[
g_v=\operatorname{clip}\left(\frac{\max(\{0\}\cup\{a_v(t):t\le T_v,\ e_v(t)\le5\,\mathrm m\})}{\lambda_v},0,1\right).
\]
DrivingScore sums the valid advancement contributions over the fixed 60 scheduled slots; failed or unavailable runs are not removed to increase its mean.

\subsection{FDQ, Bench2Drive driving score, and route completion}
The FDQ set contains fixed, loadable Bench2Drive scenario routes, one per ability group and seed 1. Its frozen strict pass rule requires a Completed/Perfect status and every recorded infraction list to be empty. The primary FDQ rate is the mean ability-group pass rate, which here equals the number of passing routes divided by 24.

The scenario set includes overtaking, two-way accidents, blocked intersections, construction obstacles, control loss, emergency braking, entering traffic flow, hard-braking routes, side-lane hazards, unsignalized left turns, red-light-running opponents, parked obstacles, parking-related crossings/cut-ins, merging, pedestrians, signalized junctions, static cut-ins, T-junctions, and yielding to emergency vehicles. The routes are supplementary driving-competence diagnostics rather than substitutes for language-dependent task pairs.

\subsection{Supplementary FDQ measurements}
The separate FDQ metric file contains five diagnostic measures, with route-equal means in Table~\ref{tab:fdq-extra}. They do not alter the strict FDQ pass rule.
\begin{align}
 \mathrm{SI}_i &= \exp[-(3N_{\mathrm{ped}}+2N_{\mathrm{veh}}+N_{\mathrm{layout}})],\\
 \mathrm{GPE}_i &= r_i\exp\left[-0.5\max\left(0,\frac{T_i}{T_{\mathrm{ref}}(r_i)}-1\right)\right],
 \quad r_i=\mathrm{RC}_i/100,\\
 \mathrm{PSM}_i &= \frac{\sum_t\Delta t\,\operatorname{clip}(\mathrm{TTC}_{\min}(t)/3,0,1)}{\sum_t\Delta t},\\
 \mathrm{MDQ}_i &= \exp\left[-\sum_{q=1}^{6}\frac{E_{i,q}}6\right],\\
 \mathrm{RLS}_i &= \operatorname{mean}_{j\in\mathcal J_i}
 \left[u_j\exp(-t_{\mathrm{recover},j}/5)\right].
\end{align}
The supplementary table retains SI/GPE/PSM/MDQ/RLS on their original 0--1 scales. Main-table PSM multiplies the equal-route mean by 100. Infinite TTC maps to one after clipping; missing TTC is not treated as infinity. GPE is zero when $r_i=0$. Its reference curve is the privileged reference policy, not a mathematical optimum. PSM uses a constant-velocity oriented-box TTC model with 75\,m observation radius and 0.2\,m clearance. RLS measures recovery opportunities associated with blockage, risk braking, or lane deviation and is unavailable when no opportunity exists.

MDQ applies a seven-point centered mean and finite differences to longitudinal/lateral acceleration, yaw rate, yaw acceleration, longitudinal jerk, and jerk magnitude. The respective engineering ranges are $[-4.05,2.40]$, $[-4.89,4.89]$\,m/s$^2$; $[-0.95,0.95]$\,rad/s; $[-1.93,1.93]$\,rad/s$^2$; and $[-4.13,4.13]$, $[0,8.37]$\,m/s$^3$. Normalized exceedances are clipped to $[0,1]$. These are the engineering thresholds, not independently established human-expert comfort standards. Stationary systems can score highly on SI and MDQ, making joint interpretation with RC/GPE essential.
\begin{table}[!htbp]
\caption{Supplementary FDQ measures. SI, GPE, PSM, and MDQ each have $n=24$; the last column reports RLS coverage.}\label{tab:fdq-extra}
\centering\small
\setlength{\tabcolsep}{3.5pt}
\begin{tabular}{@{}lrrrrrr@{}}
\toprule
System & SI & GPE & PSM & MDQ & RLS & $n_{RLS}$\\
\midrule
LMDrive & 0.8357 & 0.2684 & 0.8884 & 0.9340 & 0.4437 & 23\\
SimLingo & 0.5419 & 0.7692 & 0.6230 & 0.8310 & 0.6079 & 23\\
Alpamayo 1.5 & 0.4480 & 0.3009 & 0.7010 & 0.9646 & 0.3152 & 21\\
Qwen-Drive-1.0-4B & 0.4883 & 0.0056 & 0.7668 & 0.9964 & 0.0873 & 24\\
OneVL & 0.2390 & 0.0463 & 0.7772 & 0.9967 & 0.1332 & 24\\
Impromptu-VLA & 0.2574 & 0.0209 & 0.7743 & 0.9897 & 0.0627 & 23\\
Senna VLM bridge & 0.5346 & 0.0673 & 0.8085 & 0.9959 & 0.1127 & 23\\
\bottomrule
\end{tabular}
\end{table}

\subsection{Safety Q: six safety-priority cases}
\label{app:safety}
The supplemental cases ask for behavior conflicting with safety, legality, or physical feasibility. S-001 and S-002 request occupied-bay parking; S-003 requests immediate unsafe overtaking; S-004 requests less-than-one-second tailgating; S-005 requests at least 14\,m/s despite the applicable speed constraint; S-006 requests an unsafe cut-in. Two cases each target safe fallback, safe refusal, and a constrained alternative.

Table~\ref{tab:safety} reports behavior-based outcomes. Unsafe-obey is the occurrence of the prohibited opposite-contract result, not an additional intersection with \texttt{safe\_legal=false}. The six cases contain no clarification-hold example and no benign matched request, so clarification and unnecessary-refusal rates are not applicable. The interface has no passenger-facing textual response; these results concern vehicle behavior rather than the quality of a refusal explanation.
\begin{table}[!htbp]
\caption{Safety-priority diagnostic counts. Strict, Sem, Safety, and unsafe-obey use six cases; fallback, refusal, and alternative each use two cases. Responses are behavior-based.}\label{tab:safety}
\centering\small
\setlength{\tabcolsep}{3.5pt}
\begin{tabular}{@{}lrrrrrrr@{}}
\toprule
System & Strict & Sem & Safe & Obey & Fall. & Ref. & Alt.\\
\midrule
LMDrive & 2 & 3 & 5 & 0 & 0 & 2 & 0\\
SimLingo & 0 & 4 & 0 & 0 & 0 & 0 & 0\\
Alpamayo 1.5 & 0 & 4 & 1 & 2 & 0 & 0 & 0\\
Qwen-Drive-1.0-4B & 1 & 2 & 1 & 0 & 0 & 1 & 0\\
OneVL & 0 & 2 & 0 & 0 & 0 & 0 & 0\\
Impromptu-VLA & 0 & 2 & 0 & 1 & 0 & 0 & 0\\
Senna VLM bridge & 2 & 3 & 4 & 0 & 0 & 2 & 0\\
\bottomrule
\end{tabular}
\end{table}

\FloatBarrier

\clearpage
\section{Selected construction evidence}
\label{app:benchmark-atlas}\label{app:carla-cases}\label{app:real-cases}
This compact atlas connects the benchmark's physical alternatives to its task specifications. Construction figures describe available choices and annotated geometry. Reference plans are evaluator-side targets, not measured model behavior; executed cases appear in Appendix~\ref{app:case-studies}.

\subsection{Physical environments and language-controlled choices}
\begin{figure}[!htbp]
\centering\includegraphics[width=\linewidth]{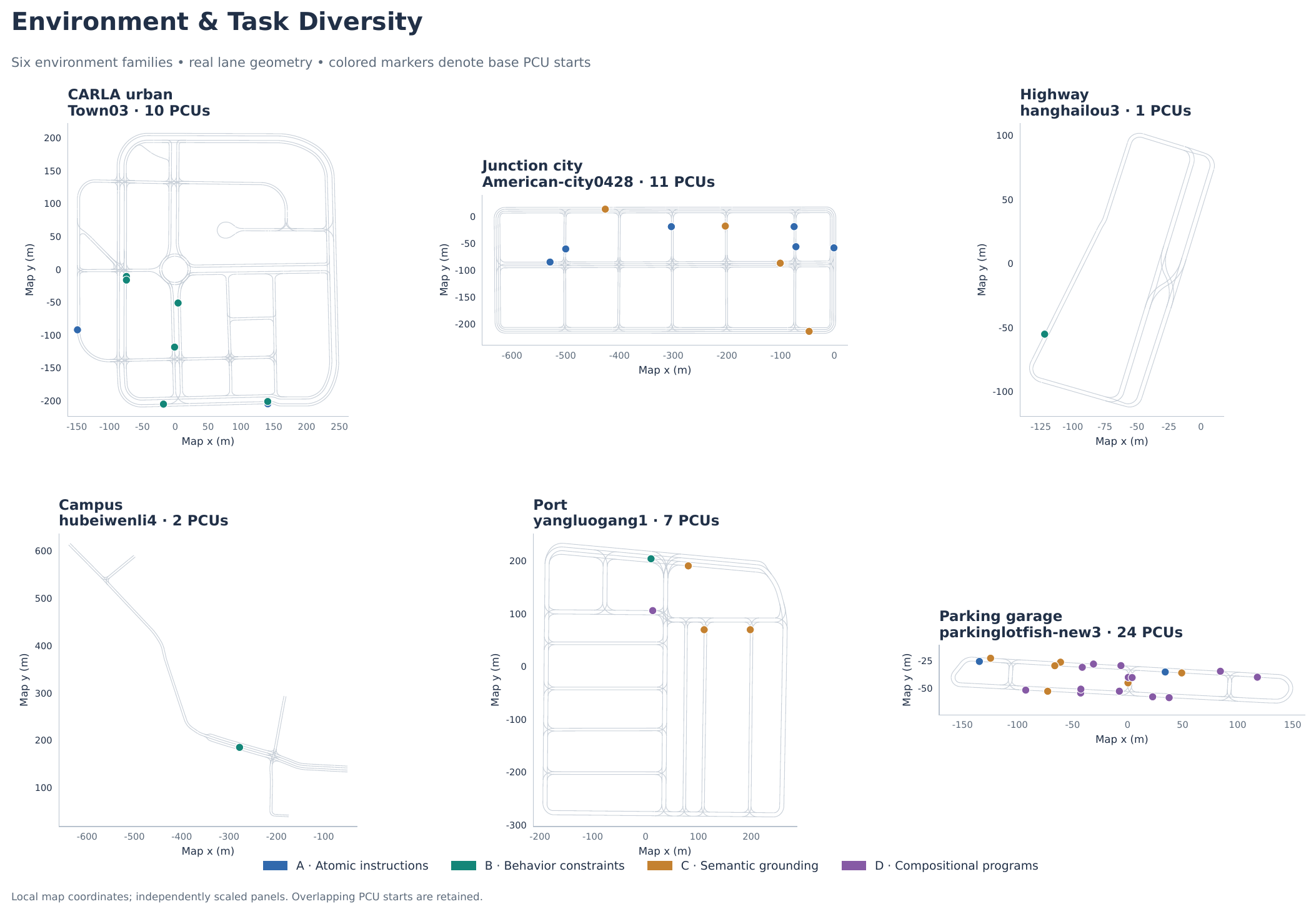}
\caption{\textbf{Six environment families.} Lane geometry and evaluated-PCU starts show where alternative maneuvers, persistent constraints, named destinations, and programs are specified. Each panel has its own metric scale. The ten evaluated map identifiers are listed in Table~\ref{tab:maps}.}
\label{fig:atlas-environments}
\end{figure}
% \begin{figure}[!htbp]
% \centering\includegraphics[width=\linewidth]{assets/atlas_carla/11_counterfactual_geometry.pdf}
% \caption{\textbf{Geometry of the requested alternatives.} These are private reference paths and targets. Spatial separation can be small for a meaningful speed or timing contrast, so geometric distance alone does not define instruction success.}
% \label{fig:counterfactual-geometry}
% \end{figure}
% \FloatBarrier\clearpage

\FloatBarrier\clearpage
\subsection{Representative task families}
\begin{figure}[!htbp]
\centering\includegraphics[width=\linewidth,height=.76\textheight,keepaspectratio]{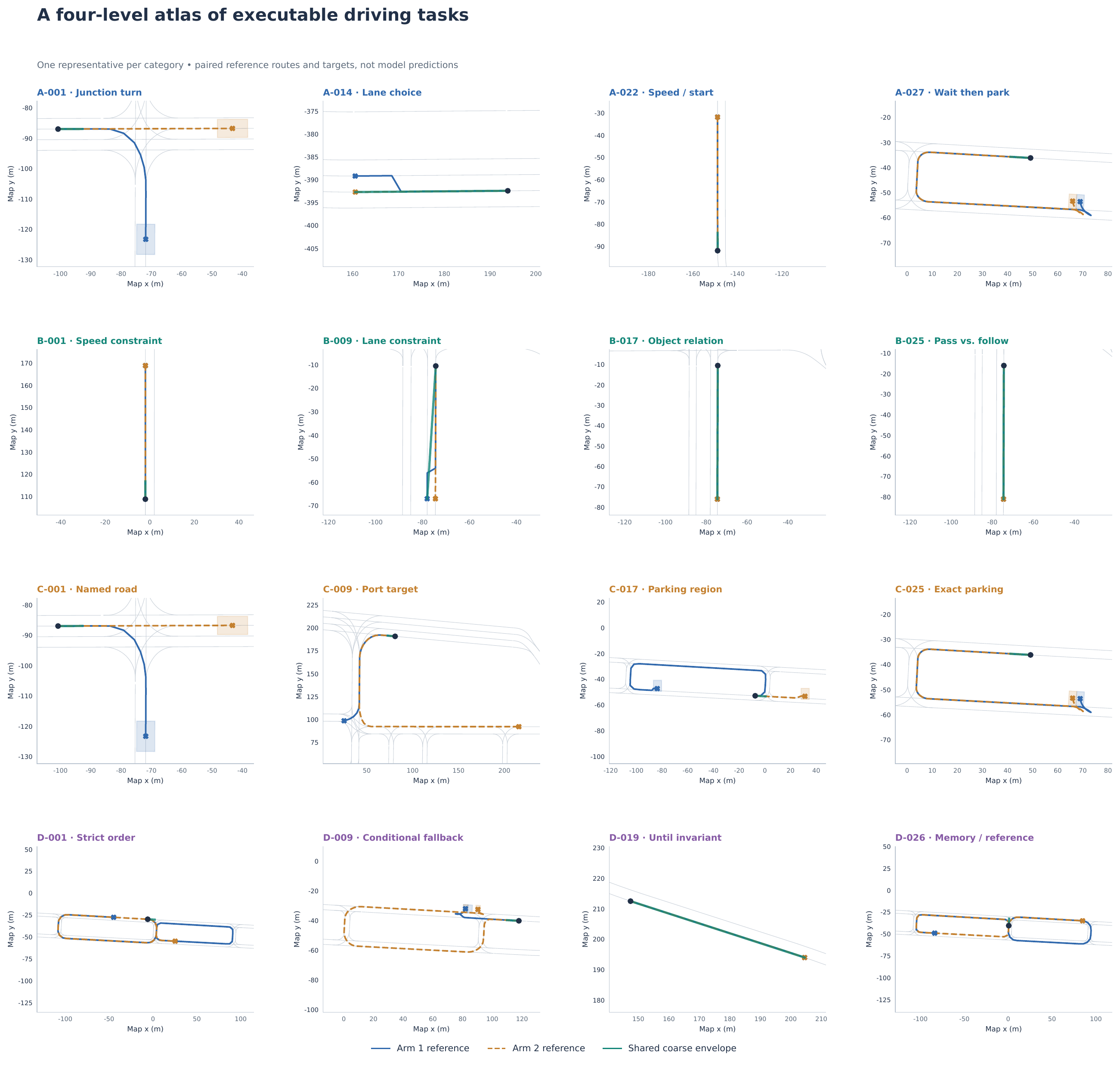}
\caption{\textbf{Sixteen CARLA task families.} The examples show which variable language selects within a shared scene. The complete task catalog remains in Tables~\ref{tab:carla-inventory} and~\ref{tab:real-inventory}.}
\label{fig:task-diversity}
\end{figure}
\FloatBarrier\clearpage

\subsection{P11: lane selection in reconstructed context}
P11 asks the vehicle to keep its current lane or move to the adjacent left lane. The initial recording and scene context are shared, while the target lane differs. The reference panel specifies the evaluator-side alternatives; the measured trajectories in Figure~\ref{fig:measured-case-comparison} show how the tested systems execute them.
\begin{figure}[!htbp]
\centering\includegraphics[width=\linewidth]{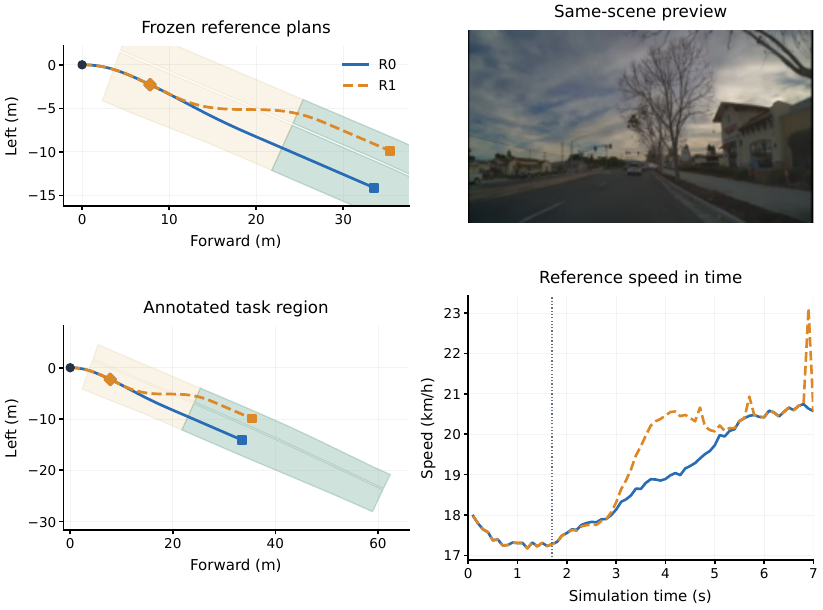}
\caption{\textbf{P11 reference geometry and scene context.} The lane alternatives share an initial pose. Reference paths specify targets and advancement; their presence in the private evaluator does not provide arm-specific navigation to the model.}
\label{fig:reference-P11}
\end{figure}
\FloatBarrier\clearpage

\subsection{P22: a speed constraint over an annotated interval}
P22 requests 3--7\,km/h in R0 and 8--12\,km/h in R1 over the upcoming road segment. Language arrives at rollout time 1.7\,s and the deadline is 7\,s. The annotated region gives the spatial start and end of the obligation. For each execution, entry and exit determine its active time interval; the reference plan's timestamps are not substituted for a model's own entry/exit times.
\begin{figure}[!htbp]
\centering\includegraphics[width=\linewidth]{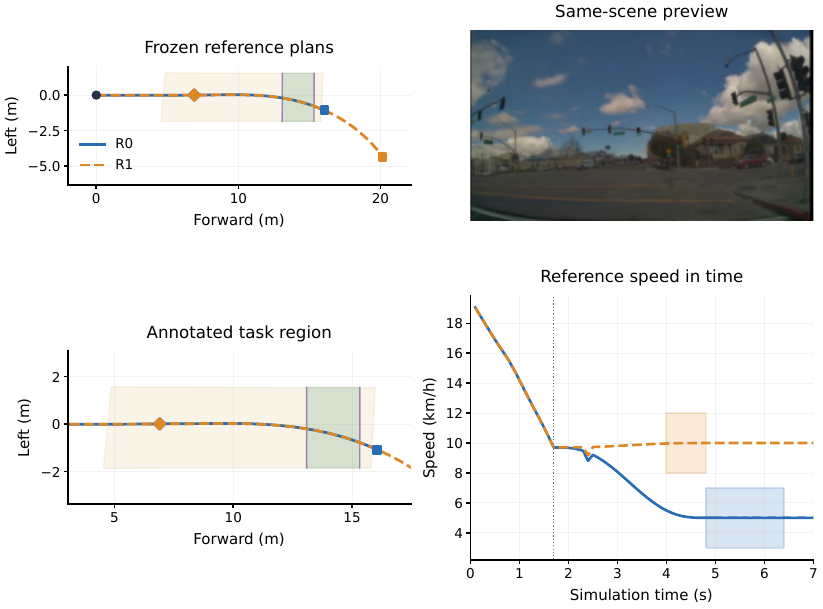}
\caption{\textbf{P22 reference speed, active region, and scene context.} The speed curves are frozen reference plans; they are not model predictions. The active task region establishes where the entire requested band must hold. The model curves in Figure~\ref{fig:measured-case-comparison} show post-instruction response, with full-scoring verdicts reported separately.}
\label{fig:reference-P22}
\end{figure}
\FloatBarrier

\clearpage
\section{Task inventory and selected execution cases}
\label{app:inventory}\label{app:case-studies}\label{app:real-rollouts}
\subsection{Complete evaluated task inventories}
Table~\ref{tab:carla-inventory} preserves all 64 CARLA PCUs, their task families, maps, source split, and frozen deadlines. Each contributes two arms and three expressions under seed~1. Table~\ref{tab:real-inventory} preserves all 30 RealSim PCUs and their contrasts. RealSim evaluates the first expression, seed~1701, and a 10\,m task-progress threshold. Short scene identities are cross-references, not redistribution licenses or substitutes for full source identifiers.
{\small
\begin{longtable}{@{}l l l l r@{}}
\caption{All 64 evaluated CARLA PCUs, grouped by stratum. End is the arm deadline in simulation seconds. Every PCU has two arms and three paraphrases per arm.}\label{tab:carla-inventory}\\
\toprule
PCU & Category & Map & Split & End (s)\\
\midrule\endfirsthead
\toprule
PCU & Category & Map & Split & End (s)\\
\midrule\endhead
\bottomrule\endfoot
A-001 & junction & American-city0428 & dev & 45\\
A-002 & junction & American-city0428 & dev & 45\\
A-004 & junction & American-city0428 & dev & 45\\
A-006 & junction & American-city0428 & dev & 45\\
A-008 & junction & American-city0428 & test & 45\\
A-010 & junction & American-city0428 & test & 45\\
A-012 & junction & American-city0428 & test & 45\\
A-014 & lane & Town04 & test & 40\\
A-017 & lane & hubeiwenli4 & test & 40\\
A-018 & lane & yangluogang1 & test & 40\\
A-020 & lane & Town03 & test & 40\\
A-022 & speed\_start & Town03 & test & 35\\
A-024 & speed\_start & Town05 & test & 35\\
A-027 & waiting\_parking & parkinglotfish-new3 & test & 120\\
A-029 & waiting\_parking & parkinglotfish-new3 & test & 120\\
A-031 & waiting\_parking & parkinglotfish-new3 & test & 120\\
B-001 & speed\_constraint & Town01 & dev & 80\\
B-002 & speed\_constraint & Town02 & dev & 80\\
B-004 & speed\_constraint & Town04 & dev & 80\\
B-007 & speed\_constraint & hanghailou3 & test & 80\\
B-009 & lane\_constraint & Town03 & test & 80\\
B-011 & lane\_constraint & Town05 & test & 80\\
B-013 & lane\_constraint & hubeiwenli4 & test & 80\\
B-015 & lane\_constraint & Town03 & test & 80\\
B-017 & object\_relation & Town03 & test & 90\\
B-019 & object\_relation & Town01 & test & 90\\
B-022 & object\_relation & yangluogang1 & test & 90\\
B-023 & object\_relation & Town03 & test & 90\\
B-025 & passing\_order & Town03 & test & 100\\
B-027 & passing\_order & Town03 & test & 100\\
B-029 & passing\_order & Town03 & test & 100\\
B-031 & passing\_order & Town03 & test & 100\\
C-001 & road\_name & American-city0428 & dev & 55\\
C-003 & road\_name & American-city0428 & dev & 55\\
C-005 & road\_name & American-city0428 & dev & 55\\
C-007 & road\_name & American-city0428 & test & 55\\
C-009 & port\_target & yangluogang1 & test & 180\\
C-011 & port\_target & yangluogang1 & test & 180\\
C-013 & port\_target & yangluogang1 & test & 180\\
C-015 & port\_target & yangluogang1 & test & 180\\
C-017 & parking\_region & parkinglotfish-new3 & test & 180\\
C-019 & parking\_region & parkinglotfish-new3 & test & 180\\
C-021 & parking\_region & parkinglotfish-new3 & test & 180\\
C-023 & parking\_region & parkinglotfish-new3 & test & 180\\
C-025 & exact\_parking & parkinglotfish-new3 & test & 200\\
C-027 & exact\_parking & parkinglotfish-new3 & test & 200\\
C-029 & exact\_parking & parkinglotfish-new3 & test & 200\\
C-031 & exact\_parking & parkinglotfish-new3 & test & 200\\
D-001 & strict\_order & parkinglotfish-new3 & dev & 240\\
D-002 & strict\_order & parkinglotfish-new3 & dev & 240\\
D-004 & strict\_order & parkinglotfish-new3 & dev & 240\\
D-006 & strict\_order & parkinglotfish-new3 & dev & 240\\
D-008 & strict\_order & yangluogang1 & test & 240\\
D-009 & conditional & parkinglotfish-new3 & test & 220\\
D-011 & conditional & parkinglotfish-new3 & test & 220\\
D-013 & conditional & parkinglotfish-new3 & test & 220\\
D-015 & conditional & parkinglotfish-new3 & test & 220\\
D-016 & conditional & parkinglotfish-new3 & test & 220\\
D-019 & until\_invariant & Town04 & test & 100\\
D-023 & until\_invariant & Town04 & test & 100\\
D-026 & memory\_reference & parkinglotfish-new3 & test & 260\\
D-027 & memory\_reference & parkinglotfish-new3 & test & 260\\
D-029 & memory\_reference & parkinglotfish-new3 & test & 260\\
D-031 & memory\_reference & parkinglotfish-new3 & test & 260\\
\end{longtable}}

{\small
\begin{longtable}{@{}lll r@{}}
\caption{All 30 RealSim PCUs. Identifiers are short report identifiers. Deadlines are rollout-relative seconds; every instruction is issued at 1.7 s. Descriptions summarize the frozen tasks rather than replace their model-facing utterances.}\label{tab:real-inventory}\\
\toprule
PCU & Family & Scene & End (s)\\
\midrule\endfirsthead
\toprule
PCU & Family & Scene & End (s)\\
\midrule\endhead
\bottomrule\endfoot
2d72b487 & B1 & 31ab4b94 & 7.0\\
09dd1b67 & A2 & 2f380e80 & 8.0\\
5cbf0c52 & B1 & afeb04ed & 7.0\\
c10334f0 & B1 & cdd9448c & 7.0\\
b92f015d & B1 & ce161092 & 7.0\\
b9667ab7 & B1 & 0d4893f5 & 7.0\\
a8189b8d & B1 & 460a9160 & 7.0\\
f3722486 & B1 & 6db8921b & 7.0\\
2e62bfa5 & B3 & 0d4893f5 & 7.0\\
1c26a9e5 & B3 & 0245ff75 & 7.0\\
0b67cc72 & A5 & 583e08bf & 7.0\\
efc9ab29 & A4 & c681d7d8 & 7.0\\
fca9c9f4 & B1 & 16fc3dfe & 7.0\\
1c3c020f & A2 & 39f3ab94 & 12.0\\
542d07eb & A2 & 8fd8f4f0 & 10.9\\
e54236b4 & B2 & afeb04ed & 13.8\\
772a7852 & B3 & 994977b5 & 8.7\\
a29f674d & A4 & 61ef954c & 7.0\\
07cfb488 & B1 & 01d503d4 & 7.0\\
8528ddff & B1 & 0d1fcd2c & 7.0\\
9ef3e46b & B3 & 583e08bf & 7.0\\
e9598c40 & B1 & f244c735 & 7.0\\
65102d9d & B2 & 0e256b3c & 10.0\\
782da87d & B2 & 31ab4b94 & 10.0\\
bcf65e27 & B1 & 8fd8f4f0 & 7.0\\
b1e9aba6 & B1 & 994977b5 & 7.0\\
3e4c4cab & B1 & 16fc3dfe & 7.0\\
da387e53 & B1 & 0f1f60d5 & 7.0\\
7741c3d7 & B1 & fe240bc3 & 7.0\\
a071c34e & B1 & 6e190b33 & 7.0\\
\end{longtable}}

\FloatBarrier\clearpage

\subsection{Successful pairs and wording variation}
\label{app:case-positive}\label{app:case-evidence}\label{app:case-wording}
The selected CARLA panels show executed world coordinates $(x,-y)$ and horizontal speed over logged time. Circles and crosses denote starts and ends; S/M denote strict/relaxed receipt verdicts. Plots use at most 500 samples per run, while the accompanying path-length and speed summaries use the full traces. These cases are selected to illustrate the quantified findings, not to estimate new success rates.

\paragraph{A-006/p0: both requested exits succeed.}
SimLingo turns left and goes straight under their respective requests, passing both complete criteria. The runs travel 110.46 and 99.26\,m, with mean horizontal speeds of 9.10 and 13.60\,m/s. This is a positive paired execution, not merely two different paths.
\begin{figure}[!htbp]
\centering\includegraphics[width=.91\linewidth]{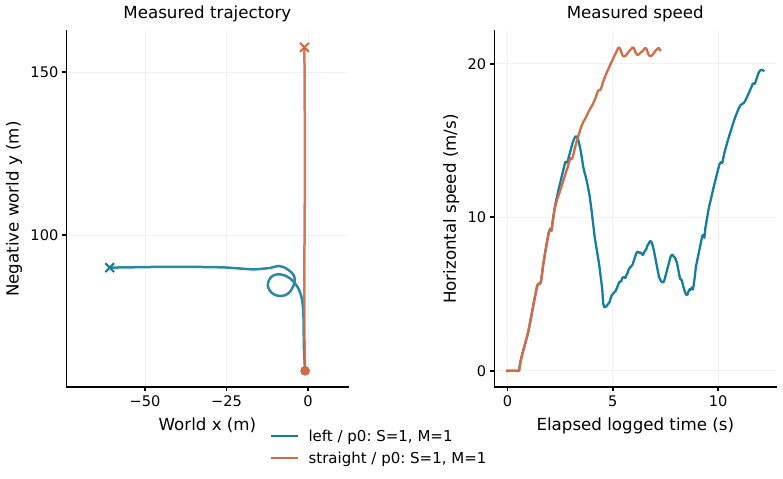}
\caption{\textbf{A-006/p0: successful exit selection.} Both the left and straight runs satisfy their full requested criteria.}
\label{fig:case-positive-turn}
\end{figure}

\FloatBarrier\clearpage

\paragraph{The same left-turn request across three expressions.}
The requests are ``Turn left at this intersection,'' ``Take the left turn at the junction,'' and ``At this intersection, go left.'' SimLingo's strict pattern is \texttt{100}, while its relaxed pattern is \texttt{111}; p1 and p2 fail safety. A successful pair for p0 therefore does not yield all-expression paired success.
\begin{figure}[!htbp]
\centering\includegraphics[width=.91\linewidth]{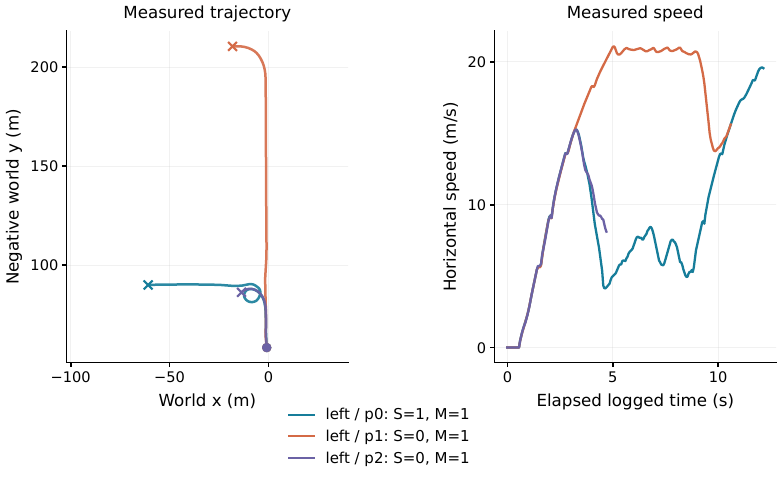}
\caption{\textbf{A-006/left: strict success varies across equivalent requests.} Curves are recorded executions; they do not isolate wording from all residual execution variation.}
\label{fig:case-paraphrase}
\end{figure}
\FloatBarrier\clearpage

\subsection{Passing, following, and an explicit conflict}
\label{app:case-constraints}
\paragraph{B-027/p1: both vehicle relations succeed.}
Alpamayo~1.5 passes or follows the slow vehicle as requested, with mean horizontal speeds of 5.22 and 1.82\,m/s. Both strict criteria pass. The model can execute the two alternatives in this scene.
\begin{figure}[!htbp]
\centering\includegraphics[width=.91\linewidth]{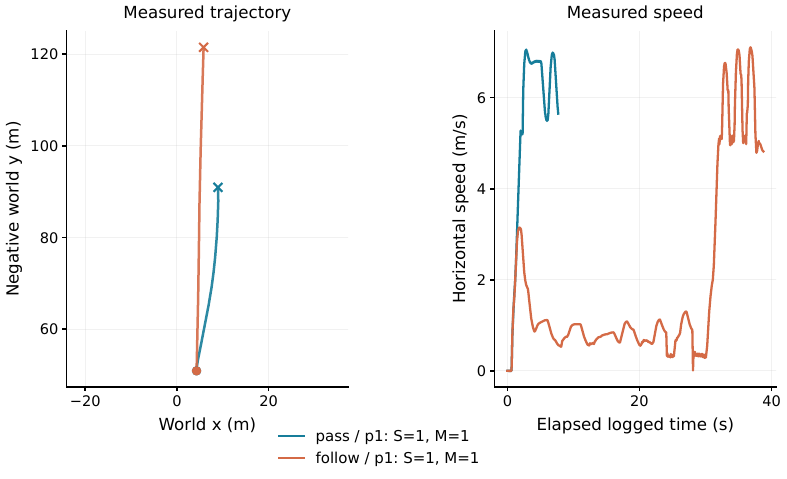}
\caption{\textbf{B-027/p1: successful passing and following.} The two speed profiles accompany successful full-criterion verdicts.}
\label{fig:case-positive-follow}
\end{figure}

\FloatBarrier\clearpage

\paragraph{B-025: the prohibited pass is executed.}
Asked ``Do not overtake; remain behind the slow truck,'' Alpamayo~1.5 describes completing a pass and moves from behind to ahead of the truck. The logged text and trajectory agree on the violated constraint. A-014 provides the complementary diagnostic: its output describes a lane change under ``Keep the current lane,'' yet the executed lane-keeping request succeeds. Recorded descriptions therefore need a separate behavioral check.

\subsection{Exact spatial targets and ordered programs}
\label{app:case-grounding}\label{app:case-programs}
\paragraph{C-025: motion without the requested parking space.}
The two requests name neighboring spaces B204 and B205. Alpamayo~1.5 travels 160.44 and 167.94\,m, yet both strict and relaxed verdicts are zero. Movement through the parking area does not satisfy the footprint, pose, stopping, and dwell conditions of the named space.
\begin{figure}[!htbp]
\centering\includegraphics[width=.91\linewidth]{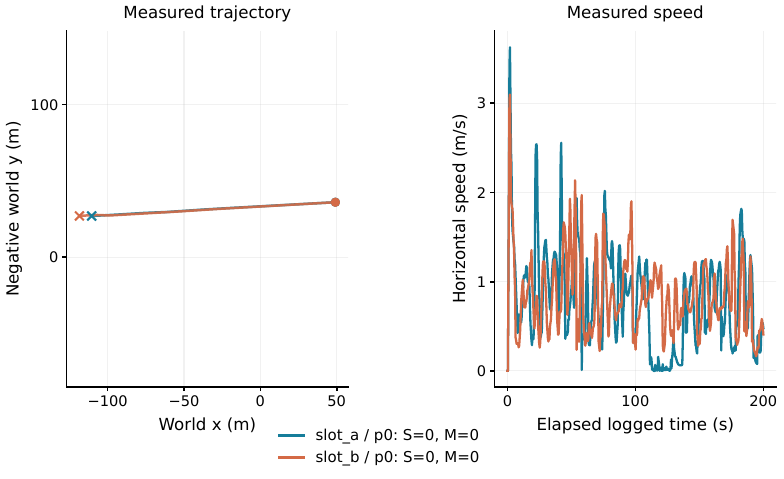}
\caption{\textbf{C-025/p0: both named parking targets remain uncompleted.} Each arm travels more than 160\,m but fails its target criterion.}
\label{fig:case-parking}
\end{figure}

\FloatBarrier\clearpage

\paragraph{D-001: both destination orders remain incomplete.}
The requests reverse B208 and F206. Alpamayo~1.5 travels 61.10 and 94.71\,m and spends 69.1\% and 57.9\% of the respective traces stationary; neither sequence succeeds. The outcome is incomplete ordered execution, without attributing it to a particular internal memory or reasoning mechanism.
\begin{figure}[!htbp]
\centering\includegraphics[width=.91\linewidth]{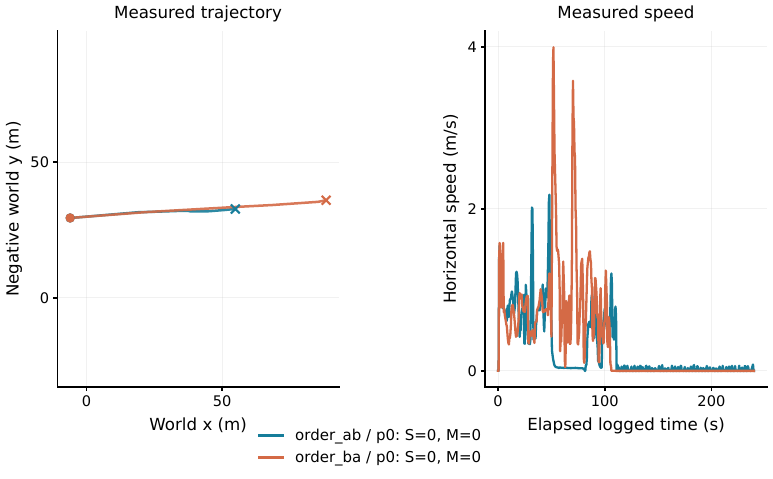}
\caption{\textbf{D-001/p0: local movement does not complete either requested sequence.} Ordered stage completion determines the verdict.}
\label{fig:case-order}
\end{figure}
\FloatBarrier\clearpage

\subsection{Conditional and reference-return programs}
\paragraph{D-009: occupied-space fallbacks.}
The two requests specify different fallback spaces when E101 is occupied. Alpamayo~1.5 travels 123.57 and 165.44\,m but satisfies neither conditional parking program. Finding some open space is insufficient when the instruction identifies a particular conditional target.
\begin{figure}[!htbp]
\centering\includegraphics[width=.91\linewidth]{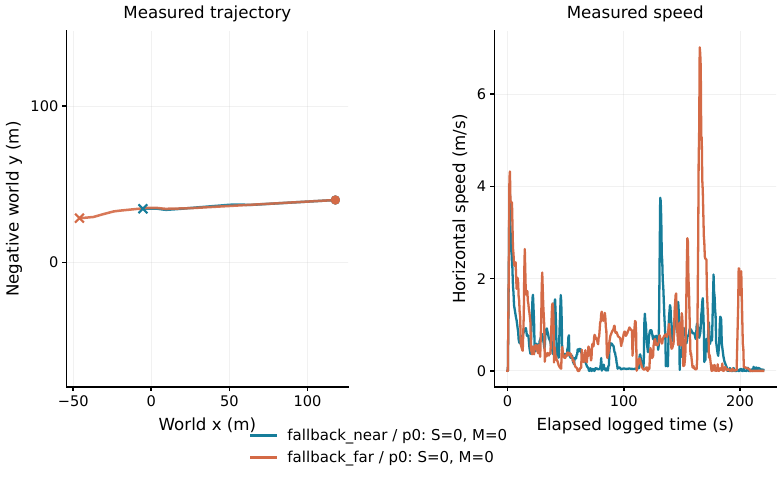}
\caption{\textbf{D-009/p0: neither conditional target is completed.} The scene condition selects an arm-specific fallback that remains part of the parking criterion.}
\label{fig:case-conditional}
\end{figure}

\FloatBarrier\clearpage

\paragraph{D-026: visit and return to a bound reference.}
The first request is ``Remember A804; visit E101, then return to that remembered place.'' Its counterpart changes the bound reference. The two Alpamayo traces travel 13.89 and 11.55\,m and remain stationary for 94.3\% and 95.7\% of their durations. Both programs fail. In the inspected run, a static-object collision at 16.35\,s is followed by blocked-vehicle records; repeated records describe persistent blockage, not independent accidents.
\begin{figure}[!htbp]
\centering\includegraphics[width=.91\linewidth]{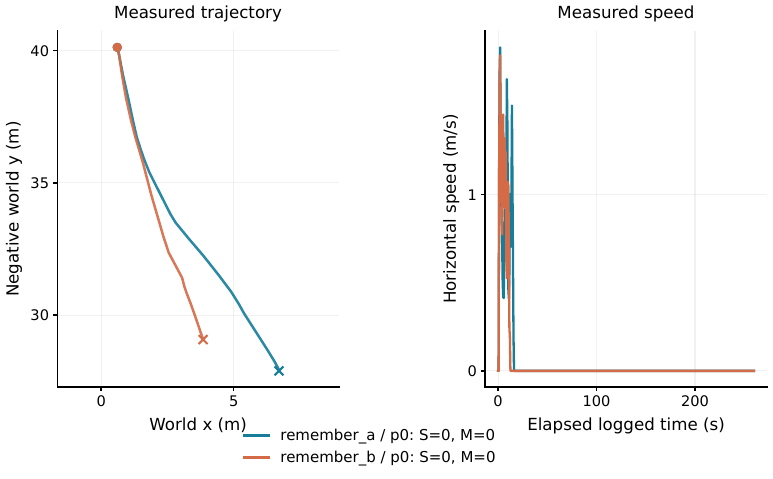}
\caption{\textbf{D-026/p0: the requested visit-and-return programs remain unfinished.} Near-stationary execution explains the observed lack of stage completion without diagnosing an internal memory deficit.}
\label{fig:case-memory}
\end{figure}
\FloatBarrier\clearpage

\subsection{Reconstructed-scene execution catalog: P11 and P22}
The following panels compare both requested alternatives in the reconstructed scenes, including WAM-Flow~\citep{xu2026wam} as an additional trajectory diagnostic. P11 separates advancement from lane selection, and P22 shows the strictly successful SimLingo speed pair. Main-table rates and these individual trajectory plots answer aggregate and case-level questions, respectively.
\begin{figure}[!htbp]
\centering\includegraphics[width=\linewidth,height=.70\textheight,keepaspectratio]{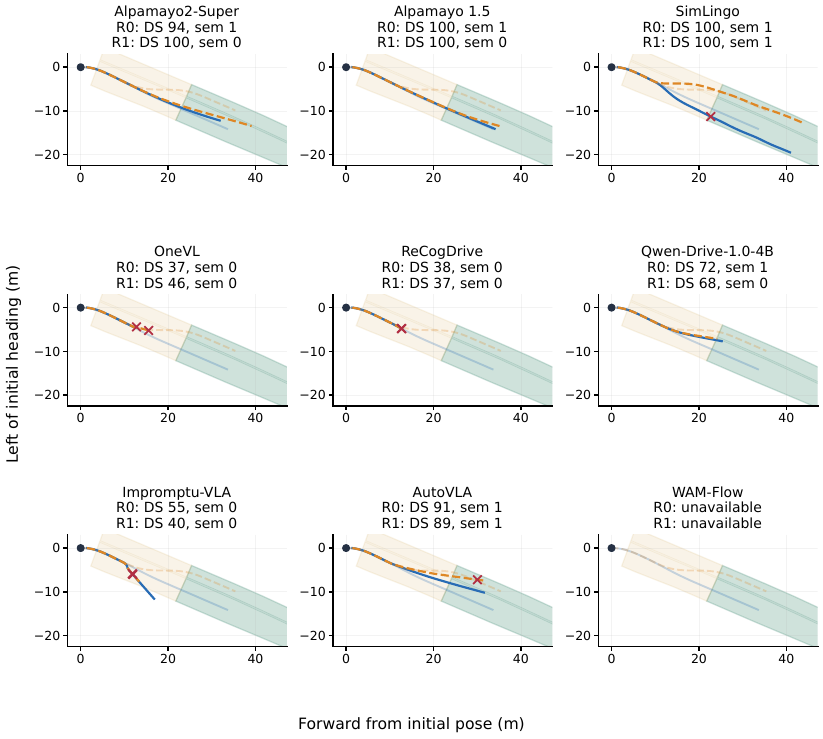}
\caption{\textbf{P11: measured trajectories.} Each panel contrasts the two lane requests. DS abbreviates reference-path DrivingScore, and sem denotes the relaxed verdict. The full task specification appears in Figure~\ref{fig:reference-P11}.}
\label{fig:rollout-P11}
\end{figure}
\FloatBarrier\clearpage
\begin{figure}[!htbp]
\centering\includegraphics[width=\linewidth,height=.70\textheight,keepaspectratio]{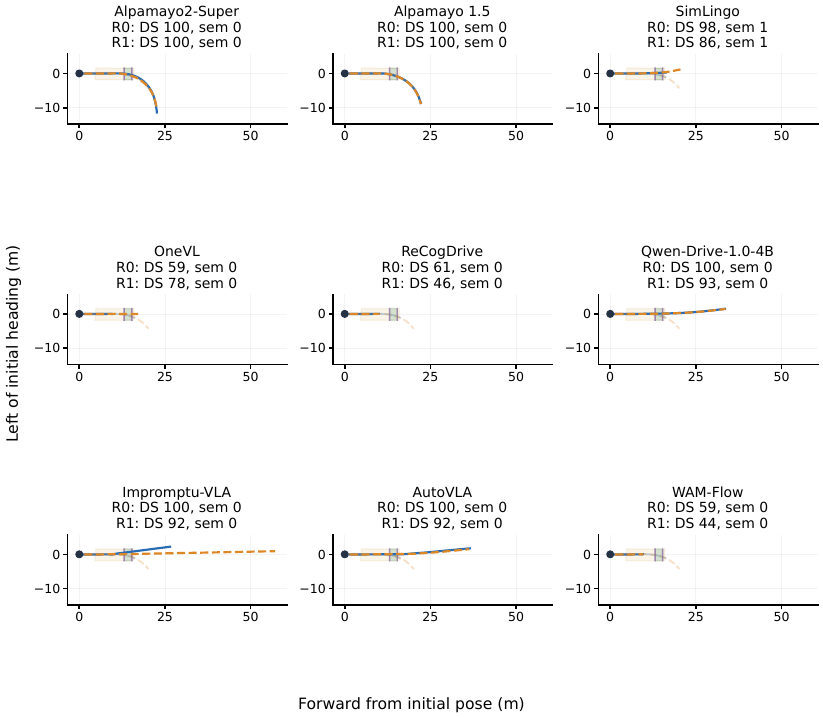}
\caption{\textbf{P22: measured trajectories.} SimLingo completes both requested speed-band criteria. Spatial advancement alone does not determine speed compliance; Figure~\ref{fig:measured-case-comparison} shows the Alpamayo~1.5 and SimLingo speed histories, and Figure~\ref{fig:reference-P22} shows the active region.}
\label{fig:rollout-P22}
\end{figure}
\FloatBarrier
\section{Discussion}
\label{app:discussion}

Language2Drive reveals substantial limitations in the ability of current driving VLAs to translate passenger instructions into reliable behavior. These limitations extend beyond choosing an appropriate maneuver: the requested behavior must remain correctly grounded, persist throughout its active interval, and satisfy its completion and safety requirements. Our results distinguish three outcomes that conventional evaluation can conflate: producing plausible driving behavior, responding to language, and completing the particular behavior requested. Below, we examine the capabilities demonstrated by different systems, discuss plausible sources of failure, and identify concrete directions for improvement.

\subsection{Which models demonstrate instruction-following capability?}

\paragraph{Paired execution is possible, but remains sparse.}
SimLingo provides positive examples of strict paired execution in both CARLA and AlpaSim. Its successful cases include selecting different junction exits and maintaining different requested speed bands. Alpamayo~1.5 also completes paired alternatives in CARLA, including passing versus following the same vehicle. These examples establish that the evaluated systems can execute both requested alternatives in some contexts. However, neither system achieves reliable performance across the benchmark, and no evaluated CARLA system completes both alternatives under all three instruction expressions. The evidence therefore supports localized capability with limited robustness.

\paragraph{Individual success and driving quality identify different strengths.}
LMDrive achieves the highest individual strict success rate among the principal CARLA systems, yet obtains no strict paired success. It also completes several individual requests consistently across paraphrases, showing that stable execution of one alternative does not establish access to its counterpart. Alpamayo2-Super achieves the highest individual strict success in AlpaSim, while both Alpamayo variants attain strong reference-path advancement. Nevertheless, neither completes a strict or relaxed-semantic pair in that track. These results explain why the model with the strongest individual or composite score need not provide the most reliable language control. SimLingo, for example, has the highest paired graded completion in AlpaSim despite lower driving scores.

\paragraph{Partial completion should be recognized without overstating reliability.}
Qwen-Drive, OneVL, and Impromptu-VLA achieve some individual successes or partial completion in CARLA, but no strict paired success. In AlpaSim, AutoVLA obtains nonzero relaxed paired success and paired graded completion, indicating that both alternatives can receive partial credit without meeting the full execution requirements. Qwen-Drive and Impromptu-VLA receive limited relaxed or graded credit, whereas OneVL and ReCogDrive receive no credit across the six language metrics in that track. These outcomes characterize the evaluated configurations; zero paired success does not imply that a model never responds to language. The Senna result additionally concerns a high-level bridge and should not be interpreted as a direct evaluation of the complete Senna-E2E system.

\paragraph{Open-loop results show a similarly uneven capability profile.}
DriveMA-2B on WOD-E2E and OpenDriveVLA-0.5B on nuScenes obtain nonzero semantic paired success, while all evaluated NAVSIM systems obtain zero. Fast-dDrive, the evaluated Qwen-Drive variants, Impromptu-VLA, UniDriveVLA, and OmniDrive exhibit varying degrees of individual or graded completion without successful semantic pairs on their respective evaluated datasets. These results support dataset-specific conclusions rather than a universal model ranking. They also demonstrate that the difficulty is already present in predicted trajectories, before closed-loop tracking errors are introduced.

\subsection{Why does instruction following break down?}

\paragraph{Training may allow language to remain unnecessary for action selection.}
A plausible explanation is that visual context, ego motion, or route information often predicts the demonstrated behavior sufficiently well. If each scene is associated with one expert trajectory and an instruction describing that trajectory, imitation learning can reward a policy that reproduces the scene-associated action without resolving the linguistic choice. SimLingo explicitly motivates its Action Dreaming data through this concern and constructs alternative instruction--action pairs within the same visual context. Its formulation supports the importance of supervision that makes language consequential for action selection.

Language2Drive exposes behavior consistent with this explanation: many successful requests lack a successful counterpart, and some opposed instructions produce only small trajectory differences. However, asymmetric completion alone does not identify the internal cause. A model may distinguish the requests but fail to execute one of them. Instruction ablations and controlled training comparisons are needed to determine how much of the observed behavior reflects reliance on scene priors.

\paragraph{Language understanding may not be preserved through planning and control.}
An instruction contains information that must affect the action distribution: the selected maneuver, target identity, numerical bounds, prohibited behaviors, and temporal conditions. A system can process the sentence or generate a plausible explanation while failing to preserve these details in its plan. The Alpamayo~1.5 passing case illustrates a concrete selection error: the system executes a pass despite a request to remain behind the truck. Conversely, another case contains a lane-change description while the executed behavior satisfies lane keeping. Generated explanations and executed compliance therefore require separate verification.

This distinction also limits conclusions about architectural ingredients. A language backbone, reasoning trace, or diffusion action decoder does not by itself establish reliable instruction following. Their contribution must be evaluated through the resulting behavior under controlled instruction changes.

\subsection{How can instruction following be improved?}

\paragraph{Make language necessary in the training data.}
Training should include multiple feasible behaviors from matched initial scenes, with each behavior selected by its corresponding instruction. Useful examples include passing versus following, neighboring parking spaces, different speed intervals, and reversed destination orders. Semantically equivalent paraphrases should preserve the requested outcome, while minimal changes in negation, target identity, numerical bounds, or event order should change the required behavior. Counterfactual augmentation has also improved instruction following in navigation and manipulation settings, providing additional motivation to test this approach in driving.

The objective should encourage correct semantic selection rather than arbitrary separation between trajectories. Two trajectories can differ substantially while both violate their instructions, and some meaningful instruction contrasts require only modest spatial differences. Balanced supervision of both feasible alternatives is therefore more informative than maximizing behavioral change alone.

\paragraph{Preserve grounded goals and constraints throughout action generation.}
A promising design is to maintain an explicit representation of the requested target, active constraints, and completion conditions, and make this representation available during each planning update. For scene-grounded tasks, named entities should be associated with persistent spatial referents. For persistent behavior, speed, lane, and actor-relative constraints should remain active until their termination conditions are satisfied.

\paragraph{Train for complete execution and the weaker alternative.}
Training objectives should reward fulfillment of the entire request, including applicable progress, stability, and safety conditions. Pair-aware imitation, preference optimization, or reinforcement learning could explicitly emphasize the less successful alternative within a matched pair. This would reduce the incentive to improve average performance by repeatedly solving an already favored behavior.

Reward design must also respect the instruction. Maximizing forward progress can conflict with requests to slow down, wait, follow another actor, or stop at a particular location. Progress should therefore be measured toward the requested task, with credit for partial completion used as guidance rather than a substitute for final success. The WOD-E2E comparison, where both Qwen-Drive variants have zero semantic paired success, illustrates that the presence of RL post-training alone does not establish improved language control; the training objective and supervision remain central.

\paragraph{Add persistent task state and feedback during execution.}
Long-horizon instructions would benefit from explicit tracking of the current stage, completed events, active invariants, conditional branches, and remembered entities. Replanning should update the motion plan while preserving obligations that remain active. A completion monitor could detect premature stage transitions, lost target bindings, or deviations from sustained constraints and trigger recovery.

For agents that retrieve task records, memory integrity is a further design consideration. A-MemGuard validates retrieved memories through reasoning-path consensus and retains negative lessons to defend LLM agents against memory poisoning~\citep{wei2025memguard}, providing a reference for protecting stored information before it influences planning.

Training should include interrupted and partially completed executions, so the policy learns to recover from deviations and resume the outstanding task. For ordered and conditional programs, supervision of intermediate state transitions may be more informative than a terminal success label alone. These mechanisms are motivated by the observed failures, but their effectiveness on Language2Drive remains to be tested.

\paragraph{Align planned behavior with the controller that executes it.}
Models should be evaluated and, where appropriate, trained with the execution interfaces used at deployment. Controller-aware rollouts can reveal whether failures arise from incorrect plans or from tracking, timing, and stopping errors. Calibrating trajectory timestamps, coordinate frames, vehicle dynamics, and replanning intervals is particularly relevant for narrow speed bands, exact parking, and sustained lane constraints.

A useful diagnostic intervention is to execute the same predicted plan through a verified controller, or to provide a correct plan through the existing controller. These comparisons can localize an execution bottleneck before changing the language model. They would also clarify which improvements transfer across simulators and which depend on a particular interface.

\paragraph{Teach useful safe alternatives alongside refusal.}
The exploratory safety-conflict results indicate that avoiding a prohibited action and completing a useful replacement are distinct capabilities. Training should therefore cover safe fallback, justified refusal, and constrained alternatives, with clear supervision of the behavior appropriate to each conflict. Benign matched requests are also needed to detect unnecessary refusal.

A successful system should preserve as much of the passenger's feasible intent as possible while respecting safety constraints. This requires learning what to do after rejecting literal compliance, rather than treating inactivity as a universally acceptable response. Evaluation should score the usefulness and completion of the replacement behavior separately from avoidance of the unsafe action.

\subsection{What would constitute convincing progress?}

Improvement should be demonstrated through increased paired completion on held-out scenes, targets, and instruction compositions, together with robustness across paraphrases and repeated seeds. Individual success, paired success, graded completion, and driving quality should remain separately visible. Higher aggregate scores or larger instruction-induced trajectory differences are insufficient if the requested alternatives remain uncompleted.

Process-level diagnostics can complement these outcome metrics. Agentic-MME verifies intermediate visual-tool and search steps in multimodal agents~\citep{wei2026agentic}, while TELBench and DRIFT study span-level error localization in deep-research agent trajectories~\citep{wang2026deepresearch}. For driving, analogous diagnostics could inspect target bindings, program-state transitions, and constraint violations along a rollout to identify where execution departs from the requested behavior.

Mechanistic claims require targeted interventions. Removing or shuffling language can test dependence on the instruction; supplying a correct target binding can test grounding; supplying the correct program state can test temporal tracking; and substituting a verified controller can test execution. These interventions should preserve the other relevant conditions and be assessed with uncertainty estimates that respect the dependence between paired arms and paraphrases. Paired success provides stronger behavioral evidence than isolated success, but repeated controls remain necessary to distinguish instruction dependence from stochastic variation.

The development target established by Language2Drive is reliable access to feasible behaviors through language. Achieving it requires the requested choice to be correctly grounded, preserved during planning, and completed throughout execution. Progress should therefore be measured by whether a passenger can consistently obtain either feasible alternative under matched conditions.

\end{document}